\documentclass[final,12pt]{colt2024} % Include author names

\title[Second Order Methods for Bandit Optimization and Control]{Second Order Methods for  Bandit Optimization and Control}
\usepackage[capitalise]{cleveref}
\usepackage{xcolor}
\usepackage{dsfont}
\usepackage{array}
\usepackage{algorithm}
\usepackage{algorithmic}
\usepackage{mathtools}

\def\M{{\mathcal M}}

\newcommand{\F}{\mathbb{F}}

\newcommand{\uv}{\ensuremath{\mathbf u}}

\newcommand{\defeq}{\stackrel{\text{def}}{=}}

\def\Y{{\mathcal Y}}

\def\M{{\mathcal M}}

\newcommand{\real}{\mathbb{R}}

\newcommand{\RBB}{\mathbb{R}}
\newcommand{\EBB}{\mathbb{E}}

\newcommand{\FM}{\mathcal{F}}

\newcommand{\e}{\mathbf{e}}
\newcommand{\ynat}{\mathbf{y}^{K}}
\newcommand{\xnat}{\mathbf{x}^{K}}
\newcommand{\embed}{\mathfrak{e}}

\newcommand\ball{\mathbb{B}}
\newcommand\sphere{\mathbb{S}}

\newcommand{\smoothfBS}[1]{\Tilde{f}_{#1}^{\ball,\sphere}}
\newcommand{\smoothfSS}[1]{\Tilde{f}_{#1}^{\sphere,\sphere}}
\newcommand{\smoothfBB}[1]{\Tilde{f}_{#1}^{\ball,\ball}}

\newcommand{\A}{\mathcal{A}}

\newcommand{\y}{\ensuremath{\mathbf y}}

\newcommand{\K}{\ensuremath{\mathcal K}}

\def\C{{\mathcal C}}

\def\x{\mathbf{x}}

\def\y{\mathbf{y}}
\def\w{\mathbf{w}}

\def\regret{\mbox{{Regret}}}
\def\mregret{\mbox{{Memory-Regret}}}
\def\cregret{\mbox{{Control-Regret}}}

\newcommand{\ignore}[1]{}

\newcommand{\eh}[1]{\noindent{\textcolor{blue}{\{{\bf EH:} \em #1\}}}}

\newcommand{\email}[1]{\texttt{#1}}

\newtheorem{assumption}{Assumption}
\newtheorem{observation}{Observation}
\DeclareMathAlphabet{\mathbfsf}{\encodingdefault}{\sfdefault}{bx}{n}

\DeclareMathOperator*{\argmin}{arg\,min}

\newcommand{\norm}[1]{\|#1\|}

\newcommand{\E}{\mathbb{E}}

\newcommand{\trace}{\mathrm{tr}}

\newcommand{\poly}{\mathrm{poly}}

\newcommand{\reals}{\mathbb{R}}
\newcommand{\eps}{\varepsilon}

\renewcommand{\leq}{~\le~}
\renewcommand{\geq}{~\ge~}

\let\oldtfrac\tfrac
\renewcommand{\tfrac}[2]{\smash{\oldtfrac{#1}{#2}}}

\let\nablaold\nabla
\renewcommand{\nabla}{\nablaold\mkern-2.5mu}
\newcommand{\inp}[2]{\left\langle #1,#2\right\rangle}

\usepackage{times}
\usepackage{mdframed}
\usepackage{float}
\usepackage{wrapfig}
\usepackage{enumitem}
\coltauthor{%
 \Name{Arun Suggala}\thanks{Equal contribution, alphabetical order.}\Email{arunss@google.com}\\
 \addr Google Research%
 \AND
 \Name{Y. Jennifer Sun}\footnotemark[1]\Email{ys7849@princeton.edu}\\
 \addr Princeton University%
 \AND
 \Name{Praneeth Netrapalli} \Email{pnetrapalli@google.com}\\
 \addr Google Research%
 \AND
 \Name{Elad Hazan} \Email{ehazan@princeton.edu}\\
 \addr Princeton University and Google DeepMind%
}

\begin{document}

\maketitle

\begin{abstract}
Bandit convex optimization (BCO) is a general framework for online decision making under uncertainty.  While tight regret bounds for general convex losses have been established,  existing algorithms achieving these bounds have prohibitive computational costs for high dimensional data.

In this paper, we propose a simple and practical BCO algorithm inspired by the online Newton step algorithm. We show that our algorithm achieves optimal (in terms of horizon) regret bounds for a large class of convex functions that satisfy a condition we call $\kappa$-convexity. This class contains a wide range of practically relevant loss functions including linear losses, quadratic losses, and generalized linear models. In addition to optimal regret,  this method is the most efficient known algorithm for several well-studied applications including bandit logistic regression.

Furthermore, we investigate the adaptation of our second-order bandit algorithm to online convex optimization with memory. We show that for loss functions with a certain affine structure, the extended algorithm attains optimal regret. This leads to an algorithm with optimal regret for bandit LQ problem under a fully adversarial noise model, thereby resolving an open question posed in \citep{gradu2020non} and \citep{sun2023optimal}.

Finally, we show that the more general problem of BCO with (non-affine) memory is harder. We derive a  $\tilde{\Omega}(T^{2/3})$ regret lower bound, even under the assumption of smooth and quadratic losses.
\end{abstract}


\section{Introduction}
Bandit convex optimization (BCO) is a prominent framework for online decision-making. It can be described as an interactive game between a learner and an adversary. At time $t\in\mathbb{N}$, the learner must choose an action $x_t$ from a convex constraint set $\K\subset\mathbb{R}^d$. Once $x_t$ is chosen and played by the learner, the adversary reveals a convex loss function $f_t:\mathbb{R}^d\rightarrow\mathbb{R}$, to which the learner suffers loss $f_t(x_t)$. Unlike full-information settings where the learner observes the loss function $f_t$ at each round, bandit feedback provides only scalar feedback—the loss associated with the chosen action, \emph{i.e.,} the scalar $f_t(x_t)$. The learner's goal is to minimize \textit{regret} over a time horizon $T$, which is defined to be the difference between the total loss incurred by the learner and the best fixed action in $\K$ had the loss sequences were known ahead of the time:
\begin{align*}
\regret_T=\sum_{t=1}^T f_t(x_t)-\min_{x\in\K}\sum_{t=1}^T f_t(x).
\end{align*}
%This setting holds great practical value since it mirrors numerous real-world scenarios where decisions must be made sequentially and under uncertainty. For instance, businesses selecting marketing strategies, medical professionals choosing treatment plans, or online platforms recommending personalized content, all operate under similar conditions where both the nature of the loss function and the consequences of alternative choices remain unknown. The bandit scenario encapsulates a fundamental challenge: the learner must operate under significant uncertainty, making decisions and learning from limited feedback, a reality that necessitates robust strategies to manage the exploration-exploitation trade-off effectively and make informed decisions over time.
There is a long line of work in the online learning community that aims to attain optimal regret guarantees in bandit convex optimization, see e.g.  \cite{lattimore2020bandit}. The optimal regret for this setting in terms of the number of iterations, on the order of $O(\sqrt{T})$, was obtained in \cite{bubeck2017kernel}. However, the regret of this  method has high polynomial dependence on the dimension, and similarly the running time is polynomial in both dimension and number of iterations, rendering it impractical in many applications. Motivated by the need for more efficient methods, works by \cite{abernethy2009competing,hazan2014bandit,suggala2021efficient} obtained more practical algorithms, but for restrictive classes of loss functions. Namely, these latter works apply to linear, strongly-convex and smooth, and quadratic losses, respectively. 

One of the interesting remaining open problems in the area is to design an online algorithm that (1) works for a rich class of functions, (2) is computationally efficient, and (3) obtains $O(\sqrt{T})$ regret without heavy dependence on the dimension. We advance this research direction with an efficient second-order method with contributions summarized below.

\subsection{Our contributions to bandit convex optimization (BCO)}
\label{sec:contributions-bco}
Our first contribution is a new algorithm called Bandit Newton Step (Algorithm~\ref{alg:simple-bqo}), which is a natural adaptation of the Online Newton Step algorithm \citep{hazan2007logarithmic} to the bandit feedback setting.  Our algorithm is an improper learning technique, \emph{i.e.,} it plays actions that lie outside of the constraint set $\K$, but competes with the best point inside of $\K$. Such improper learning algorithms are applicable to the applications we consider, such as online regression and online control.  The guarantees of our algorithm apply to a broad class of convex functions that satisfy a curvature assumption called $\kappa$-convexity, which is formally defined as the following:
\begin{wrapfigure}{r}{0.5\textwidth}
%\centering
\includegraphics[scale=0.3]{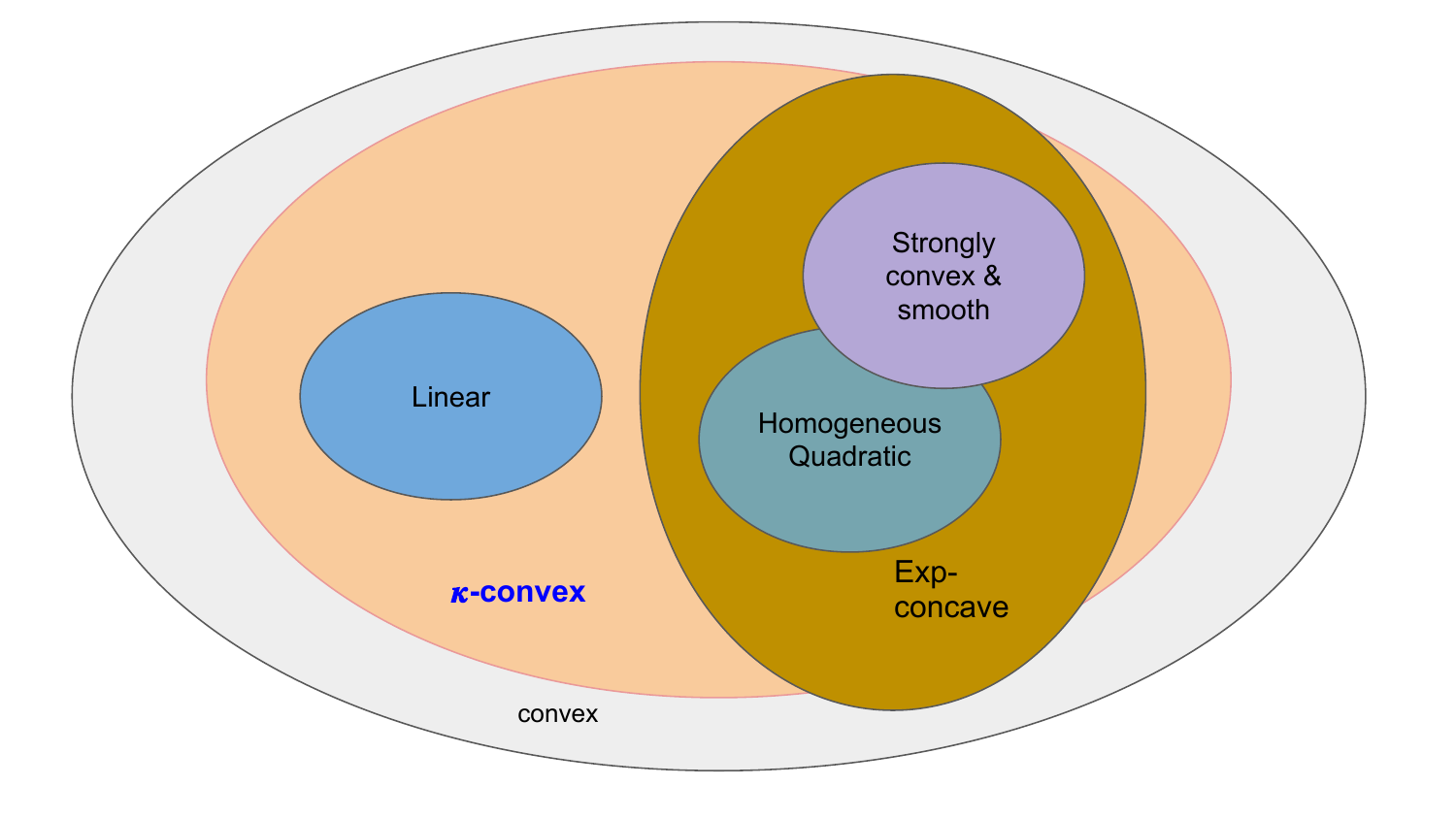} 
%\caption{relationship between loss function classes}\label{ven}
\end{wrapfigure}
%\paragraph{The class of $\kappa$-convex functions.}
%Our bandit algorithm in this section applies to a general class of convex functions that we call $\kappa$-convex. We first give a formal definition, and then proceed to describe functions that fall into this class. 
\begin{definition} [$\kappa$-convexity]
\label{def:kappa-c}
A function $f$ is called $\kappa$-convex over domain $\K \subseteq \reals^d$  iff the following holds: $f$ is convex and twice differentiable almost everywhere, and moreover $\exists c, C>0$ and a PSD matrix $0 \preceq H \preceq I$ such that the Hessian of $f$ at any $x\in\K$ satisfies
\begin{align*}
cH\preceq \nabla^2 f(x)\preceq CH \ , \ \frac{C}{c} \leq \kappa. 
\end{align*}
\end{definition}
As shown in the diagram above, the class of $\kappa$-convex functions generalizes loss functions encountered in prior works~\citep{abernethy2009competing,hazan2014bandit,suggala2021efficient} and finds broad applicability in online learning. Notably, this class relaxes the requirements of strong convexity and smoothness assumptions, demanding them only along directions where the Hessian $H$ is nonzero. Consequently, it encompasses quadratic, (generalized) linear loss functions. A more detailed discussion on $\kappa$-convexity is included in \cref{sec:kappa-convex}. 
%  Observe that this class is more general than that of strongly convex and smooth functions, as it requires strong convexity and smoothness in only those directions where $H$ is nonzero. In particular, it contains (generalized) linear loss functions. 

\begin{mdframed}
\textbf{Contribution 1.} Algorithm~\ref{alg:simple-bqo} guarantees for the class of \textbf{$\kappa$-convex} loss functions:  (1) $O(d^{2.5}\sqrt{T})$ regret upper bound, and (2) per-iteration computational complexity of $O(d^2)$.
\end{mdframed}

\begin{table*}[!htb]
	\centering
	{
        \resizebox{0.9\linewidth}{!}{
		\begin{tabular}{c|| c| c | c| c}
			\hline
			Paper & Losses & Advers.   & Regret & \begin{tabular}{c} Running time \end{tabular} %& \begin{tabular}{c} Computation in each iteration \end{tabular}
			\\ [0.5ex]
			\hline\hline
            \cite{abernethy2009competing}  & linear  &   $\checkmark$ &  $\Tilde{O}(d\sqrt{T})$ & 
			$O(d^2)$  \\\hline
            \cite{hazan2014bandit}  & \begin{tabular}{c}strongly convex,\\ smooth\end{tabular}  &   $\checkmark$ &  $\Tilde{O}(d\sqrt{T})$ & 
			$O(d)$ \\\hline
			\cite{suggala2021efficient}  & convex quadratic  &  $\checkmark$  & $\Tilde{O}(d^{16}\sqrt{T})$ & $O(d^4)$
			\\\hline
            \cite{bubeck2021kernel}  & bounded convex  &   $\checkmark$ &  $\Tilde{O}(d^{10.5}\sqrt{T})$ & $\poly(d,T)$
			\\\hline
            \cite{lattimore2020improved}  & convex  & $\checkmark$ &  $\Tilde{O}(d^{2.5}\sqrt{T})$ & $\exp(d, T)$
			\\\hline
            \cite{lattimore2021improved}  & convex  & $\times$ &  $\Tilde{O}(d^{4.5}\sqrt{T})$ & $\poly(d)^\dagger$
			\\\hline
            \cite{lattimore2023second}  & Lipschitz convex  & $\times$ &  $\Tilde{O}(d^{1.5}\sqrt{T})$ & $O(d^3)$
			\\\hline
			\textbf{Theorem~\ref{thm:expected_regret_bound}} &  $\kappa$-convex & $\checkmark$  & $\tilde{O}(d^{2.5}\sqrt{T})$& $O(d^{2})$
			\\\hline
		\end{tabular}
	}}
 \caption{Comparison with relevant prior works that achieve $\tilde{O}(\sqrt{T})$ regret for bandit convex optimization. The last column presents the per-iteration runtime excluding projections. $\dagger$: this algorithm has $d^4$ operations that happen infrequently, thus running time over $T$ iterations is at least $\max\{d^4, T d\}$.} 
\label{tab:comparison_prior_works}
\end{table*}

The consequences of this result extend to some well studied problem settings with $\kappa$-convex losses, encompassing linear and logistic regression problems (Observation~\ref{obs:examples}). For these settings, our results imply optimal in $T$ regret, and computationally efficient bandit algorithms. A detailed discussion of these contributions is provided in Appendix~\ref{sec:applications}.

% In particular, it includes loss functions used in many applications of interests: logistic bandits, portfolio selection, online learning of quantum states, etc. In many of these applications, our bandit algorithm gives the state-of-art result in either generality or regret optimality. Corollaries of our main theorem are summarized below. 

\begin{mdframed}
\textbf{Contribution 1.a.} Applied to bandit logistic regression problems, Algorithm~\ref{alg:simple-bqo}  guarantees regret $O(d^{2.5}e^{2D} \sqrt{T})$ and running time of $O(d^2)$, where $D$ is the diameter of the domain of the linear predictor.  

\paragraph{Contribution 1.b.} Applied to bandit linear regression problems, Algorithm~\ref{alg:simple-bqo}  guarantees regret $O(d^{2.5}\sqrt{T})$ and  running time of $O(d^2)$.
\end{mdframed}

\subsection{Our contributions to online control}

\begin{table*}[hbt!]
	\centering
	{
		\begin{tabular}{c|| c | c | c}
			\hline
			Paper & Noise & Observability  & Regret  
            %& \begin{tabular}{c} Computation in each iteration \end{tabular}
			\\ [0.5ex]
			\hline\hline
            \cite{gradu2020non}  &   adversarial & full & $\tilde{O}(T^{3/4})$
			\\\hline
			\cite{cassel2020bandit}  &  stochastic & full  & $\tilde{O}(\sqrt{T})$
			\\\hline
            \cite{cassel2020bandit}  &  adversarial & full  & $\tilde{O}(T^{2/3})$
			\\\hline
            \cite{sun2023optimal}  & semi-adversarial & partial & $\tilde{O}(\sqrt{T})$
			\\\hline
			\textbf{Theorem~\ref{thm:nbpc-regret}}  & adversarial  & partial & $\tilde{O}(\sqrt{T})$
			\\\hline
		\end{tabular}
	}
 \caption{Comparison with relevant prior works for bandit control of LQR problem.}
	\label{tab:control_comparison}
\end{table*}

Next, we proceed to studying the extension of our bandit algorithm to online control of linear dynamical systems. The setting of online control is more challenging and general than BCO, since each instantaneous loss depends on the system's current state, which then depends on the past controls chosen by the learner. Very often, online control algorithms are derived from standard online learning algorithms applied to loss functions with memory. In this setting (\cite{anava2015online}),  at each round $t$, the loss function $f_t:(\mathbb{R}^d)^m\rightarrow\mathbb{R}$ depends not only on the current decision of the learner but the most recent $m$ decisions  ($m$ is often referred to as the memory length). The learner's objective is to minimize regret, now defined as
\begin{align*}
\mregret_T=\sum_{t=m}^T f_t(x_{t-m+1},\dots,x_t)-\min_{x\in\K}\sum_{t=m}^T f_t(x,\dots,x). 
\end{align*}
%One of the most important application of online learning with memory is the reduction from nonstochastic control problems (see \cite{hazan2022introduction} for a survey of results). 
The reason memory is important for control is that for stable linear dynamical systems, the effect of past states and control decays exponentially over time, allowing truncation and reduction to online learning with memory. % However, in the bandit setting, the optimal attainable bounds for bandit convex optimization with memory (BCO-M), and for bandit online control, remain unknown. 
The best known bound for bandit linear control with adversarial noise is $\tilde{O}(T^{2/3})$ \citep{cassel2020bandit}. The work of \citep{sun2023optimal} gave a $O(\sqrt{T})$ regret bound for online LQ problem~\footnote{Linear quadratic (LQ) problem refers to controlling of linear dynamical systems with quadratic costs. LQ problem is one of the most fundamental problems in control theory.}, but under a more restrictive semi-adversarial noise model. The model assumes that the noise contains a stochastic component, whose second moment admits a lower bound. It remained an open problem to resolve the optimal rates for bandit linear control and bandit convex optimization with memory (BCO-M) in general.

%\cite{cassel2020bandit,sun2023optimal} showed the optimal $\tilde{O}(\sqrt{T})$ regret for BCO-M if the loss functions are strongly convex and smooth. This limits the noise model to be semi-adversarial. . Without the strongly convex assumptions, the best known regret upper bound is $\tilde{O}(T^{2/3})$ proved in \cite{cassel2020bandit} for general convex and smooth loss functions. Several works in this area (\cite{cassel2020bandit,gradu2020non,sun2023optimal}) all asked an open question that whether optimal rates can be obtained for fully adversarial noise model. 

We address this question, and give new tight upper and lower bounds for these settings as follows. 

\begin{mdframed}
\textbf{Contribution 2.} For bandit LQ problem, we propose \texttt{NBPC}, an efficient controller algorithm that achieves $\tilde{O}(\sqrt{T})$ regret against fully adversarial noise model.
\end{mdframed}
This upper bound leverages the special memory structure, called affine memory, in linear control problems. To see the significance of the affine memory structure, we give a tight lower bound for BCO with general memory, as per below.  
\begin{mdframed}
\textbf{Contribution 3.} BCO-M with general quadratic and smooth loss functions has a $\tilde{\Omega}(T^{2/3})$ regret lower bound even when some degree of improper learning is allowed. 
\end{mdframed}
%Taken together, these contributions separate the difficulty of bandit linear control from that of  BCO-M.
Our analysis establishes a clear distinction between the complexities of bandit linear control and BCO-M, indicating that structural properties of linear control problems are crucial for obtaining optimal regret bounds.

\subsection{Related works}
%In this section, we group and survey the related works by topic. %The discussion of related work and results from BCO applications will be deferred to later sections.  

\paragraph{Bandit convex optimization.} 

The setting of BCO is the limited (zero-order) analogue of online convex optimization, see \citep{hazan2022oco,lattimore2020bandit} for an in-depth treatment of OCO, BCO, and other bandit settings. The first BCO algorithms for adversarial losses were gradient based \citep{flaxman2004online}, and generally applicable, but did not attain the optimal regret bound. A flurry of research followed with more general and efficient methods, culminating in an optimal regret bound and polynomial time algorithm \citep{bubeck2021kernel}. That still is not the end of the picture, as the regret bounds and running time for the latter algorithm have a large (polynomial) dependence on the dimension. 

Special cases of loss functions were addressed in a sequence of works. Notable in these is the optimal regret algorithm for bandit linear optimization \citep{abernethy2009competing}, which is the most important subcase. Further generalizations with improved runtimes and/or regret bounds include strongly convex + smooth~\citep{hazan2014bandit}, and quadratic losses~\citep{suggala2021efficient}. Another important research theme explores efficient bandit algorithms for stochastic losses. The second-order method of \citet{lattimore2023second} is the most relevant to our work; it develops a bandit Newton approach for minimizing Lipschitz convex losses over $\RBB^d$.

\paragraph{Online control.} Online nonstochastic control considers the problem of controlling a discrete-time linear dynamical system with an adversarially chosen perturbation and cost sequence. The goal is to compete with the best fixed policy in some benchmark policy class. \cite{agarwal2019online} is the first work that leverages the stability assumption on the system and proved reduction from nonstochastic control to online learning with memory, where the latter was studied in \citep{anava2015online}. The line of work is too extensive to include here, and we refer those who are interested to \citet{hazan2022introduction} for a survey. The most relevant works to ours are perhaps~\citet{cassel2020bandit, gradu2020non, sun2023optimal}, where the authors also studied nonstochastic control. A comparison of results is included in Table~\ref{tab:control_comparison}.

There are two other relevant works in the full information setting~\citep{simchowitz2020improper, simchowitz2020making}. \cite{simchowitz2020improper} proposed the state-of-art benchmark policy class of disturbance response controllers (DRCs), which both \citep{sun2023optimal} and we use as the benchmark policy class for regret evaluation. \cite{simchowitz2020making} leveraged the special structure of control loss and studied the subclass of OCO-M with memory problem which they referred to as OCO with Affine Memory (OCO-AM). This insight is crucial to our analysis. 

\paragraph{Lower bounds for bandit algorithms.} 

%The special case of multi-armed bandits was understood earlier when \cite{auer2002nonstochastic} and \cite{audibert2009minimax} showed $\Theta(\sqrt{T })$ regret. 

Bandit convex optimization (BCO) is strictly harder than its full information counterpart, and therefore in general the regret is lower bounded by $\Omega(\sqrt{T})$. \cite{shamir2013complexity} showed a tight $\Omega(\sqrt{T})$ lower bound on the regret that holds even for strongly convex and smooth loss functions. 

%Several upper bounds have been developed for extending standard OCO algorithm to OCO-M with an additional multiplicative factor of the memory length in the regret upper bound (\cite{anava2015online} etc.). Usually these memory length are chosen to be of order $\mathrm{poly}(\log T)$ and thus preserves the regret upper bound for the no-memory case up to logarithmic factors. Similar results were also attained in the bandit setting (\cite{gradu2020non, sun2023optimal}). Yet the optimal rate BCO-M (even for quadratic and smooth loss functions) remain resolved unless strong convexity assumption were made on the losses (\cite{sun2023optimal}). The best known rate for convex and smooth loss functions is $\tilde{O}(T^{2/3})$ (\cite{cassel2020bandit,sun2023optimal}).

Bandit convex optimization with memory (BCO-M) is harder than BCO.  \cite{dekel2014bandits, koren2017multi} showed a $\tilde{\Omega}(T^{2/3})$-lower bound for multi-armed bandit problems with switching/moving cost. The addition of a switching cost is a form of loss function with memory and therefore inspires the question that whether bandit optimization with memory over compact and convex decision set with natural loss functions exhibits the same restriction. We answer this question by showing $\tilde{\Omega}(T^{2/3})$ regret lower bound for quadratic and smooth loss that is adaptive at time $t$ to the algorithm's decisions up to time $t-m$, where $m$ is the memory length. We emphasize that most if not all existing upper bound analysis for BCO and BCO-M holds for such adaptivity assumption. Table~\ref{tab:lower_bound_comparison} includes a summary of results in regret lower bounds for bandit algorithms.

\begin{table*}[hbt!]
	\centering
	{
		\begin{tabular}{c|| c  | c | c}
			\hline
			Paper & Decision set   & Regret &  Note  
            %& \begin{tabular}{c} Computation in each iteration \end{tabular}
			\\ [0.5ex]
			\hline\hline
%            \cite{audibert2009minimax}  &  discrete & $\Omega(\sqrt{T})$ & $\times$ 
%			\\\hline
			\cite{dekel2014bandits}  &  discrete & $\tilde{\Omega}(T^{2/3})$ &  $0$-$1$ moving cost
			\\\hline
            \cite{koren2017multi}  &  discrete & $\tilde{\Omega}(T^{2/3})$ &   metric moving cost
			\\\hline
%            \cite{shamir2013complexity} &  continuous  & $\Omega(\sqrt{T})$   & $\times$ & strongly convex, smooth
%			\\\hline
            \cite{cesa2013online} &  continuous  & $\Omega(T^{2/3})$    & $0$-$1$ moving cost
			\\\hline
			\textbf{Theorem~\ref{thm:lower-bound}}  & continuous  & $\tilde{\Omega}(T^{2/3})$ & smooth, quadratic loss with memory
			\\\hline
		\end{tabular}
	}
 \caption{Comparison with relevant prior works on lower bounds for bandit algorithms.}
	\label{tab:lower_bound_comparison}
\end{table*}

\subsection{Paper organization}

In Section \ref{sec:preliminaries} we give details of the setting and notations considered. In Section \ref{sec:simple-bco} we give the simplest version of Bandit Newton Step (\texttt{BNS}) and main theorem statement. Following that we give the extension of \texttt{BNS} to the OCO with affine memory setting in Section~\ref{sec:bqo-AM}. Lower bounds for OCO with memory are given in Section~\ref{sec:bco-m}. Details of the application to control and proof details are left to the appendix. 

\subsection{Notations}
Denote $\mathbb{B}^d=\{x\in\mathbb{R}^d:\|x\|_2\le 1\}$ to be the unit ball in $\mathbb{R}^d$,  and  $\mathbb{S}^{d-1}=\{x\in\mathbb{R}^d:\|x\|_2= 1\}$ to be the unit sphere in $\mathbb{R}^d$. We consider the following matrix norms: $\|M\|_{\mathrm{op}}=\sup\{\|Mx\|_2:\|x\|_2=1\}$ denotes the operator norm. $\|M\|=\rho(M)$ denotes the largest singular value of $M$. For $u_1,\dots,u_n\in\mathbb{R}^d$, $u_{1:n}=(u_1,\dots,u_n)\in\mathbb{R}^{dn}$ denotes the concatenated vector. For matrix-vector products, $A_1,\dots,A_n\in\mathbb{R}^{m\times d}$, $A_{1:n}u_{1:n}=(A_1u_1,\dots,A_nu_n)\in\mathbb{R}^{mn}$ denotes the concatenated matrix-vector product. For $A\in\mathbb{R}^{d\times d}$ such that $A\preceq 0$, we denote the induced matrix norm on $\mathbb{R}^d$ as $\|v\|_{A}=\sqrt{v^{\top}Av}$ and its dual norm as $\|\cdot\|_A^{*}$. It can be checked that $\|\cdot\|_A^{*}=\|\cdot\|_{A^{-1}}$ if $A$ is invertible.

\section{Preliminaries} \label{sec:preliminaries}
%We divide the preliminaries for three topics: BCO (used in Section~\ref{sec:simple-bco}), BCO with memory (BCO-M, used in Section~\ref{subsec:bqo-am-prelims} and Section~\ref{sec:bco-m}), and nonstochastic control of LQ problem (used in Section~\ref{subsec:reduction-control}).

\subsection{BCO}
\label{sec:improper-learning}

We consider the BCO problem over a convex and compact set $\K \subset \mathbb{R}^d$. At each round, the learner is allowed to play a point $y_t\in\mathbb{R}^d$, after which a loss $f_t(y_t)\in\mathbb{R}$ is revealed to the learner. If $y_t$ is restricted to lie in $\K$, then the learner is called a \emph{proper} learner. If $y_t$ can lie outside $\K$, then the learner is called \emph{improper}. In this work, we provide an improper learning algorithm for $\kappa$-convex losses. The performance of the learner is measured by regret, i.e. the difference between the loss incurred by the learner and the best single point $x^*\in\K$ had the loss functions were known ahead of the time, i.e.
\begin{align*}
\regret_T=\sum_{t=1}^T f_t(y_t) - \min_{x^* \in \K} \sum_{t=1}^T f_t(x^*) .
\end{align*}
We make the following assumptions on the sequence of loss functions $\{f_t\}_{t=1}^T$:
\begin{assumption} [Oblivious Adversary] 
\label{assumption:oblivious-adversary}
The sequence of losses $\{f_t\}_{t=1}^T$ are fixed ahead of the game (i.e. the sequence of loss unctions does not depend on the learner's decisions).
\end{assumption}
\begin{assumption} [$\kappa$-convex losses]
\label{assumption:curvature}
The sequence of losses $\{f_t\}_{t=1}^T$ are $\kappa$-convex. That is, $\exists c, C > 0,$ and PSD matrices $0 \preceq H_t \preceq I$ such that 
\[
\forall x \in \K+\ball^d, \  t\in [T], \quad cH_t\preceq \nabla^2 f_t(x)\preceq CH_t \ , \ \text{where } \frac{C}{c} \leq \kappa.
\]
\end{assumption}
\begin{assumption} [Bounded range and gradients] 
\label{assumption:bounded-range-and-grads}
$\exists B, L> 0$ such that
$\{f_t\}_{t=1}^T$ satisfies
\begin{align*}
\max_{1\le t\le T}\sup_{x\in\K} f_t(x)\le B, \ \ \ \max_{1\le t\le T}\sup_{x\in\K} \|\nabla f_t(x)\|_2\le L. 
\end{align*}
\end{assumption}
% Note that Assumption~\ref{assumption:curvature} is strictly more general than the strong convexity assumption used to obtain optimal rates in~\cite{hazan2014bandit} and captures a wide class of functions of interests (e.g. all convex quadratic functions), which we will discuss in the later sections.  

%\subsection{BCO with memory (BCO-M)}
%\label{sec:bco-m-prelim}
%BCO-M is an extension of the BCO problem to loss functions that depend on the learner's recent $m$ decisions, where $m$ is often referred to as the ``memory length". The goal for an online learner is to minimize regret, defined as the excess total loss incurred by the learner comparing to the best decision in some convex bounded set $\K\in\mathbb{R}^d$, i.e.
%\begin{align*}
%\mregret_T=\sum_{t=m}^T f_t(y_{t-m+1}^{\A},\dots,y_{t}^{\A})-\min_{x\in\K}\sum_{t=m}^T \bar{f}_t(x),
%\end{align*}
%where $\bar{f}_t$ is the induced unary form of $f_t$, i.e. $\bar{f}_t(x)=f_t(x,\dots,x)$. 

%We also use $\regret_T(x)$ to denote the regret when compared to some $x\in\K$. When clear from context, we drop $\A$ from $y_t^{\A}$. Similar to the BCO setting, we impose the following regularity assumptions on $f_t$'s with respect to $\K$:

\subsection{Control of LQ problem}
\label{sec:control-prelims}
An important application of BCO-M is its reduction from nonstochastic LQ control. 
Consider the following partially observable linear dynamical system governed by dynamics $(A, B)$ and observation matrix $C$:
\begin{align}
\label{eq:lds}
\x_{t+1}=A\x_t+B\uv_t+\w_t, \ \ \ \y_t=C\x_t+\e_t,
\end{align}
where $\x_t, \uv_t$ are the state and control played at time $t$, respectively. $\w_t$ is the perturbation at time $t$. The system's state is not directly accessible. Instead, an observation, $\y_t$, which is usually a noisy low-dimension projection of the state $\x_t$ coupled with observation noise $\e_t$, is accessible. $A,B,C$ are of appropriate dimensions, i.e. $A\in\mathbb{R}^{d_{\x}\times d_{\x}}$ $B\in\mathbb{R}^{d_{\x}\times d_{\uv}}$, and $C\in\mathbb{R}^{d_{\y}\times d_{\x}}$, where $d_{\x},d_{\y},d_{\uv}$ are the dimensions of $\x_t,\y_t,\uv_t$. $\w_t\in\mathbb{R}^{d_{\x}}$, and $\e_t\in\mathbb{R}^{d_{\y}}$. %We assume that the system is stable, i.e. $\rho(A)<1$. 
In addition, we assume the system has bounded dynamics and permits a strongly stable linear policy:
\begin{assumption} [Strongly stabilizable with bounded dynamics]
\label{assumption:bounded-dynamics-and stability}
The matrices that govern the system dynamics and observations in \cref{eq:lds} are bounded: $\|A\|_{\mathrm{op}}\le \kappa_A$,  $\|B\|_{\mathrm{op}}\le \kappa_B$,  $\|C\|_{\mathrm{op}}\le \kappa_C$, and $\exists K\in\mathbb{R}^{d_{\uv}\times d_{\y}}$ with $A+BKC=HLH^{-1}$ for some $H\succ 0$ and $\max\{\|K\|, \|H\|\|H^{-1}\|\}\le \kappa$, $\|L\|\le 1-\gamma$ for some $\kappa>0$, $0<\gamma \le 1$. 
\end{assumption}
In the nonstochastic control setting, we make no stochastic assumption on the perturbation and noise sequence $\{\w_t, \e_t\}_{t=1}^T$, other than that they are bounded, i.e. 
\begin{align*}
\max_{t\in[T]}\{\max\{\|\w_t\|_2,\|\e_t\|_2\}\}\le R_{\w,\e}.
\end{align*}
Moreover, at time $t$, after a control $\uv_t$ is played by the learner, a time-varying quadratic cost 
\begin{align}
\label{eq:quad-control-cost}
c_t(\y_t, \uv_t)=\y_t^{\top} Q_t\y_t+\uv_t^{\top}R_t\uv_t
\end{align}
is revealed to the learner. We make the following assumption on the adversary.
\begin{assumption} [Oblivious adversary]
\label{assumption:adversary-adaptivity-control}
The adversary that chooses the cost and perturbation sequences is oblivious.
\end{assumption} 
Similar to \cite{simchowitz2020making, sun2023optimal}, we make strong convexity assumption on $c_t$:
\begin{assumption} [Strongly convex and smooth quadratic cost]
\label{assumption:control-loss-curvature}
$\exists \beta\ge 1\ge \alpha>0$ such that $Q_t, R_t\succeq \alpha I$ and $Q_t, R_t\preceq \beta I$.
\end{assumption}
In the bandit setting with partially observable systems, the learner only has access to the scalar loss $c_t(\y_t,\uv_t)$ in addition to the observation $\y_t$ at time $t$. 
The performance in nonstochastic control is to minimize \textit{regret}, defined by the total cost incurred by the learner $\A$ over a time horizon $T$ and the would-be cost incurred by the best policy in a benchmark policy class $\Pi$, had the costs been known ahead of the game, i.e.
\begin{align}
\label{eq:control-regret}
\cregret_T(\A)=\sum_{t=1}^T c_t(\y_t^{\A}, \uv_t^{\A})-\min_{\pi\in\Pi} \sum_{t=1}^T c_t(\y_t^{\pi}, \uv_t^{\pi}). 
\end{align}
Here, $(\y_t^{\A}, \uv_t^{\A})$ is the observation-control pair reached by learner $\A$, and $(\y_t^{\pi}, \uv_t^{\pi})$ is the would-be observation-control pair if the policy $\pi$ was carried from the beginning of the time. It is evident from the definition of regret that the strength of regret as a performance metric strongly depends on the richness of the policy class $\Pi$. The standard benchmark policy class used in literature (\cite{simchowitz2020improper, simchowitz2020making, sun2023optimal}) is the disturbance response controller (DRC) policy class, formally given by the following definition:
\begin{definition} [DRC policy class]
\label{def:dac}
(1) A disturbance response controller (DRC) is a policy $\pi_M$ of length $m$ parametrized by a sequence of $m$ matrices $M=(M^{[j]})_{j=0}^{m-1}$ of dimension $d_{\uv}\times d_{\y}$ such that the control at $t$ is given by $\uv_t^{\pi_M}=K\y_t+\sum_{j=0}^{m-1}M^{[j]}\ynat_{t-j}$, where $\ynat_t$ is the would-be observation had the linear policy $K$ been carried out from the beginning of the time.  

(2) A DRC policy class $\M(m,R_{\M})$ parametrized by $(m, R_{\M})$ is the set of all disturbance response controllers of length $m$ that obey the norm bound:
\begin{align*}
\|M\|_{\ell_1,\mathrm{op}}=\sum_{j=0}^{m-1}\|M^{[j]}\|_{\mathrm{op}}\le R_{\M}. 
\end{align*}
\end{definition}
\section{\texttt{BNS}: Bandit Newton Step for $\kappa$-Convex Losses}
\label{sec:simple-bco}
In this section we describe our bandit Newton algorithm \texttt{BNS} (see Algorithm~\ref{alg:simple-bqo}), which achieves $\tilde{O}(\sqrt{T})$ regret guarantee for the class of $\kappa$-convex functions (Definition~\ref{def:kappa-c}). At a high level, our algorithm tries to estimate the missing information (\emph{i.e.,} gradient and Hessian) about the unknown loss function and then takes a Newton step to compute the next action.  To estimate the gradient and Hessian at time step $t$, we employ the following randomized sampling scheme: We first sample a point uniformly from an ellipsoid centered at the current iterate $x_t$. The covariance matrix of this ellipsoid is constructed using prior Hessian estimates of the loss functions $\{f_s\}_{s=1}^{t-1}$. We then obtain one-point feedback from the adversary regarding the loss value at the sampled point.  This feedback enables gradient and Hessian estimation (see lines 3-5 of Algorithm~\ref{alg:simple-bqo}).
\begin{algorithm}
\caption{\texttt{BNS}: Bandit Newton Step}\label{alg:simple-bqo}
\begin{flushleft}
  {\bf Input:} convex compact set $\K\subset\mathbb{R}^d$, step size $\eta>0$, condition number $\kappa'>0$, time horizon $T\in\mathbb{N}$.
\end{flushleft}
\begin{algorithmic}[1]
\STATE Initialize: $x_1 \in\K$, $\tilde{A}_0 = I  $
\FOR{$t = 1, \ldots, T$}
\STATE Draw $v_{t,1}, v_{t,2}$ uniformly from $\sphere^{d-1} = \{x\in\real^d|\|x\|_2 = 1\}$, and let $y_t=x_t+\frac{1}{2} \tilde{A}_{t-1}^{-\frac{1}{2}}(v_{t,1}+v_{t,2})$.
\STATE Play $y_t$, observe $f_t(y_t)$.
\STATE Create gradient and Hessian estimates $\tilde{\nabla}_t,\tilde{H}_t$: \label{line:est}
\begin{align*}
\tilde{\nabla}_{t}=2df_t(y_t) \tilde{A}_{t-1}^{\frac{1}{2}}v_{t,1}, \ \ \ \tilde{H}_{t}=2d^2 f_t(y_t)\tilde{A}_{t-1}^{\frac{1}{2}}(v_{t,1}v_{t,2}^{\top}+v_{t,2}v_{t,1}^{\top})\tilde{A}_{t-1}^{\frac{1}{2}}.
\end{align*}
\STATE  Update $\tilde{A}_{t} = \tilde{A}_{t-1} + \frac{\eta}{\kappa'} \tilde{H}_t $. Compute $x_{t+1}= \prod_{\K}^{\tilde{A}_t} \left[ x_t - \eta \tilde{A}_t^{-1}  \tilde{\nabla}_t  \right]$, 
where $\prod_{\K}^{\tilde{A}_t}$ is the projection onto  $\K$ w.r.t the norm $\|\cdot\|_{\tilde{A}_t}$. \label{line:newton-update}
\ENDFOR 
\end{algorithmic}
\end{algorithm}
% \subsection{Algorithm and Regret Bound}
%In this section, we first show that the cumulative Hessian estimate $\Tilde{A}_t$ concentrates well around the true cumulative Hessian. Next, we rely on this result to bound the regret of Algorithm~\ref{alg:simple-bqo}. 

\begin{theorem}[\texttt{BNS} regret]
\label{thm:expected_regret_bound}
For $d,T\in\mathbb{N}$, suppose that $\{f_t\}_{t=1}^T$ and the convex compact set $\K\subset\mathbb{R}^d$ satisfy Assumptions~\ref{assumption:oblivious-adversary},\ref{assumption:curvature},\ref{assumption:bounded-range-and-grads}. 
Let $B^* \coloneqq B+\sqrt{2}(L+\sqrt{2}C)$. Then, \texttt{BNS} (\cref{alg:simple-bqo}) with input $(\K, \eta,\kappa',T)$ with $\kappa'\geq \kappa$, and $\eta\le \kappa'(24 d^{3/2}B^*\kappa\sqrt{T\log(dT^2)})^{-1}$ satisfies the following regret guarantee for any $x\in\K$:
%$$   \sum_t \E[f_t(y_t)  -  f_t(x)] = \text{diam}(\K)dB\sqrt{\frac{CT\log{T}}{c\kappa}}.$$
$$   \E\left[\sum_{t=1}^T f_t(y_t)  -  f_t(x)\right] \le \frac{\mathrm{diam}(\K)^2  }{2\eta} + 3\eta d^2 (B^*)^2 T + \frac{ 2d\kappa\kappa'\log(1+ \eta CT/\kappa')} {\eta }.$$
For the case of $\mathrm{diam}(\K) = \sqrt{d}$, by setting $\kappa'=\kappa$, and $\eta= (24 d^{3/2}B^*\sqrt{T\log(dT^2)})^{-1}$, we have
\begin{align*}
\E[\regret_T]\le \tilde{O}(d^{2.5}\kappa^2 B^* \sqrt{T}).
\end{align*}
\end{theorem}
% \js{We have worse dependence here on $C$ since the diameter bound on $y_t$'s also depend on $C$.}
\begin{proof}{(Sketch)} One of the key steps in the proof is to show that the cumulative Hessian estimator $\tilde{A}_{t-1}$ concentrates well around the true cumulative Hessian w.h.p (Lemma~\ref{lem:concentration-cumulative-hessian}). This result implies that the iterates generated by the algorithm are well defined. Another implication of this result is that the action $y_t$ chosen by the learner is at most distance $2$ (measured in $\ell_2$ norm) away from $\K$ w.h.p. Our regret analysis relies on the following two  functions: $\smoothfBS{t}(x) \coloneqq \E_{u\sim \ball, v \sim \sphere}[f_t(x + \frac{1}{2} \tilde{A}_{t-1}^{-1/2}(u+v))|\FM_{t-1}]$, $\smoothfSS{t}(x) \coloneqq \E_{u, v \sim \sphere}[f_t(x + \frac{1}{2} \tilde{A}_{t-1}^{-1/2}(u+v))|\FM_{t-1}]$, where $\sphere = \sphere^{d-1}$ and $\ball = \ball^d$. These two functions are different smoothed variants of $f_t$. A simple application of Stokes' theorem~\citep{flaxman2004online} shows that $\EBB[\tilde{\nabla}_t|\FM_{t-1}] = \nabla \smoothfBS{t}(x_t),$ and $\EBB[f_t(y_t)|\FM_{t-1}] = \smoothfSS{t}(x_t).$ 
% \begin{align*}
% &\smoothfBB{t}(x) \coloneqq \E_{u\sim \ball, v \sim \ball}\left[f_t\left(x + \frac{1}{2} \tilde{A}_{t-1}^{-\frac{1}{2}}(u+v)\right)\right]\\
% &\smoothfBS{t}(x) \coloneqq \E_{u\sim \ball, v \sim \sphere}\left[f_t\left(x + \frac{1}{2} \tilde{A}_{t-1}^{-\frac{1}{2}}(u+v)\right)\right]\\
% &\smoothfSS{t}(x) \coloneqq \E_{u\sim \sphere, v \sim \sphere}\left[f_t\left(x + \frac{1}{2} \tilde{A}_{t-1}^{-\frac{1}{2}}(u+v)\right)\right].
% \end{align*}
In the rest of the proof we perform the following regret decomposition and bound each of the terms
\begin{align*}
    \sum_{t=1}^{T}\E[f_t(y_t) -  f_t(x)] & = \sum_{t=1}^{T}\E[\smoothfBS{t}(x_t) -  \smoothfBS{t}(x)]  +  \E[\smoothfBS{t}(x) -  f_t(x)]  +  \E[\smoothfSS{t}(x_t) - \smoothfBS{t}(x_t)].
\end{align*}
The second and third terms capture the penalty we pay for exploration. To bound the first term, we view our algorithm as performing stochastic Newton method on the smoothed losses $\smoothfBS{t}$ (Lemma~\ref{lem:stoc_ons_full_info_regret}).
\end{proof}

\noindent\textbf{Running time.} Most steps of the \texttt{BNS} algorithm can be implemented with only matrix-vector products, that  run in time $O(d^2)$ for dense matrices. In addition, its SVD can be maintained in $O(d^2)$ time using the techniques of \citep{golub-matrix,gu1994stable}, since we have only rank-2 updates each iteration. This allows us to compute both the inverse and inverse square-root in linear time.\vspace{0.03in}

\noindent\textbf{Comparison with relevant prior work.}
We take a moment to compare Algorithm~\ref{alg:simple-bqo} with that in \citep{suggala2021efficient, lattimore2023second}. While these works have also developed bandit Newton methods, our work differs from theirs in several key aspects: (a) unlike the method of \citet{suggala2021efficient}, which is restricted to quadratic losses, our algorithm is applicable to a broader class of loss functions. Furthermore, our approach is simple and avoids focus regions and restart conditions of \citet{bubeck2017kernel, suggala2021efficient}. The main reason for this is that $\kappa$-convexity helps us get an optimistic hessian estimate for the entire action space. Finally, the regret of our algorithm has a better dependence in $d$. (b) In contrast to \citet{lattimore2023second}, which focuses on stochastic losses in unconstrained domains, our algorithm is designed for the more general adversarial bandit setting in constrained domains.\vspace{0.03in}

\noindent\textbf{Improved dependence on $\kappa$.} Observe that the regret in Theorem~\ref{thm:expected_regret_bound} has a $\kappa^2$ dependence. In Appendix~\ref{sec:improved_kappa_dependence}, we provide an alternate analysis that improves this dependence to $\kappa$, but at the cost of worse dependence on $d$. In particular, this analysis leads to a regret of $\tilde{O}(d^3\kappa B^*\sqrt{T}),$ instead of $\tilde{O}(d^{2.5}\kappa^2 B^* \sqrt{T})$ stated in Theorem~\ref{thm:expected_regret_bound}. This improved $\kappa$ dependence is especially useful in problems with very large $\kappa$, \emph{e.g.,} logistic regression where $\kappa = e^{\text{diam}(\K)}$ (Appendix~\ref{sec:logistic-regression}). We believe a tighter concentration bound for the cumulative Hessian $\tilde{A}_t$ can improve the regret to $\tilde{O}(d^{1.5}\kappa B^*\sqrt{T})$.
%In particular, \cite{lattimore2023second} works under stochastic setting over $\mathbb{R}^d$ and assumes that (1) the loss functions are Lipschitz; (2) global minimizer exists; (3) the algorithm starts within a bounded distance from the optimum. 
% \begin{theorem}
% \label{thm:expected_regret_bound}
% Suppose the adversary is oblivious. Suppose the algorithm is run with $\kappa\leq\frac{c}{C}$. Let $\lambda^*=\max_{t\in[T]}\lambda_{\max}(H_t)$. Then the following holds for any $x\in\K$:
% %$$   \sum_t \E[f_t(y_t)  -  f_t(x)] = \text{diam}(\K)dB\sqrt{\frac{CT\log{T}}{c\kappa}}.$$
% $$   \E\left[\sum_{t=1}^T f_t(y_t)  -  f_t(x)\right] = \frac{\mathrm{diam}(\K)^2  }{2\eta} + 3\eta d^2 (B+\sqrt{2}(G+\sqrt{2}C\lambda^*))^2 T + \frac{C d\log(1+\kappa \eta C\lambda^*T)} {c\eta \kappa}.$$
% In particular, by setting $\kappa=\frac{c}{C}$, $\eta=O(\kappa^{-1}T^{-\frac{1}{2}})$, we have
% \begin{align*}
% \E[\regret_T]\le O(\kappa^{-1}\sqrt{T}).
% \end{align*}
% \end{theorem}

% \begin{remark}
%     Suppose $\kappa = \frac{c}{C}$. Then the regret scales as $\tilde{O}(\kappa^{-1}d\sqrt{T})$.  The dependence on $\kappa$ might not be optimal. Recent works have considered GLMs in the stochastic bandit setting and proposed algorithms that achieve $\tilde{O}(\kappa^{-\frac{1}{2}}d\sqrt{T})$ regret~\cite{faury2020improved}.
% \end{remark}

%\subsection{Proper Learning}
%\label{sec:proper-learning}
%\input{content/BQO}
\section{\texttt{NBPC}: Newton Bandit Perturbation Controller for LQ problem}
\label{sec:bqo-AM}

One extension of interest is the application of a BCO algorithm to online control of LQ problem under bandit feedback. Recall the linear time-invariant dynamical system defined in \cref{eq:lds} and the control regret metric in \cref{eq:control-regret} with quadratic costs of the structure in \cref{eq:quad-control-cost}. A natural question is that whether an optimal regret guarantee can be extended to bandit nonstochastic LQ problem. 

\paragraph{Challenges.} The possibility of an optimal rate in the bandit nonstochastic LQ problem remains unresolved. Optimal rates are established only for stochastic \citep{cassel2020bandit} or semi-adversarial \citep{sun2023optimal} settings via reduction to BCO-M. The transition from semi-adversarial to fully adversarial models is significant, especially since strong convexity, which generally improves regret bounds in online learning, does not guarantee the same for loss functions in nonstochastic control problems faced with adversarial disturbances. This is because the loss function reduced from the control problem does not necessarily inherit strong convexity in the presence of fully adversarial disturbances. 

Can optimal regret in bandit LQ problem be established for a fully adversarial noise model? We answer this question in two folds. On one hand, in \cref{sec:bco-m}, we will show the limitation of BCO-M for general convex loss functions by showing a lower bound of $\tilde{\Omega}(T^{2/3})$. However, in this section we will show that despite the hardness of BCO-M for general convex loss functions, an optimal $\tilde{O}(\sqrt{T})$ regret is still attainable if
we leverage additional structure of the control problem. \cite{simchowitz2020making} established that the reduction from nonstochastic LQ problem falls within a subclass of the OCO-M framework, where the the dependence on previous decisions assumes an affine structure. This discovery enabled the derivation of the first optimal logarithmic regret bounds for full-information (as opposed to bandit) LQ problem in the presence full adversarial disturbances. The challenge remains in extending the optimal regret bound to the bandit settings. With only scalar feedback, the learner must balance the bias and variance of the gradient estimator variance to obtain optimal regret. Without strong convexity and additional structural assumptions, this balance becomes harder to achieve. %The current best regret rate in fully adversarial settings is $\tilde{O}(T^{2/3})$, as shown by \cite{cassel2020bandit}.

This section will be organized as the following: \cref{subsec:bqo-am-prelims} will introduce the preliminaries of the problem of bandit quadratic optimization with affine memory (BQO-AM) including important assumptions. \cref{subsec:BQO-AM-algo} describes \texttt{BNS-AM}, an algorithm that achieves optimal regret guarantee for the class of problems described in \cref{subsec:bqo-am-prelims}. \cref{subsec:reduction-control} will discuss the reduction from bandit LQ problem to the problem of BQO-AM and introduce the controller \texttt{NBPC} (Newton Bandit Perturbation Controller) which achieves the optimal $\tilde{O}(\sqrt{T})$ control regret. 

\subsection{BQO-AM preliminaries}
\label{subsec:bqo-am-prelims}

In the BQO-AM setting, similar to the general BCO-M setting, the learner is asked to play a decision $y_t\in\mathbb{R}^d$ at time $t$. After the decision is made, an adversary reveals a cost function with memory length $m$, $f_t:(\mathbb{R}^{d})^m\rightarrow\mathbb{R}$, that is evaluated the learner's $m$ most recent decisions. For the subclass of BQO-AM problems, $f_t$ admits the following structure:
\begin{assumption} [Affine memory structure]
\label{assumption:affine-memory}
The loss function $f_t:(\mathbb{R}^d)^m\rightarrow\mathbb{R}$ admits the structure
\begin{align}
\label{eq:bqo-am-function}
f_t(y_{t-m+1:t})&=q_{t}\left(B_t+\sum_{i=0}^{m-1}G^{[i]}Y_{t-i}y_{t-i}\right),
\end{align}
where $G=(G^{[i]})_{i=0}^{m-1}$ is a sequence of matrices of dimension $n\times p$, $Y_{t-m+1:t}$ are $m$ matrices of dimension $p\times d$, and $B_t$ is a vector of dimension $n$. $q_t:\mathbb{R}^n\rightarrow \mathbb{R}$ is a quadratic function with Hessian $Q_t$. We assume that $\alpha I\preceq Q_t \preceq \beta I$, for some $0<\alpha\le 1 \le \beta$. Additionally, we assume that the learner has access to the following quantities:
\begin{align*}
H_t =G_t^{\top}G_t \in \mathbb{R}^{d\times d}, \ \ \ G_t=\sum_{i=0}^{m-1}G^{[i]}Y_{t-i}\in \mathbb{R}^{n\times d}, \ \ \ \forall t\in[T]. 
\end{align*}
\end{assumption}
We note that the strong convexity assumption here does not translate to the strong convexity of $f_t$ without further assumptions on $Y_{t-m+1:t}$.
%\begin{align*}
%f_t(y_{t-m+1:t})&=q_{t}^{(1)}\left(B_t+\sum_{i=1}^{m-1}G^{[i]}Y_{t-i}y_{t-i}\right)+q_{t}^{(2)}(Y_ty_t),
%\end{align*}
%where $G=(G^{[i]})_{i=0}^{m-1}$ is a sequence of matrices of dimension $n\times p$, $Y_{t-m+1:t}$ are $m$ matrices of dimension $p\times d$, and $B_t$ is a vector of dimension $n$. $q_t^{(1)}:\mathbb{R}^n\rightarrow \mathbb{R}$ and $q_t^{(2)}:\mathbb{R}^p\rightarrow \mathbb{R}$ are quadratics defined by some matrices $Q_t\in\mathbb{R}^{n\times n}$, $R_t\in\mathbb{R}^{p\times p}$:
%\begin{align*}
%q_{t}^{(1)}(x)=x^{\top} Q_t x, \ \ \  q_{t}^{(2)}(y)=y^{\top}R_ty. 
%\end{align*}
% We assume that $\alpha I\preceq Q_t, R_t$, and $Q_t, R_t \preceq \beta I$, for some $0<\alpha\le\beta$. We note that the strong convexity assumption here does not translate to the strong convexity of $f_t$, without further assumptions on $Y_{t-m+1:t}$. We assume that the learner has access to the following quantities:
%\begin{align*}
%H_t = G_t^{\top}G_t+Y_t^{\top}Y_t\in\mathbb{R}^{d\times d}, \ \ \ G_t=\sum_{i=1}^{m-1}G^{[i]}Y_{t-i}\in \mathbb{R}^{n\times d}, \ \ \ \forall t\in[T]. 
%\end{align*}
%\ar{Is this a reasonable assumption? Can we estimate these quantities from single point feedback? Since $q_t$'s are strongly convex + smooth, we might be able to do it.}\eh{I think the idea is that they are determined by the system matrices $A,B$ and the noises $w_{1:t}$, which are known, Jennifer? } \js{Yes. I will explain in the following subsection. The noise sequence can be computed, and the system matrices are either known or can be estimated.}
Let $\bar{f}_t$ denote the induced unary form of $f_t$: $\bar{f}_{t}(x)=f_t(x,\dots,x)$.
The Hessian $\nabla^2\bar{f}_{t}$ satisfy the following property: $\alpha H_t\preceq \nabla^2\bar{f}_t\preceq \beta H_t$ since $\nabla^2\bar{f}_t=G_t^{\top} Q_tG_t\in[\alpha H_t, \beta H_t]$.
We make the following regularity assumptions.
\begin{assumption} [Diameter and gradient bound] 
\label{assumption:diameter-grad-bound}
$\K\subset\mathbb{R}^d$ has Euclidean diameter bound $D$, i.e. $$\sup_{x,y\in\K} \|x-y\|_2 \le D.$$ 
The value of $f_t$ is bounded by $B$~\footnote{The assumption on bounded function value is non-standard. We add this assumption for simplicity. All of our results hold if substituting the value bound by $LD$.}, and gradient norm is bounded $L$ over $\K^m$, $\forall t$, i.e.
\begin{align*}
\max_{t\in[T]}\sup_{z\in\K^m} |f_t(z)|\le B, \ \ \ \max_{t\in[T]}\sup_{z\in\K^m} 
\|\nabla f_t(z)\|_2\le L. 
\end{align*}
\end{assumption}
In existing literature on BCO-M, the algorithm holds regret guarantee against $m$-step delayed adversaries (\cite{gradu2020non, sun2023optimal}). Here, we also allow the same adaptivity of the adversary:
\begin{assumption} [Adaptivity of the adversary] 
\label{assumption:adversary-adaptivity}
The adversary is allowed to be $(t-m)$-steps adaptive. In particular, if we denote as $\F_t$ the filtration generated by the algorithm's randomness up to time $t$, then the loss function $f_t$ supplied by the adversary at time $t$ is $\F_{t-m}$-measurable. 
\end{assumption}
Moreover, we make the following regularity assumptions on $G, Y_t$. In particular, it can be shown that when reducing from nonstochastic LQ problems satisfying Assumption~\ref{assumption:bounded-dynamics-and stability},\ref{assumption:adversary-adaptivity-control},\ref{assumption:control-loss-curvature}, the following assumption is satisfied.

\begin{assumption}[Exponential decay and positive convolution invertibility-modulus]
\label{assumption:convolution-invertibility-modulus}
$G, Y_t$ satisfy that for $m=\poly(\log T)$,
\begin{align*}
&\kappa(G)=\min\left\{1, \inf_{\sum_{n\ge 0}\|u_n\|_2^2=1}\left\{\sum_{n=0}^{\infty}\left\|\sum_{i=0}^nG^{[i]}u_{n-i}\right\|_2^2\right\}\right\}=\Omega(1),\\
&\|G\|_{\ell_1,\mathrm{op}}=\sum_{i=0}^{m-1}\|G^{[i]}\|_{\mathrm{op}}\le R_G=O(1), \ \ \ \sum_{i=m}^{\infty}\|G^{[i]}\|_{\mathrm{op}}\le \frac{R_G}{T},\\
&\max_{0\le i\le m-1} \|(G^{[i]})^{\top}G^{[i]}\|_2\le \tilde{R}_G=O(1), \ \ \ \max_{t\in[T]}\|Y_t\|_\mathrm{op}\le R_Y=O(1).
\end{align*}
\end{assumption}
%It can be shown that for nonstochastic LQ problems, Assumption~\ref{assumption:convolution-invertibility-modulus} is satisfied. The reduction is established in Section~\ref{subsec:reduction-control}. 

%\begin{proof}
%To see the first inequality, note that since $f_t$ is quadratic, $[\nabla^2 f_t]_{ii}=\nabla^2_{y_{t-m+i}} f_t$, where $\nabla^2_{y_{t-m+i}} f_t$ is the Hessian of $f_t$ w.r.t. $y_{t-m+i}$. $\forall 1\le i\le m-1$, by smoothness assumption on $Q_t$,
%\begin{align*}
%[\nabla^2 f_t]_{ii}=\nabla^2_{y_{t-m+i}} f_t=Y_{t-m+i}^{\top}(G^{[m-i]})^{\top}Q_tG^{[m-i]}Y_{t-m+i}\preceq \beta\tilde{R}_{G}Y_{t-m+i}^{\top}Y_{t-m+i}\preceq \beta\tilde{R}_{G}H_{t-m+i},
%\end{align*}
%and similarly by smoothness assumption on $R_t$,
%\begin{align*}
%[\nabla^2 f_t]_{mm}=\nabla^2_{y_{t}} f_t=Y_{t}^{\top}R_tY_{t}&\preceq \beta\tilde{R}_{G}Y_t^{\top}Y_t\preceq \beta\tilde{R}_{G}H_t. 
%\end{align*}
%Summing over $i$, we get the desired inequality. To see the second inequality, by chain rule, we have $\nabla^2\bar{f}_{t}=G_t^{\top}Q_tG_t+Y_t^{\top}R_tY_t\succeq \alpha H_t$, and $\nabla^2\bar{f}_t\preceq \beta H_t$ similarly. 
%\end{proof}

\subsection{\texttt{BNS-AM}: Algorithm and Guarantees for BQO-AM problems}
\label{subsec:BQO-AM-algo}

We introduce the following \texttt{BNS-AM} algorithm for BQO-AM problems. The update rule resembles that of \texttt{BNS} (\cref{alg:simple-bqo}), except that the learner uses $H_t$ to update the preconditioner matrix in place of the Hessian estimator as in \texttt{BNS}. 

\begin{algorithm}
\caption{\texttt{BNS-AM}: Bandit Newton Step with Affine Memory}\label{alg:bqo-memory}
\begin{flushleft}
  {\bf Input:} convex compact set $\K\subset \mathbb{R}^d$, step size $\eta>0$, memory parameter $m\in\mathbb{N}$, time horizon $T\in\mathbb{N}$, curvature parameter $\alpha>0$. 
\end{flushleft}
\begin{algorithmic}[1]
\STATE Initialize: $x_{1}=\dots=x_{m}\in\K$, $\tilde{g}_{0:m-1} = \mathbf{0}_{d}$, $ \hat{A}_{0:m-1} = mI$.
\STATE Sample $u_t\sim \sphere^{d-1}$ i.i.d. uniformly at random for $t=1,\dots,m$. 
\STATE Set $y_t = x_t + \hat{A}_{t-1}^{-\frac{1}{2}}u_t$, $t=1,\dots,m$. 
\FOR{$t = m, \ldots, T$}
\STATE Play $y_t$, observe $f_t(y_{t-m+1:t})$, receive $H_t$.
\STATE Update $\hat{A}_t=\hat{A}_{t-1}+\frac{\eta \alpha}{2}H_t$. 
\STATE Create gradient estimate: $\tilde{g}_{t}=df_t(y_{t-m+1:t}) \sum_{j=0}^{m-1}\hat{A}_{t-1-j}^{\frac{1}{2}}u_{t-j}\in\mathbb{R}^{d}$. \label{line:grad-est-m} 
\STATE Update
$ x_{t+1} = \prod_{\K}^{\hat{A}_{t-m+1}} \left[ x_{t} - \eta \hat{A}_{t-m+1}^{-1}  \tilde{g}_{t-m+1}  \right]$.
\STATE Sample $u_{t+1}\sim \sphere^{d-1}$ uniformly at random, independent of previous steps.  
\STATE Set $y_{t+1}=x_{t+1}+\hat{A}_{t-m+1}^{-\frac{1}{2}}u_{t+1}$. \label{line:y-m}
\ENDFOR
\end{algorithmic}
\end{algorithm}

The affine memory structure and access to $H_t$ in \cref{assumption:affine-memory} provides upper and lower bounds on the Hessian of the loss functions, addressing the challenge in creating Hessian estimators in high dimensions for loss functions with memory. This allows \texttt{BNS-AM} to achieve optimal regret for BQO-AM problems. 

\begin{theorem} [\texttt{BNS-AM} regret]
\label{thm:bqo-m-regret}
For $d,T\in\mathbb{N}$, suppose that the sequence of loss functions $\{f_t\}_{t=1}^T$ and the convex compact set $\K\subset\mathbb{R}^d$ satisfy Assumptions~\ref{assumption:affine-memory},\ref{assumption:diameter-grad-bound},\ref{assumption:adversary-adaptivity},\ref{assumption:convolution-invertibility-modulus}. Then, \texttt{BNS-AM} (\cref{alg:bqo-memory}) with inputs $(\K, \eta, m, T, \alpha)$ with $m=\poly(\log T)$, $\eta=\tilde{\Theta}\left(\frac{1}{\alpha\sqrt{T}}\right)$, $\alpha$ given by \cref{assumption:affine-memory} satisfies the following regret guarantee against any $x\in\K$:
\begin{align*}
\E[\mregret_T(x)]\le \tilde{O}\left(\frac{\beta}{\alpha}B^*\sqrt{T}\right), 
\end{align*}
with $B^*=B+(L+\beta\sqrt{m})\sqrt{m}$.
\end{theorem}

\paragraph{Proof overview.} We decompose the regret into losses incurred by exploration perturbation, movement cost due to estimating the instantaneous loss with the unary form of loss evaluated at the current iterate, and the regret of bandit Newton step without memory. 
\begin{align*}
f_t(y_{t-m+1:t})-\bar{f}_t(x)= \underbrace{f_t(y_{t-m+1:t})-f_t(x_{t-m+1:t})}_{(\text{perturbation loss})}+\underbrace{f_t(x_{t-m+1:t})-\bar{f}_t(x_t)}_{(\text{movement cost})}+\underbrace{\bar{f}_t(x_t)-\bar{f}_t(x)}_{(\text{no-memory regret})}
\end{align*}
The affine memory structure together with regularity conditions allow us to bound each of the above terms by $\tilde{O}(\sqrt{T})$.

\subsection{Reduction from the Bandit LQ problem}
\label{subsec:reduction-control}

The \texttt{BNS-AM} algorithm and its regret guarantee for BQO-AM problems described in the previous section has significant applications in bandit LQ problems. In particular, recall the bandit LQ problem described in \cref{sec:control-prelims}. The bandit LQ problem can be indeed formulated as a BQO-AM problem up to negligible truncation errors. In particular, $\exists f_t$ satisfying Assumptions~\ref{assumption:affine-memory},\ref{assumption:diameter-grad-bound},\ref{assumption:adversary-adaptivity},\ref{assumption:convolution-invertibility-modulus} such that the control cost function $c_t$ can be expressed by $f_t$. The reduction is formally proved in Section~\ref{sec:reduction-control} and \ref{sec:regualrity-conditions-known}. This reduction makes it possible to directly adapt the \texttt{BNS-AM} algorithm to the control setting. The control algorithm, which we refer to as Newton Bandit Perturbation Controller (\texttt{NBPC}), is specified in Section~\ref{sec:nbpc-algorithm} with the following control regret guarantee. An extension of the result to unknown systems is included in Appendix~\ref{appendix:unknown}.

\begin{theorem} [\texttt{NBPC} regret, Algorithm~\ref{alg:bandit-control-known}]
\label{thm:nbpc-regret}
For $T,d_{\x},d_{\uv},d_{\y}\in\mathbb{N}$, consider a partially observable linear dynamical system (\cref{eq:lds}) with dynamics $(A,B,C)$ and state, control, and observation dimensions $d_{\x}, d_{\uv}, d_{\y}\in\mathbb{N}$ respectively, and a sequence of cost functions $\{c_t\}_{t=1}^T$ satisfying Assumptions~\ref{assumption:bounded-dynamics-and stability},\ref{assumption:adversary-adaptivity-control},\ref{assumption:control-loss-curvature}. Suppose the controller \texttt{NBPC} is run with inputs ($\M(m,R_{\M}), \eta, T, (A,B,C), K, \alpha$) with $m=\poly(\log T)$, DRC policy class $\M(m,R_{\M})$  (Definition~\ref{def:dac}), $\eta=\tilde{\Theta}\left(\frac{1}{\alpha\sqrt{T}}\right)$, strongly stabilizing controller $K$, and $\alpha>0$ given by Assumption~\ref{assumption:control-loss-curvature}. Then, \texttt{NBPC} satisfies the following regret guarantee:
\begin{align*}
\E[\cregret_T(\texttt{NBPC})]=\E\left[\sum_{t=1}^T c_t(\y_t,\uv_t)-\min_{M\in\M(m,R_{\M})}\sum_{t=1}^T c_t(\y_t^M,\uv_t^M)\right]\le \tilde{O}\left(\frac{\beta^2}{\alpha}\sqrt{T}\right). 
\end{align*}
\end{theorem}

\section{Lower Bound for BCO-M}
\label{sec:bco-m}
As mentioned in the previous section, BCO-M is hard for general convex loss functions. In this section, we turn our attention to the general problem of BCO-M. %The algorithms developed by \cite{gradu2020non, sun2023optimal} as well as Algorithm~\ref{alg:bqo-memory} work against an adversary that is $(t-m)$-adaptive at time $t$, where $m$ is the memory parameter (Assumption~\ref{assumption:adversary-adaptivity}). 
We are interested in a lower bound under Assumption~\ref{assumption:adversary-adaptivity}.
We show a regret lower bound of $\tilde{\Omega}(T^{2/3})$ for this setting even for quadratic and smooth loss functions and $m=2$. This result holds even if we allow some degree of improper learning.

Inspired by the work of~\cite{dekel2014bandits}, where they showed a $\tilde{\Omega}(T^{2/3})$-regret lower bound for multi-armed bandits with switching costs, even if the switching costs part of the loss functions are given to the learner through a full information feedback model. To establish a lower bound in our setting, we note that \cite{sun2023optimal} showed that for strongly convex quadratic functions it is possible to attain $\tilde{O}(\sqrt{T})$ regret upper bound. The main relaxation of assumption in our setting is that we remove the strong convexity assumption, thus allowing quadratic loss of the form $(x_t-x_{t-1})^2$, which has a non-invertible Hessian. This form of quadratic losses has an intuitive connection to the switching costs analyzed by~\cite{dekel2014bandits}. 

However, directly reducing from the hard case constructed by~\cite{dekel2014bandits} is insufficient: (1) we are interested in a continuous decision space, and the switching cost at $0$ is non-smooth, and (2) if we are interested in quadratic switching costs, then the switching cost decays much faster around $0$, making direct reductions from the hard instance for discrete set inferior to the lower bound of $T^{2/3}$. However, with a modified instance, we are able to recover the $T^{2/3}$ lower bound, even in the setting where we allow improper learning.

\begin{theorem} [BCO-M lower bound]
\label{thm:lower-bound}
There exists a sequence of loss functions such that the regret incurred by any (possibly randomized) online bandit quadratic optimization with memory length $m\ge 2$ algorithm $\A$ is at least $\tilde{\Omega}(T^{2/3})$.
\end{theorem}
% By (the easy direction of) Yao's minimax principle, this also implies the same lower bound against randomized algorithms as well. 

\paragraph{Proof idea.} We briefly sketch the idea of the loss construction here. The loss function at time $t$ consists of three parts: a scaled linear component with a random sign, a random process indexed at $t$, and a moving cost between two consecutive iterates. The best strategy depends on the random sign, which is chosen by the adversary ahead of time. For a sufficiently small scale on the linear component, the learner needs to move sufficiently to learn the random sign from bandit feedback and incur the corresponding moving cost, leading to the $\tilde{\Omega}(T^{2/3})$ lower bound. 

A matching upper bound for quadratic and smooth functions in the presence of $\F_{t-m}$-adaptive adversary can also be established. We believe that the improper leaning algorithm in \citep{cassel2020bandit} can also be generalized to scenarios where the adversary is allowed to be $\F_{t-m}$-adaptive. However, since this was not spelled out explicitly, we give an alternative first-order proper learning algorithm for completeness in Section~\ref{sec:bco-m-upper-bound}.
\section{Conclusions, discussions, and future work}

In this paper, we considered a bandit version of the Online Newton Step algorithm that attains near-optimal regret bounds for a large class of convex loss functions we call $\kappa$-convex with a low dimension dependence. %We find the identification of this large class of convex functions, that naturally allows efficient second order methods, to be of independent interest. 
In addition to the application to bandit convex optimization, we showed how our methods close an open problem in online control and separated the difficulty of bandit LQR/LQG and the general BCO-M problem. 
Many interesting open questions remain: 
\begin{itemize}[leftmargin=*]
\setlength{\itemsep}{1pt}
\item \textbf{Proper BCO for $\kappa$-convex losses.} \texttt{BNS} is improper. This prohibits applications to problems where proper learning is required (e.g. portfolio selection). 
It is interesting to see whether there exists efficient proper learning algorithm that achieves a $O(\sqrt{T})$ regret with reasonable dependence on the dimension for the class of $\kappa$-convex problems. 

%\item \textbf{Logarithmic factors in bandit LQR/LQG.} Our regret bound for bandit control of LQR/LQG is optimal up to logarithmic factors. It would be interesting to see whether true optimal bounds in the bandit setting can be obtained by similar methods as seen in \citep{lin2023online}. 

\item \textbf{Extension to nonstochastic control with general cost functions.} It is worth exploring if tighter bounds are achievable for bandit linear control with general convex loss functions with adversarial perturbations.

\item \textbf{Lower bound for oblivious adversary.} If the adversary is oblivious in BCO-M problems, we are interested in whether $\Omega(T^{2/3})$ regret lower bound can be achieved. 
\end{itemize}

%\section{Acknowledgements}
%EH and JS gratefully acknowledge funding from the National Science Foundation, the Office of Naval Research, and Open Philanthropy. 

\acks{EH and JS gratefully acknowledge funding from the National Science Foundation, the Office of Naval Research, and Open Philanthropy.}

\bibliography{main}

\begin{thebibliography}{43}
\providecommand{\natexlab}[1]{#1}
\providecommand{\url}[1]{\texttt{#1}}
\expandafter\ifx\csname urlstyle\endcsname\relax
  \providecommand{\doi}[1]{doi: #1}\else
  \providecommand{\doi}{doi: \begingroup \urlstyle{rm}\Url}\fi

\bibitem[Abeille and Lazaric(2017)]{abeille2017linear}
Marc Abeille and Alessandro Lazaric.
\newblock Linear thompson sampling revisited.
\newblock In \emph{Artificial Intelligence and Statistics}, pages 176--184.
  PMLR, 2017.

\bibitem[Abernethy et~al.(2009)Abernethy, Hazan, and
  Rakhlin]{abernethy2009competing}
Jacob~D Abernethy, Elad Hazan, and Alexander Rakhlin.
\newblock Competing in the dark: An efficient algorithm for bandit linear
  optimization.
\newblock 2009.

\bibitem[Agarwal et~al.(2019)Agarwal, Bullins, Hazan, Kakade, and
  Singh]{agarwal2019online}
Naman Agarwal, Brian Bullins, Elad Hazan, Sham Kakade, and Karan Singh.
\newblock Online control with adversarial disturbances.
\newblock In \emph{International Conference on Machine Learning}, pages
  111--119. PMLR, 2019.

\bibitem[Anava et~al.(2015)Anava, Hazan, and Mannor]{anava2015online}
Oren Anava, Elad Hazan, and Shie Mannor.
\newblock Online learning for adversaries with memory: price of past mistakes.
\newblock \emph{Advances in Neural Information Processing Systems}, 28, 2015.

\bibitem[Arbenz and Golub(1988)]{golub-matrix}
Peter Arbenz and Gene~H. Golub.
\newblock On the spectral decomposition of hermitian matrices modified by low
  rank perturbations with applications.
\newblock \emph{SIAM Journal on Matrix Analysis and Applications}, 9\penalty0
  (1):\penalty0 40--58, 1988.

\bibitem[Auer et~al.(2002)Auer, Cesa-Bianchi, Freund, and
  Schapire]{auer2002nonstochastic}
Peter Auer, Nicolo Cesa-Bianchi, Yoav Freund, and Robert~E Schapire.
\newblock The nonstochastic multiarmed bandit problem.
\newblock \emph{SIAM journal on computing}, 32\penalty0 (1):\penalty0 48--77,
  2002.

\bibitem[Bubeck et~al.(2017)Bubeck, Lee, and Eldan]{bubeck2017kernel}
S{\'e}bastien Bubeck, Yin~Tat Lee, and Ronen Eldan.
\newblock Kernel-based methods for bandit convex optimization.
\newblock In \emph{Proceedings of the 49th Annual ACM SIGACT Symposium on
  Theory of Computing}, pages 72--85, 2017.

\bibitem[Bubeck et~al.(2021)Bubeck, Eldan, and Lee]{bubeck2021kernel}
S{\'e}bastien Bubeck, Ronen Eldan, and Yin~Tat Lee.
\newblock Kernel-based methods for bandit convex optimization.
\newblock \emph{Journal of the ACM (JACM)}, 68\penalty0 (4):\penalty0 1--35,
  2021.

\bibitem[Bullins et~al.(2016)Bullins, Hazan, and Koren]{bullins2016limits}
Brian Bullins, Elad Hazan, and Tomer Koren.
\newblock The limits of learning with missing data.
\newblock \emph{Advances in Neural Information Processing Systems}, 29, 2016.

\bibitem[Cassel and Koren(2020)]{cassel2020bandit}
Asaf Cassel and Tomer Koren.
\newblock Bandit linear control.
\newblock \emph{Advances in Neural Information Processing Systems},
  33:\penalty0 8872--8882, 2020.

\bibitem[Cesa-Bianchi et~al.(2011)Cesa-Bianchi, Shalev-Shwartz, and
  Shamir]{cesa2011efficient}
Nicolo Cesa-Bianchi, Shai Shalev-Shwartz, and Ohad Shamir.
\newblock Efficient learning with partially observed attributes.
\newblock \emph{Journal of Machine Learning Research}, 12\penalty0 (10), 2011.

\bibitem[Cesa-Bianchi et~al.(2013)Cesa-Bianchi, Dekel, and
  Shamir]{cesa2013online}
Nicolo Cesa-Bianchi, Ofer Dekel, and Ohad Shamir.
\newblock Online learning with switching costs and other adaptive adversaries.
\newblock \emph{Advances in Neural Information Processing Systems}, 26, 2013.

\bibitem[Chen and Hazan(2021)]{chen2021black}
Xinyi Chen and Elad Hazan.
\newblock Black-box control for linear dynamical systems.
\newblock In \emph{Conference on Learning Theory}, pages 1114--1143. PMLR,
  2021.

\bibitem[Dekel et~al.(2014)Dekel, Ding, Koren, and Peres]{dekel2014bandits}
Ofer Dekel, Jian Ding, Tomer Koren, and Yuval Peres.
\newblock Bandits with switching costs: T 2/3 regret.
\newblock In \emph{Proceedings of the forty-sixth annual ACM symposium on
  Theory of computing}, pages 459--467, 2014.

\bibitem[Dong et~al.(2019)Dong, Ma, and Van~Roy]{dong2019performance}
Shi Dong, Tengyu Ma, and Benjamin Van~Roy.
\newblock On the performance of thompson sampling on logistic bandits.
\newblock In \emph{Conference on Learning Theory}, pages 1158--1160. PMLR,
  2019.

\bibitem[Faury et~al.(2022)Faury, Abeille, Jun, and
  Calauz{\`e}nes]{faury2022jointly}
Louis Faury, Marc Abeille, Kwang-Sung Jun, and Cl{\'e}ment Calauz{\`e}nes.
\newblock Jointly efficient and optimal algorithms for logistic bandits.
\newblock In \emph{International Conference on Artificial Intelligence and
  Statistics}, pages 546--580. PMLR, 2022.

\bibitem[Flaxman et~al.(2004)Flaxman, Kalai, and McMahan]{flaxman2004online}
Abraham~D Flaxman, Adam~Tauman Kalai, and H~Brendan McMahan.
\newblock Online convex optimization in the bandit setting: gradient descent
  without a gradient.
\newblock \emph{arXiv preprint cs/0408007}, 2004.

\bibitem[Foster et~al.(2018)Foster, Kale, Luo, Mohri, and
  Sridharan]{foster2018logistic}
Dylan~J Foster, Satyen Kale, Haipeng Luo, Mehryar Mohri, and Karthik Sridharan.
\newblock Logistic regression: The importance of being improper.
\newblock In \emph{Conference On Learning Theory}, pages 167--208. PMLR, 2018.

\bibitem[Gradu et~al.(2020)Gradu, Hallman, and Hazan]{gradu2020non}
Paula Gradu, John Hallman, and Elad Hazan.
\newblock Non-stochastic control with bandit feedback.
\newblock \emph{Advances in Neural Information Processing Systems},
  33:\penalty0 10764--10774, 2020.

\bibitem[Gu and Eisenstat(1994)]{gu1994stable}
Ming Gu and Stanley~C Eisenstat.
\newblock A stable and efficient algorithm for the rank-one modification of the
  symmetric eigenproblem.
\newblock \emph{SIAM journal on Matrix Analysis and Applications}, 15\penalty0
  (4):\penalty0 1266--1276, 1994.

\bibitem[Hazan(2022)]{hazan2022oco}
Elad Hazan.
\newblock \emph{Introduction to online convex optimization}.
\newblock MIT Press, 2022.

\bibitem[Hazan and Kale(2011)]{hazan2011newtron}
Elad Hazan and Satyen Kale.
\newblock Newtron: an efficient bandit algorithm for online multiclass
  prediction.
\newblock \emph{Advances in neural information processing systems}, 24, 2011.

\bibitem[Hazan and Koren(2012)]{hazan2012linear}
Elad Hazan and Tomer Koren.
\newblock Linear regression with limited observation.
\newblock \emph{arXiv preprint arXiv:1206.4678}, 2012.

\bibitem[Hazan and Levy(2014)]{hazan2014bandit}
Elad Hazan and Kfir Levy.
\newblock Bandit convex optimization: Towards tight bounds.
\newblock \emph{Advances in Neural Information Processing Systems}, 27, 2014.

\bibitem[Hazan and Singh(2022)]{hazan2022introduction}
Elad Hazan and Karan Singh.
\newblock Introduction to online nonstochastic control.
\newblock \emph{arXiv preprint arXiv:2211.09619}, 2022.

\bibitem[Hazan et~al.(2007)Hazan, Agarwal, and Kale]{hazan2007logarithmic}
Elad Hazan, Amit Agarwal, and Satyen Kale.
\newblock Logarithmic regret algorithms for online convex optimization.
\newblock \emph{Machine Learning}, 69:\penalty0 169--192, 2007.

\bibitem[Hazan et~al.(2014)Hazan, Koren, and Levy]{hazan2014logistic}
Elad Hazan, Tomer Koren, and Kfir~Y Levy.
\newblock Logistic regression: Tight bounds for stochastic and online
  optimization.
\newblock In \emph{Conference on Learning Theory}, pages 197--209. PMLR, 2014.

\bibitem[Koren et~al.(2017)Koren, Livni, and Mansour]{koren2017multi}
Tomer Koren, Roi Livni, and Yishay Mansour.
\newblock Multi-armed bandits with metric movement costs.
\newblock \emph{Advances in Neural Information Processing Systems}, 30, 2017.

\bibitem[Lattimore(2020)]{lattimore2020improved}
Tor Lattimore.
\newblock Improved regret for zeroth-order adversarial bandit convex
  optimisation.
\newblock \emph{Mathematical Statistics and Learning}, 2\penalty0 (3):\penalty0
  311--334, 2020.

\bibitem[Lattimore and Gyorgy(2021)]{lattimore2021improved}
Tor Lattimore and Andras Gyorgy.
\newblock Improved regret for zeroth-order stochastic convex bandits.
\newblock In \emph{Conference on Learning Theory}, pages 2938--2964. PMLR,
  2021.

\bibitem[Lattimore and Gy{\"o}rgy(2023)]{lattimore2023second}
Tor Lattimore and Andr{\'a}s Gy{\"o}rgy.
\newblock A second-order method for stochastic bandit convex optimisation.
\newblock \emph{arXiv preprint arXiv:2302.05371}, 2023.

\bibitem[Lattimore and Szepesv{\'a}ri(2020)]{lattimore2020bandit}
Tor Lattimore and Csaba Szepesv{\'a}ri.
\newblock \emph{Bandit algorithms}.
\newblock Cambridge University Press, 2020.

\bibitem[McMahan and Blum(2004)]{mcmahan2004online}
H~Brendan McMahan and Avrim Blum.
\newblock Online geometric optimization in the bandit setting against an
  adaptive adversary.
\newblock In \emph{Learning Theory: 17th Annual Conference on Learning Theory,
  COLT 2004, Banff, Canada, July 1-4, 2004. Proceedings 17}, pages 109--123.
  Springer, 2004.

\bibitem[McMahan and Streeter(2012)]{mcmahan2012open}
H~Brendan McMahan and Matthew Streeter.
\newblock Open problem: Better bounds for online logistic regression.
\newblock In \emph{Conference on Learning Theory}, pages 44--1. JMLR Workshop
  and Conference Proceedings, 2012.

\bibitem[Shalev-Shwartz and Singer(2007)]{shalev2007primal}
Shai Shalev-Shwartz and Yoram Singer.
\newblock A primal-dual perspective of online learning algorithms.
\newblock \emph{Machine Learning}, 69:\penalty0 115--142, 2007.

\bibitem[Shamir(2013)]{shamir2013complexity}
Ohad Shamir.
\newblock On the complexity of bandit and derivative-free stochastic convex
  optimization.
\newblock In \emph{Conference on Learning Theory}, pages 3--24. PMLR, 2013.

\bibitem[Simchowitz(2020)]{simchowitz2020making}
Max Simchowitz.
\newblock Making non-stochastic control (almost) as easy as stochastic.
\newblock \emph{Advances in Neural Information Processing Systems},
  33:\penalty0 18318--18329, 2020.

\bibitem[Simchowitz et~al.(2020)Simchowitz, Singh, and
  Hazan]{simchowitz2020improper}
Max Simchowitz, Karan Singh, and Elad Hazan.
\newblock Improper learning for non-stochastic control.
\newblock In \emph{Conference on Learning Theory}, pages 3320--3436. PMLR,
  2020.

\bibitem[Suggala et~al.(2021)Suggala, Ravikumar, and
  Netrapalli]{suggala2021efficient}
Arun~Sai Suggala, Pradeep Ravikumar, and Praneeth Netrapalli.
\newblock Efficient bandit convex optimization: Beyond linear losses.
\newblock In \emph{Conference on Learning Theory}, pages 4008--4067. PMLR,
  2021.

\bibitem[Sun et~al.(2023)Sun, Newman, and Hazan]{sun2023optimal}
Y~Jennifer Sun, Stephen Newman, and Elad Hazan.
\newblock Optimal rates for bandit nonstochastic control.
\newblock \emph{arXiv preprint arXiv:2305.15352}, 2023.

\bibitem[Tropp(2011)]{tropp2011freedman}
Joel Tropp.
\newblock Freedman's inequality for matrix martingales.
\newblock 2011.

\bibitem[Vershynin(2018)]{vershynin2018high}
Roman Vershynin.
\newblock \emph{High-dimensional probability: An introduction with applications
  in data science}, volume~47.
\newblock Cambridge university press, 2018.

\bibitem[Zimmert and Lattimore(2022)]{zimmert2022return}
Julian Zimmert and Tor Lattimore.
\newblock Return of the bias: Almost minimax optimal high probability bounds
  for adversarial linear bandits.
\newblock In \emph{Conference on Learning Theory}, pages 3285--3312. PMLR,
  2022.

\end{thebibliography}

\newpage
\tableofcontents
\newpage

\appendix
\section{Discussion on $\kappa$-convexity}
\label{sec:kappa-convex}
In this section, we provide a more elaborated discussion on the curvature assumption of $\kappa$-convexity (Definition~\ref{def:kappa-c}, Section~\ref{sec:contributions-bco}). It is easy to see that $\kappa$-convexity is implied by linearity (by taking $H=0$ in Definition~\ref{def:kappa-c}) and by strong convexity together with smoothness (by taking $H=I$, and $c,C$ to be the strong convexity and smoothness parameters). 

A common relaxation of strong convexity used in the literature of Newton step based algorithms to obtain optimal rates is exp-concavity, which requires strong convexity in the direction of the gradient. The class of functions that satisfy exp-concavity is formally given by the following definition. 
\begin{definition}[Exp-concave functions, \citep{hazan2022oco}]
A convex function $f:\mathbb{R}^n\rightarrow\mathbb{R}$ is $\alpha$-exp-concave over a convex compact set $\K\subset\mathbb{R}^n$ if $g:\K\rightarrow\mathbb{R}$ given by $g(x)=e^{-\alpha f(x)}$ is concave. Equivalently, if $f$ is twice differentiable, $\nabla^2 f(x)\succeq \alpha\nabla f(x)\nabla f(x)^{\top}$ holds for all $x\in\K$. 
\end{definition}

The motivation for considering exp-concave functions goes beyond their appealing curvature properties, which enable optimal regret guarantees. Many popular losses that do not satisfy strong convexity are exp-concave, such as the squared loss used in linear regression, logistic loss for classification, and logarithmic loss for portfolio optimization. 

Unfortunately, the conditions of exp-concavity and $\kappa$-convexity are not directly comparable: linear functions do not satisfy exp-concavity, while $\kappa$-convexity requires smoothness assumption and uniform bounds on the Hessian over $\K$. Luckily, $\kappa$-convexity does capture a rich class of functions whose curvature properties lie between convexity and strong convexity, encompassing the most popular exp-concave loss considered in practice. To see this, first we establish Observation~\ref{obs:kappa-c-expressive-d}, which states $\kappa$-convexity for the class of functions that can be written as a composition of an affine function and a strongly convex and smooth function. 

\begin{observation}
\label{obs:kappa-c-expressive-d}
Let $\K\subset\mathbb{R}^n$ be a convex compact set. For $n,d\in\mathbb{N}$, $\alpha,\beta>0$,
consider the class of functions composed of an affine function and a strongly convex and smooth function
\begin{align}
\label{eq:nice-losses}
\mathcal{L}_{\alpha,\beta,n,d}(\K)=\{f:\mathbb{R}^n\rightarrow\mathbb{R}\mid &f(x)=g\circ\ell (x),  \ \ell\in\mathcal{C}^2:\mathbb{R}^n\rightarrow\mathbb{R}^d, \ \nabla^2\ell=0, \nonumber \\
&g\in\mathcal{C}^2:\mathbb{R}^d\rightarrow\mathbb{R}, \ \alpha I_d\preceq \nabla^2 g(z)\preceq \beta I_d, \ \forall z\in\ell(\K)\}.
\end{align}
$f\in\mathcal{L}_{\alpha,\beta,n,d}(\K)$ is $\kappa$-convex on $\K$ with $\kappa=\frac{\beta}{\alpha}$ and $H=\nabla\ell\nabla\ell^{\top}$. 
\end{observation}

In Observation~\ref{obs:examples}, we further note that many popular loss functions of interests are described by the function class in Observation~\ref{obs:kappa-c-expressive-d}. 

\begin{observation}
\label{obs:examples}
The following loss functions belong to $\mathcal{L}_{\alpha,\beta,n,d}(\K)$ defined in \cref{eq:nice-losses} for any convex compact domain $\K\subset\mathbb{R}^n$ and bounded $\mathcal{X}\subset\mathbb{R}^n, \mathcal{Y}\subset\mathbb{R}$:
\begin{itemize}
\item (Squared loss in linear regression) $w\mapsto (y-w^{\top}x)^2$, $x\in\mathcal{X}, y\in\mathcal{Y}$.
\item (Logistic loss in classification) $w\mapsto \log(1+\exp(-y\cdot x^{\top}w))$, $x\in\mathcal{X}, y\in\mathcal{Y}$.
\item (Logarithmic loss in portfolio optimization) $w\mapsto -\log(w^{\top}x)$, $x\in\mathcal{X}$.
\end{itemize}
\end{observation}
In \cref{sec:applications}, we will further discuss the implications of our improved bandit algorithm on the three examples described in Observation~\ref{obs:examples}, including comparisons with previous results in these problems.

\section{Applications of Bandit Newton Step } \label{sec:applications}
In this section, we illustrate the applicability of our bandit Newton method (Algorithm~\ref{alg:simple-bqo}) to a few problems where it achieves optimal regret guarantees. First, we acknowledge that since the algorithm in \citet{bubeck2021kernel} works for a richer class of loss functions than ours, applying their algorithm will also give the optimal regret bounds in terms of $T$ in these applications. However, the high polynomial dependence on the dimension ($d^{9.5}$) and the $\mathrm{poly}(d,T)$ amortized running time make the algorithm by \cite{bubeck2021kernel} impractical in many real-world applications. Therefore, for simplicity, we are going to exclude the comparison with \cite{bubeck2021kernel} in this section.

\subsection{Online Logistic Regression}
\label{sec:logistic-regression}
Logistic regression is a widely utilized and effective technique for performing classification tasks. For simplicity, we consider binary classification. In this framework, the goal is to predict the probability that a given input point belongs to one of two possible classes. These classes are typically represented by the label set $\Y=\{-1,1\}$. The foundational assumption of logistic regression is that the log-odds, or the logarithm of the odds ratio between the two classes, can be linearly modeled by the input features. Mathematically, this relationship is expressed as:
\begin{align*}
\log\left(\frac{\mathbb{P}(Y=1\mid X=x)}{1-\mathbb{P}(Y=1\mid X=x)}\right)=w^{\top}x+b,
\end{align*}
where $x$ denotes the feature vector, $w$ represents the weight vector, and $b$ is the bias term. This formulation provides a linear decision boundary in the feature space.

In the online setting of logistic regression, the learning process is sequential and adaptive. At each time step $t$, an adversary presents a new instance, consisting of a feature vector and a label $(x_t,y_t)\in\mathbb{R}^d\times \{-1,1\}$. Subsequently, the learning algorithm selects a linear predictor $w_t$ from a set $\mathcal{W}=\{w\in\mathbb{R}^d:\|w\|_2\le D\}$, where $\|w\|_2\le D$ ensures that the chosen predictor does not have excessively large weights, thereby controlling the model's complexity. 
The performance of the predictor $w_t$ is then evaluated based on the incurred loss, which, for logistic regression, is defined as:
\begin{align*}
f_t(w_t)=f(w_t;x_t, y_t)=\log(1+\exp(-y_t\cdot x_t^{\top}w_t)).
\end{align*}
In the bandit framework, where the sequence of labeled data may be confidential and not directly observable by the learner, the learner is only exposed to a scalar classification cost, denoted as $f_t(w_t)$.
It is often assumed for simplicity that $\|x_t\|_2\le 1$, $\forall t$, in which case the Hessian of $f_t$ satisfies
\begin{align*}
\frac{e^{-D}}{(1+e^{-D})^2}x_tx_t^{\top}\preceq \nabla^2 f_t(w_t) = \frac{\exp(-y_t\cdot x_t^\top w_t)} {(1 + \exp(-y_t\cdot x_t^\top w_t))^2} \cdot x_tx_{t}^{\top}\preceq \frac{1}{2}x_tx_t^{\top}.
\end{align*}
This shows that $f_t$ is $\kappa$-convex with $\kappa=\exp(D)$, making it suitable to apply Algorithm~\ref{alg:simple-bqo}. 

\begin{corollary}[Regret bound for bandit logistic regression] 
\label{cor:logistic-regression-regret}
Suppose a learner is playing over a decision set $\mathcal{W}$ with diameter bound $D$ against an oblivious adversary picking a labeled vector $(x_t,y_t)\in\mathbb{R}^d\times \{-1,1\}$ with $\|x_t\|_2\le 1$, then Algorithm~\ref{alg:simple-bqo} guarantees
\begin{align*}
\E[\regret_T]=\E\left[\sum_{t=1}^T f(w_t; x_t, y_t) - \min_{w\in\mathcal{W}}\sum_{t=1}^T f(w; x_t,y_t)\right]\le O(d^{2.5}e^{2D}\sqrt{T}).
\end{align*}
\end{corollary}

Let's take a moment to compare the result of Corollary~\ref{cor:logistic-regression-regret} with existing results. To the best of our knowledge, this is the \textit{first} bandit logistic regression algorithm that works against adversarially chosen labeled vectors that achieves an \textit{optimal} regret bound in $T$ and $d^{5/2}$ dependence on $d$. 
Table~\ref{tab:logistic_regression_comparisons} summarized some of the state-of-art results in the field. First, let's focus on the logistic bandits setting. \cite{hazan2011newtron} and \cite{foster2018logistic} both studied the problem which they refer to as bandit multi-class prediction, which we refer to as the semi-bandit feedback model in Table~\ref{tab:logistic_regression_comparisons}. In this setting, although the learner accesses loss through bandit feedback, the vector $x_t$ is revealed to the learner. The more general case where both the vector and the label are given to the learner through bandit feedback was studied only under stochastic assumption on the input vector, in which case $\tilde{O}(e^D\sqrt{T})$ and $\tilde{O}(\sqrt{T})$ bounds were achieved by \cite{abeille2017linear} and \cite{dong2019performance} for frequentist and Bayesian regret bounds, respectively.

One question remaining is regarding the exponential dependence on $D$ in the bound of Corollary~\ref{cor:logistic-regression-regret}. To shed some light on this question, we consider the full information setting. \cite{hazan2007logarithmic} showed that running online newton step gives the optimal dependence ($\log T$) in $T$ but suffers an exponential factor in terms of the diameter $D$, leading to the open problem of whether a $\tilde{O}(\mathrm{poly}(D))$ regret bound is attainable in \cite{mcmahan2012open}. Later, \cite{hazan2014bandit} showed a negative result that for $d\ge 2$, the regret for this problem is at least $\tilde{\Omega}(e^D\vee \sqrt{DT})$ if the algorithm is proper, i.e. $w_t\in\mathcal{W}$. \cite{foster2018logistic} showed that when allowing improper learning, $\tilde{O}(1)$-regret bound is attainable. It will be an interesting open problem to investigate whether the exponential dependence on $D$ can be dropped in logistic bandits. 

\begin{table*}[hbt!]
	\centering
	\small % Reducing the font size
	{
		\begin{tabular}{c||c|c|>{\centering\arraybackslash}p{1cm}|c|c|c} % Adjust the width (2cm) as needed
			\hline
			Paper & Feedback & Advers. & Proper & Regret & Comp. & Note  %&%
			\\ [0.5ex]
			\hline\hline
            \cite{hazan2007logarithmic}  & full  &   $\checkmark$ & $\checkmark$  & $\tilde{O}(e^D)$ & $O(d^2)$
			\\\hline
            \cite{hazan2014logistic}  & full  &   $\times$ & $\checkmark$  & $\tilde{\Omega}(e^{D}\vee \sqrt{DT})$ & --
			\\\hline
			\cite{foster2018logistic}  & full  &  $\checkmark$ & $\times$  & $\tilde{O}(1)$  & $\mathrm{poly}(d,T)$
			\\\hline
            \cite{hazan2011newtron}  & semi-bandit  & $\checkmark$ & $\checkmark$ & $\tilde{O}(e^D \wedge DT^{2/3})$  & $O(d^2)$
			\\\hline
            \cite{foster2018logistic}  & semi-bandit  & $\checkmark$ & $\times$ & $\tilde{O}(e^D \wedge \sqrt{T})$ & $\mathrm{poly}(d,T)$
			\\\hline
            \cite{dong2019performance}  & bandit  & $\times$ & $\checkmark$ & $\tilde{O}(\sqrt{T})$  & $\mathrm{poly}(d)$ & Bayesian
			\\\hline
             \cite{faury2022jointly}  & bandit  & $\times$ & $\checkmark$ & $\tilde{O}(e^D \vee \sqrt{T})$  & $O(d^2)$ & frequentist
			\\\hline
			\textbf{Corollary~\ref{cor:logistic-regression-regret}} &  bandit & $\checkmark$ & $\times$ & $O(e^{2D}\sqrt{T})$& $O(d^{2})$
			\\\hline
		\end{tabular}
	}
 \caption{Comparison with relevant prior works for online logistic regression. $\tilde{O}, \tilde{\Omega}$ in the regret column hide all parameters other than $D,T$ and logarithmic factors in $D,T$.}
	\label{tab:logistic_regression_comparisons}
\end{table*}

\subsection{Online Linear Regression with Limited Observations}
Consider the standard online learning setting for linear regression problems. At each time $t$, the learner is supplied with a vector $x_t\in\mathcal{X}\subset\mathbb{R}^d$ and is asked to output a weight vector $w_t\in\mathcal{W}\subset\mathbb{R}^d$ such that the learner's label prediction for $x_t$ is $\hat{y}_t=w_t^{\top}x_t\in \mathbb{R}$. After the learner outputs $w_t$, the adversary reveals the true label $y_t$, to which the learner suffers loss $f(w_t;x_t)=(\hat{y}_t-y_t)^2$. 

Previous works such as \citep{cesa2011efficient, hazan2012linear, bullins2016limits} have considered this framework in the setting of limited observations. In the limited observation setting, the learner observes only a subset of the attributes in $x_t$. We consider a more generalized version of the problem, where the learner does not necessarily observe $(x_t,y_t)$ at all but simply a scalar loss of $f(w_t;x_t)\in\mathbb{R}_+$. Note that the loss is $\kappa$-convex with $\kappa=1$.

\begin{corollary} [Regret bound for bandit linear regression]
Suppose a learner is playing over a decision set $\mathcal{W}$ with diameter bound $D_{\mathcal{W}}$ against an oblivious adversary picking a labeled vector $(x_t,y_t)$, then Algorithm~\ref{alg:simple-bqo} guarantees
\begin{align*}
\E[\regret_T]=\E\left[\sum_{t=1}^T f(w_t;x_t)-\min_{w\in\mathcal{W}}\sum_{t=1}^T f(w;x_t)\right]\le O(d^{2.5}\sqrt{T}).
\end{align*}
\end{corollary}

\section{Online Newton Step in Full Information Setting}
In this section, we provide the preliminaries on the online Newton step (\texttt{ONS}) algorithm in the full information setting. This will serve as a building block to derive the regret guarantee of our proposed Newton step based bandit algorithm (\cref{alg:simple-bqo}). 

Let $\K\subset\mathbb{R}^d$ be a convex compact set. \texttt{ONS} chooses each iterate through Newton-step descent with a preconditioner $A_t$ set to be the scaled cumulative Hessian of the loss functions received so far.
\begin{algorithm}
\caption{Online Newton Step (\texttt{ONS})}
\label{alg:ons_full_information}
\begin{flushleft}
  {\bf Input:} convex compact set $\K\subset\mathbb{R}^d$, step size $\eta>0$, Hessian multiplier $\kappa'>0$, time horizon $T\in\mathbb{N}$.
\end{flushleft}
\begin{algorithmic}[1]
\STATE Initialize: $x_1 \in\K$, $A_0 = I $.
\FOR{$t = 1, \ldots, T$}
\STATE Play $x_t$, observe $f_t$.
% \STATE Estimate gradient and Hessians $\tilde{\nabla}_t \leftarrow \O_g(f_t),\tilde{H}_t \leftarrow \O_h(f_t)$
\STATE  Update $A_{t} = A_{t-1} + \frac{\eta}{\kappa'} \nabla^2f_t(x_t) $, and compute $x_{t+1}$
\begin{align*}
    &x_{t+1} = \prod_{\K}^{A_t} \left[ x_t - \eta A_t^{-1}  \nabla f_t(x_t)  \right]
\end{align*}
\ENDFOR 
\end{algorithmic}
\end{algorithm}

The following theorem bounds the regret of \texttt{ONS} against any single $x\in\K$, as a function of the step size $\eta$, the diameter of $\K$, and the curvature parameter of the sequence of loss functions.

\begin{theorem}[\texttt{ONS} Full Information Regret]
\label{thm:ons_full_info_regret}
Suppose that the Online Newton Step (Algorithm~\ref{alg:ons_full_information}) with input $(\K,\eta,\kappa',T)$ applied to a  sequence of loss functions $\{f_t\}_{t=1}^T$ that are twice differentiable. Moreover, suppose $\{A_t\}_{t=1\dots T}$, the cumulative Hessians, are invertible. 
Then, Algorithm~\ref{alg:ons_full_information}  guarantees the following regret upper bound for any $x\in \K$, 
\[
\sum_{t=1}^T f_t(x_t) - f_t(x) \leq \frac{\text{diam}(\K)^2}{2\eta} + \sum_{t=1}^T \frac{\eta}{2} \norm{\nabla f_t(x_t)}_{A_t^{-1}}^2 - \sum_{t=1}^T\Delta_t(x;\kappa').
\]
Here, $\Delta_t(x;\kappa')$ is defined as
\begin{align*}
\Delta_t(x;\kappa') \coloneqq f_t(x) - f_t(x_t) - \inp{\nabla f_t(x_t)}{x - x_t} - \frac{1}{2\kappa'}(x-x_t)^{\top}\nabla^2f_t(x_t)(x-x_t).
\end{align*}
\end{theorem}

\begin{proof}
    For any $x\in \K$, we have that
\begin{align*}
    \norm{x_{t+1} - x}_{A_t}^2 &\stackrel{(a)}{\le} \norm{x_t - \eta A_t^{-1}\nabla f_t(x_t) - x}_{A_t}^2 \\
    &= \norm{x_t - x}_{A_t}^2 - 2\eta \inp{A_t^{-1}\nabla f_t(x_t)}{x_t - x}_{A_t} + \eta^2 \norm{A_t^{-1}\nabla f_t(x_t)}_{A_t}^2 \\
    &= \norm{x_t - x}_{A_{t}}^2 - 2\eta  \inp{\nabla f_t(x_t)}{x - x_t} + \eta^2 \norm{\nabla f_t(x_t)}_{A_t^{-1}}^2  \\
    &= \norm{x_t - x}_{A_{t-1}}^2 + 2\eta \left(\frac{1}{2\kappa'}\norm{x_t - x}_{\nabla^2f_t(x_t)}^2  + \inp{\nabla f_t(x_t)}{x - x_t}\right) + \eta^2 \norm{\nabla f_t(x_t)}_{A_t^{-1}}^2.
\end{align*}
Where $(a)$ is due to the projection property.
Now, consider the second term in the right hand side of the last equality. By definition of $\Delta_t(x;\kappa')$, we can rewrite it as
\begin{align*}
     \frac{1}{2\kappa'}\norm{x_t - x}_{\nabla^2f_t(x_t)}^2  + \inp{\nabla f_t(x_t)}{x - x_t} 
     &= f_t(x) - f_t(x_t) - \Delta_t(x;\kappa')
\end{align*}
Substituting this in the previous display, and rearranging the terms gives us
\[
f_t(x_t) - f_t(x) \leq \frac{\norm{x_t - x}_{A_{t-1}}^2 - \norm{x_t - x}_{A_{t}}^2}{2\eta} + \frac{\eta}{2} \norm{\nabla f_t(x_t)}_{A_t^{-1}}^2 - \Delta_t(x;\kappa').
\]
Summing this over $t = 1 \dots T$ gives us 
\[
\sum_{t=1}^T f_t(x_t) - f_t(x) \leq \frac{\norm{x_1 - x}_2^2}{2\eta} + \sum_{t=1}^T \frac{\eta}{2} \norm{\nabla f_t(x_t)}_{A_t^{-1}}^2 - \sum_{t=1}^T\Delta_t(x;\kappa').
\]
We now use the fact that $\norm{x_1 - x}_2 \leq \text{diam}(\K)$ to obtain the desired bound.
\end{proof}

Using the regret bound established in \cref{thm:ons_full_info_regret}, we will derive the regret inequality for the bandit Newton based algorithm through a reduction from the full information setting.

\section{Regret of Bandit Newton Method for Improper Learning (Section~\ref{sec:improper-learning})}
In this section, we first provide a result which converts any online second order algorithm in full information setting to one that uses stochastic gradients and Hessians estimators in the bandit setting. In Section~\ref{sec:apx_improper_bco_proof}, we rely on this result to bound the regret of Algorithm~\ref{alg:simple-bqo}.

We first begin by formally defining the class of second order online convex optimization (OCO) algorithms - the family of regret minimization algorithms for which this reduction works - in Definition~\ref{def:second-order-oco}. 

\begin{definition} [Second order OCO algorithm]
\label{def:second-order-oco}
Let $\A$ be a deterministic online convex optimization algorithm on $\K\subset\mathbb{R}^d$ receiving an arbitrary sequence of $T\in\mathbb{N}$ twice differentiable loss functions $f_1,\dots,f_T:\mathbb{R}^d\rightarrow\mathbb{R}$ and producing decisions $x_1\leftarrow \A(\emptyset), \dots, x_t\leftarrow\A(f_1,\dots,f_{t-1})$. $\A$ is called a second order online algorithm if the following holds:
\begin{itemize}
\item Let $\hat{f}_t:\mathbb{R}^d\rightarrow\mathbb{R}$ be the quadratic function defined as
\begin{align*}
\hat{f}_t(x)=\frac{1}{2}x^{\top}\nabla^2 f_t(x_t)x +\nabla f_t(x_t)^{\top}x-x_t^{\top}\nabla^2f_t(x_t)x.
\end{align*}
Then $\forall t\in[T]$:
\begin{align*}
\A(f_1,\dots,f_{t-1})=\A(\hat{f}_1,\dots,\hat{f}_{t-1}).
\end{align*}
\end{itemize}
\end{definition}

By Definition~\ref{def:second-order-oco}, it is clear that the iterates $x_t$ of a second order OCO algorithm $\A$ is completely determined by the gradients and Hessians of the loss functions of the previous iterations at the corresponding decision points. Therefore, we can rewrite
\begin{align*}
x_t\leftarrow \A(\nabla f_1(x_1),\dots, \nabla f_{t-1}(x_{t-1}), \nabla^2 f_1(x_1),\dots, \nabla^2 f_{t-1}(x_{t-1})).
\end{align*}

We consider a formal reduction from any second order OCO algorithm to an algorithm in the bandit setting. 

\begin{lemma}[Bandit Reduction]
\label{lem:stoc_ons_full_info_regret}
Let $\K\subset \mathbb{R}^d$ be a convex compact set. Let $\A$ be a second order online algorithm (Definition~\ref{def:second-order-oco}) on $\K$ that ensures a regret bound with respect any $x\in\K$ of the following form for any sequence of twice differentiable loss functions $\{f_t\}_{t=1}^T$ over a time horizon $T\in\mathbb{N}$:
    \[
    \sum_{t=1}^T f_t(x_t) - f_t(x) \leq B_{\A, x}(f_1, \dots f_T).
    \]
    Define the points $\{x_t\}_{t=1}^T$ as: $x_1 \leftarrow \A(\emptyset),$ $x_t \leftarrow \A(\tilde{\nabla}_1, \dots \tilde{\nabla}_{t-1}, \tilde{H}_1, \dots \tilde{H}_{t-1})$, i.e. $x_t$ is given by the output of $\A$ given gradients $\tilde{\nabla}_1,\dots,\tilde{\nabla}_{t-1}$ and Hessians $\tilde{H}_1,\dots, \tilde{H}_{t-1}$, where $\tilde{\nabla}_t, \tilde{H}_t$ are (conditionally) unbiased estimators of the gradient and Hessian of $f_t$ at $x_t$, respectively, i.e.
    \begin{align*}
    \EBB[\tilde{\nabla}_t|x_1,f_1,\dots x_t, f_t] &= \nabla f_t(x_t),\\
    \EBB[\tilde{H}_t|x_1,f_1,\dots x_t, f_t] &= \nabla^2 f_t(x_t).
    \end{align*}
    Define the stochastic approximations $h_t:\K\rightarrow\mathbb{R}$ of $f_t:\K\to\RBB$ as follows:
    \[
    h_t(x) = \inp{\tilde{\nabla}_t}{x} + \frac{1}{2}(x-x_t)^{\top}\tilde{H}_t (x-x_t) 
    \]
    %\ \kappa \nabla^2 f_t(x_t) \preceq \EBB[\tilde{H}_t|x_1,f_1,\dots x_t, f_t] \preceq \frac{1}{\kappa} \nabla^2 f_t(x_t),
    Then, the following holds for any $x \in \K$:
    \[
    \sum_{t=1}^T \EBB\left[f_t(x_t) - f_t(x)\right] \leq \EBB\left[B_{\A, x}(h_1, \dots h_T)\right] - \sum_{t=1}^T \EBB[\Delta_t(x)],
    \]
    where $\Delta_t(x)$ is the error in second order Taylor series expansion of $f_t$ which is defined as
    \[
    \Delta_t(x) \coloneqq f_t(x) - f_t(x_t) - \inp{\nabla f_t(x_t)}{x-x_t} -  \frac{1}{2}(x-x_t)^{\top}\EBB[\tilde{H}_t|x_1, f_1, \dots x_t, f_t] (x-x_t).
    \]
\end{lemma}
\begin{proof}
    Observe that from the definition of $h_t$, we have $\nabla h_t(x_t) = \tilde{\nabla}_t, \nabla^2 h_t(x_t) = \tilde{H}_t$. So, deterministically applying a second order algorithm $\A$ on $h_t$ is equivalent to stochastically applying $\A$ on functions $f_t$. So by the regret assumption on $\A$, we have
    \begin{align}
    \label{eqn:stochastic_ons_regret_intd}
        \sum_{t=1}^T h_t(x_t) - h_t(x) \leq B_{\A, x}(h_1, \dots h_T).
    \end{align}
    Next, note that
    \begin{align*}
        \EBB[h_t(x_t)] \stackrel{(a)}{=}  \EBB\left[\inp{\EBB[\tilde{\nabla}_t|x_1,f_1,\dots x_t, f_t]}{x_t}\right] = \EBB\left[\inp{\nabla f_t(x_t)}{x_t}\right],
    \end{align*}
    where we used the fact that $\tilde{\nabla}_t$ is an unbiased estimate of the true gradient in $(a)$. A similar argument shows that
    \begin{align*}
        \EBB[h_t(x)] = \EBB\left[\inp{\nabla f_t(x_t)}{x}\right] + \frac{1}{2}\EBB\left[(x-x_t)^{\top}\EBB[\tilde{H}_t|x_1, f_1, \dots x_t, f_t] (x-x_t)\right].
    \end{align*}
    So, we have
    \begin{align*}
        \EBB[h_t(x_t) - h_t(x)] &= \EBB\left[\inp{\nabla f_t(x_t)}{x_t-x} \right] - \frac{1}{2}\EBB\left[(x-x_t)^{\top}\EBB[\tilde{H}_t|x_1, f_1, \dots x_t, f_t] (x-x_t)\right]\\
        & = \EBB[f_t(x_t) - f_t(x) + \Delta_t(x)].
    \end{align*}
    The Lemma now follows by taking expectations in Equation~\eqref{eqn:stochastic_ons_regret_intd}.
\end{proof}

Together with Theorem~\ref{thm:ons_full_info_regret}, we are almost ready to establish the regret guarantee for \cref{alg:simple-bqo}. Note that \cref{alg:simple-bqo} uses unbiased Hessian estimators $\tilde{H}_t$ (Line~\ref{line:est}) that is not necessarily positive semidefinite. However, we will show that the preconditioner $\tilde{A}_t$, the cumulative Hessian, concentrates well around its mean, making the operation in Line~\ref{line:newton-update} of \cref{alg:simple-bqo} well-defined. 

\subsection{Concentration of Cumulative Hessian Estimate}

The following lemma shows that the cumulative Hessian estimators concentrates around its mean with high probability.

\begin{lemma} [Concentration of cumulative Hessian estimate]
\label{lem:concentration-cumulative-hessian}
Consider a sequence of functions $\{f_t\}_{t=1}^T$ that satisfies Assumption~\ref{assumption:oblivious-adversary}, Assumption\ref{assumption:curvature}, and Assumption~\ref{assumption:bounded-range-and-grads}. $\forall t\in[T]$, define the smoothed function $\smoothfBB{t}$ of $f_t$ as
\[
\smoothfBB{t}(x) \coloneqq \E_{u\sim \ball, v \sim \ball}\left[f_t\left(x + \frac{1}{2} \tilde{A}_{t-1}^{-\frac{1}{2}}(u+v)\right)\Big| \FM_{t-1}\right],
\]
where $\mathcal{F}_t$ denotes the filtration generated by the algorithm's possible randomness up to time $t$. 
Let $\tilde{A}_t=\tilde{A}_{t-1}+\frac{\eta}{\kappa'}\tilde{H}_{t}$, and $A_{t}=A_{t-1}+\frac{\eta}{\kappa'}\nabla^2\smoothfBB{t}(x_t)$, where $A_0=\tilde{A}_{0}=I$. Let $B^*=B+\sqrt{2}(L+\sqrt{2}C)$. Suppose $T = \tilde{\Omega}(d)$, and $\eta\le \kappa'(24 d^{3/2}B^*\kappa\sqrt{T\log(dT^2)})^{-1}$, $\forall t$, with probability at least $1-\frac{t}{T^2}$, for every $s\le t$, 
\begin{align*}
\|I-A_{s}^{-\frac{1}{2}}\tilde{A}_sA_{s}^{-\frac{1}{2}}\|_2\le \frac{1}{2}.
\end{align*}
\end{lemma}
\begin{proof}
First observe that using Stokes' theorem~\citep{flaxman2004online}, we have
\[
\E[\tilde{H}_t\mid \FM_{t-1}] = \nabla^2 \smoothfBB{t}(x_t)
\]
where $\FM_t$ denotes the filtration generated by $\{v_{s,1},v_{s,2}\}_{s=1}^{t}$. 

\paragraph{Proof by induction.} Base case $t=0$ is given by construction. Suppose the inequalities hold with probability at least $1-\frac{t-1}{T^2}$ for any $0\le s\le t-1$. Recall that $f_t$ is assumed to be bounded by $B$ and $L$-Lipschitz over $\K$ and $C$-smooth. Note that $|f_t(y_t)|$ is bounded by
\begin{align*}
|f_t(y_t)|&\le |f_t(x_t)|+\left|f_t\left(x_t+\frac{1}{2}\tilde{A}_{t-1}^{-\frac{1}{2}}(v_{t,1}+v_{t,2})\right)-f_t(x_t)\right|\\
&\le B +  \left\|\nabla f_t\left(x_t+\frac{1}{2}\tilde{A}_{t-1}^{-\frac{1}{2}}(v_{t,1}+v_{t,2})\right)\right\|_2 \left\|\frac{1}{2}\tilde{A}_{t-1}^{-\frac{1}{2}}(v_{t,1}+v_{t,2})\right\|_2\\
&\le B+\left(L+C\left\|\frac{1}{2}\tilde{A}_{t-1}^{-\frac{1}{2}}(v_{t,1}+v_{t,2})\right\|_2\right)\left\|\frac{1}{2}\tilde{A}_{t-1}^{-\frac{1}{2}}(v_{t,1}+v_{t,2})\right\|_2\\
&\le B+(L+C\|\tilde{A}_{t-1}^{-\frac{1}{2}}\|_2)\|\tilde{A}_{t-1}^{-\frac{1}{2}}\|_2.
\end{align*} 
By induction hypothesis, with probability at least $1-\frac{t-1}{T^2}$, $\tilde{A}_{t-1}\succeq \frac{1}{2}A_{t-1} = \frac{1}{2}(I + \eta/\kappa' \sum_{s=1}^{t-1}\E_{s-1}[\tilde{A}_s])\succeq \frac{1}{2}I$, and thus the above inequality implies $|f_t(y_t)|\le B+\sqrt{2}\left(L+\sqrt{2}C\right)=:B^*$. 
Next, define $L_t$ and $U_t$ as follows
\[
L_t \coloneqq I + \frac{c\eta}{\kappa'} \sum_{s=1}^t H_s , \quad U_t \coloneqq I + \frac{C\eta}{\kappa'} \sum_{s=1}^t H_s.
\]
Observe that from our definition of $\kappa$-convex losses, $A_t$ can be lower and upper bounded as follows
\[
L_t \preceq A_t \preceq U_t.
\]
Now consider the following
\begin{align*}
    \|I-A_{t}^{-\frac{1}{2}}\tilde{A}_tA_{t}^{-\frac{1}{2}}\|_2 &= \| A_{t}^{-\frac{1}{2}}(A_{t} - \tilde{A}_{t})A_{t}^{-\frac{1}{2}}\|_2\\
    & \leq \|A_{t}^{-\frac{1}{2}}L_t^{1/2}\|_2^2\|L_{t}^{-\frac{1}{2}}(A_{t} - \tilde{A}_{t})L_{t}^{-\frac{1}{2}}\|_2\\
    & \stackrel{(a)}{\leq} \|L_{t}^{-\frac{1}{2}}(A_{t} - \tilde{A}_{t})L_{t}^{-\frac{1}{2}}\|_2,
\end{align*}
where $(a)$ follows from the fact that $L_t \preceq A_t$. This shows that it suffices to bound $\|L_{t}^{-\frac{1}{2}}(A_{t} - \tilde{A}_{t})L_{t}^{-\frac{1}{2}}\|_2$. Next, consider the following
\begin{align*}
L_{t}^{-\frac{1}{2}}(A_{t} - \tilde{A}_{t})L_{t}^{-\frac{1}{2}}&=\frac{\eta}{\kappa'} \sum_{s=1}^t L_{t}^{-\frac{1}{2}}(\E[\tilde{H}_s\mid \FM_{s-1}]-\tilde{H}_s)L_{t}^{-\frac{1}{2}},
\end{align*}
where $Z_s=L_{t}^{-\frac{1}{2}}(\E[\tilde{H}_s\mid \FM_{s-1}]-\tilde{H}_s)L_{t}^{-\frac{1}{2}}$ forms a martingale sequence  with respect to the filtration $\FM_{s}$. We  rely on matrix Freedman inequality to bound $\|\sum_{s=1}^t Z_s\|$~\citep{tropp2011freedman}. To do this, we derive bounds for the first and second moments of $Z_s$.
%\js{This does not form a martingale sequence for functions that are not quadratic.} 

\paragraph{Bounding $Z_s$.} We first show that $Z_s$ is a bounded random variable. To see this, note that by definition of $\tilde{H}_t$, and the fact that  $-2I\preceq v_{t,1}v_{t,2}^{\top}+v_{t,2}v_{t,1}^{\top}\preceq 2I$, we have
\begin{align*}
&-4d^2B^*\tilde{A}_{s-1}\preceq \tilde{H}_s\preceq 4d^2B^*\tilde{A}_{s-1},\\
&-4d^2B^*\tilde{A}_{s-1}\preceq \E[\tilde{H}_s\mid \FM_{s-1}]\preceq 4d^2B^*\tilde{A}_{s-1}.
\end{align*}
Thus, with probability at least $1-\frac{t-1}{T^2}$, $\forall s\le t$,
\begin{align*}
Z_s&\preceq 8d^2B^*L_{t}^{-\frac{1}{2}}\tilde{A}_{s-1}L_{t}^{-\frac{1}{2}}\\
&\preceq 8d^2B^* L_{s-1}^{-\frac{1}{2}}\tilde{A}_{s-1}L_{s-1}^{-\frac{1}{2}}\\
&\stackrel{(a)}{\preceq} 8d^2B^* \kappa A_{s-1}^{-\frac{1}{2}}\tilde{A}_{s-1}A_{s-1}^{-\frac{1}{2}}\\
&\stackrel{(b)}{\preceq} 12d^2B^*\kappa I,
\end{align*}
and 
\begin{align*}
Z_s&\succeq -4d^2B^*L_t^{-\frac{1}{2}} \tilde{A}_{s-1} L_t^{-\frac{1}{2}}\\
&\succeq -4d^2B^* L_{s-1}^{-\frac{1}{2}}\tilde{A}_{s-1}L_{s-1}^{-\frac{1}{2}}\\
&\stackrel{(c)}{\succeq} -4d^2B^*\kappa A_{s-1}^{-\frac{1}{2}}\tilde{A}_{s-1}A_{s-1}^{-\frac{1}{2}}\\
& \stackrel{(d)}{\succeq} -6d^2B^*\kappa I,
\end{align*}
where $(a), (c)$ follow from the definition of $L_t, U_t$ and the fact that the losses are $\kappa$-convex, and $(b), (d)$ follow from the induction hypothesis that $\|I-A_{s}^{-\frac{1}{2}}\tilde{A}_sA_{s}^{-\frac{1}{2}}\|_2\le \frac{1}{2}$ and  $\tilde{A}_s$ is PSD $, \forall s\le t-1$ w.h.p.

\paragraph{Bounding $2^{nd}$ moments of $Z_s$.} Next, we bound the second moments of $Z_s$
\begin{align*}
    \EBB[Z_s^2|\FM_{s-1}] \stackrel{(a)}{\preceq} \EBB[L_{t}^{-\frac{1}{2}}\tilde{H}_sL_t^{-1}\tilde{H}_sL_{t}^{-\frac{1}{2}}|\FM_{s-1}],
\end{align*}
where $(a)$ follows from the fact that for any random matrix $X$: $\EBB[(X-\EBB[X])X-\EBB[X])^\top]\preceq \EBB[XX^\top].$ Continuing
\begin{align*}
    &\EBB[L_{t}^{-\frac{1}{2}}\tilde{H}_sL_t^{-1}\tilde{H}_sL_{t}^{-\frac{1}{2}}|\FM_{s-1}] \\
    &= L_t^{-1/2}\tilde{A}_{s-1}^{1/2} \EBB[4d^4f_s(y_s)^2(v_{s,1}v_{s,2}^\top +v_{s,2}v_{s,1}^\top )\tilde{A}_{s-1}^{1/2} L_t^{-1}\tilde{A}_{s-1}^{1/2}(v_{s,1}v_{s,2}^\top +v_{s,2}v_{s,1}^\top )|\FM_{s-1}]\tilde{A}_{s-1}^{1/2} L_t^{-1/2}\\
    &\stackrel{(a)}{\preceq} L_t^{-1/2}\tilde{A}_{s-1}^{1/2} \EBB[6d^4f_s(y_s)^2(v_{s,1}v_{s,2}^\top +v_{s,2}v_{s,1}^\top )(v_{s,1}v_{s,2}^\top +v_{s,2}v_{s,1}^\top )|\FM_{s-1}]\tilde{A}_{s-1}^{1/2} L_t^{-1/2},\\
    &\stackrel{(b)}{\preceq} 6d^4(B^*)^2 \kappa L_t^{-1/2}\tilde{A}_{s-1}^{1/2} \EBB[(v_{s,1}v_{s,2}^\top +v_{s,2}v_{s,1}^\top )(v_{s,1}v_{s,2}^\top +v_{s,2}v_{s,1}^\top )|\FM_{s-1}]\tilde{A}_{s-1}^{1/2} L_t^{-1/2},
\end{align*}
where $(a)$ follows from the fact that $\tilde{A}_{s-1}^{1/2} L_t^{-1}\tilde{A}_{s-1}^{1/2} \preceq 1.5\kappa I$ w.h.p, and $(b)$ follows from the fact that $f_s(y_s) \leq B^*$. Next, we rely on the facts that $\EBB[v_{s,1}v_{s,2}^\top v_{s,1}v_{s,2}^\top] = d^{-2}I,$ and $\EBB[v_{s,1}v_{s,2}^\top v_{s,2}v_{s,1}^\top] = d^{-1}I$ to obtain
\begin{align*}
    &\EBB[L_{t}^{-\frac{1}{2}}\tilde{H}_sL_t^{-1}\tilde{H}_sL_{t}^{-\frac{1}{2}}|\FM_{s-1}] \\
    &\preceq 12d^3(B^*)^2(1+d^{-1})\kappa L_t^{-1/2}\tilde{A}_{s-1}L_t^{-1/2}\\
    & \stackrel{(a)}{\preceq} 36d^3(B^*)^2\kappa^2,
\end{align*}
where $(a)$  follows from the fact that $ L_t^{-1/2}\tilde{A}_{s-1}L_t^{-1/2} \preceq 1.5\kappa I$ w.h.p.
\paragraph{Matrix Freedman.} With the above bounds on the martingale sequence, we can bound $L_{t}^{-\frac{1}{2}}(A_{t} - \tilde{A}_{t})L_{t}^{-\frac{1}{2}}$ by matrix Freedman inequality (Lemma~\ref{lem:matrix-freedman}) \footnote{Freedman's inequality requires the first and second moment bounds to hold almost surely. Observe that we only have these bounds in high probability. However, there is a standard workaround to this (see for instance \citet{suggala2021efficient}), where one can create an alternate martingale that satisfies these bounds a.s., and is exactly equal to the original martingale w.h.p.}
% By matrix Azuma's inequality, assuming the inequality holds for $s\le t-1$, with probability at least $1-\frac{1}{T^2}$, 
\begin{align*}
\|I-\tilde{A}_{t}^{-\frac{1}{2}}A_t\tilde{A}_{t}^{-\frac{1}{2}}\|_2 &\leq \|L_{t}^{-\frac{1}{2}}(A_{t} - \tilde{A}_{t})L_{t}^{-\frac{1}{2}}\|_2 \\
& \leq \frac{12\eta \kappa d^{3/2}B^*\sqrt{t\log(dT^2)}}  {\kappa'} + \frac{8\eta \kappa d^2B^*\log(dT^2)}{\kappa'}. 
\end{align*}
The RHS of the above inequality is less than $1/2$ for our choice of $\eta$.
% \js{Do we ever use $U_t$ in this proof? Does that mean we just need a lower bound for $A_t$ in this case?} \textcolor{red}{yes, we do use $U_t$ for deriving inequalities $(a), (c)$.}
\end{proof}

\begin{lemma} [Matrix Freedman, Theorem 1.2 in \citep{tropp2011freedman}] 
\label{lem:matrix-freedman}
Consider a self-adjoint matrix martingale $\{Y_k\}_{k\ge 0}$ of dimension $d$. Let $\{X_k\}_{k\ge 1}$ be the difference sequence of $\{Y_k\}_{k\ge 0}$. Assume that the difference sequence is uniformly bounded almost surely: $\exists R$ such that
\begin{align*}
\lambda_{\max}(X_k)\le R \ \ \ \text{almost surely} \ \ \ \forall k.
\end{align*}
Define the predictable quadratic variation process
\begin{align*}
W_k=\sum_{j=1}^k \E[X_j^2\mid\F_{j-1}], \ \ \ \forall k.
\end{align*}
Then $\forall \eps\ge 0, \sigma^2>0$,
\begin{align*}
\mathbb{P}\left(\exists k\ge 0: \lambda_{\max}(Y_k)\ge \eps \  \text{and }\|W_k\|_2\le \sigma^2\right)\le d\cdot\exp\left(-\frac{\eps^2/2}{\sigma^2+R\eps/3}\right).
\end{align*}
    
\end{lemma}

\begin{corollary}
\label{cor:f_bound}
Let $B^* = B+\sqrt{2}(L+\sqrt{2}C)$. With probability at least $1-\frac{1}{T}$, the $y_t$ played by Algorithm~\ref{alg:simple-bqo} satisfies
\begin{align*}
|f_t(y_t)|\le B^*, \ \forall t\in[T]. 
\end{align*}
\end{corollary}
\begin{proof}
Note that in the proof of Lemma~\ref{lem:concentration-cumulative-hessian}, we have shown that $|f_t(y_t)|\le B^*$ if $\|I-A_{t-1}^{-\frac{1}{2}}\tilde{A}_{t-1}A_{t-1}^{-\frac{1}{2}}\|_2\le \frac{1}{2}$. Thus,
\begin{align*}
\mathbb{P}\left(|f_t(y_t)|\le B^*, \ \forall t\in[T]\right)\ge \mathbb{P}\left(\|I-A_{t}^{-\frac{1}{2}}\tilde{A}_{t}A_{t}^{-\frac{1}{2}}\|_2\le \frac{1}{2}, \ \forall t\in[T]\right)\ge 1-\frac{1}{T}. 
\end{align*}
\end{proof}

\subsection{Concentration of Cumulative Hessian Estimate [Improved $\kappa$ dependence]}
\label{sec:improved_kappa_dependence}

In this section, we provide an alternate analysis for Hessian concentration that improves the dependence of regret on $\kappa$, but at the cost of worse dependence on $d$. In particular, this analysis leads to a regret of 
$\tilde{O}(d^3\kappa B^*\sqrt{T}),$ instead of $\tilde{O}(d^{2.5}\kappa^2 B^* \sqrt{T})$ stated in Theorem~\ref{thm:expected_regret_bound}. This analysis relies on empirical Freedman's inequality~\citep{zimmert2022return} instead of the matrix Freedman used in the previous section.
\begin{lemma} [Concentration of cumulative Hessian estimate]
\label{lem:concentration-cumulative-hessian-improved-kappa}
Consider a sequence of functions $\{f_t\}_{t=1}^T$ that satisfies Assumption~\ref{assumption:oblivious-adversary}, Assumption~\ref{assumption:curvature}, and Assumption~\ref{assumption:bounded-range-and-grads}. $\forall t\in[T]$, define the smoothed function $\smoothfBB{t}$ of $f_t$ as
\[
\smoothfBB{t}(x) \coloneqq \E_{u\sim \ball, v \sim \ball}\left[f_t\left(x + \frac{1}{2} \tilde{A}_{t-1}^{-\frac{1}{2}}(u+v)\right)\Big| \FM_{t-1}\right],
\]
where $\mathcal{F}_t$ denotes the filtration generated by the algorithm's possible randomness up to time $t$. 
Let $\tilde{A}_t=\tilde{A}_{t-1}+\frac{\eta}{\kappa'}\tilde{H}_{t}$, and $A_{t}=A_{t-1}+\frac{\eta}{\kappa'}\nabla^2\smoothfBB{t}(x_t)$, where $A_0=\tilde{A}_{0}=I$. Let $B^*=B+\sqrt{2}(L+\sqrt{2}C)$. Suppose $T = \Omega(d^2)$, and $\eta\le \kappa'\left(200d^2\sqrt{T\log{dTB^*}}\right)^{-1}$, $\forall t$, with probability at least $1-\frac{t}{T^2}$, for every $s\le t$, 
\begin{align*}
\|I-A_{s}^{-\frac{1}{2}}\tilde{A}_sA_{s}^{-\frac{1}{2}}\|_2\le \frac{1}{2}.
\end{align*}
\end{lemma}
\begin{proof}
Similar to the proof of Lemma~\ref{lem:concentration-cumulative-hessian}, we rely on induction to prove the result. 
\paragraph{Proof by induction.} Base case $t=0$ is given by construction. Suppose the inequalities hold with probability at least $1-\frac{t-1}{T^2}$ for any $0\le s\le t-1$. Recall that $f_t$ is assumed to be bounded by $B$ and $L$-Lipschitz over $\K$ and $C$-smooth. Note that $|f_t(y_t)|$ is bounded by
\begin{align*}
|f_t(y_t)|&\le |f_t(x_t)|+\left|f_t\left(x_t+\frac{1}{2}\tilde{A}_{t-1}^{-\frac{1}{2}}(v_{t,1}+v_{t,2})\right)-f_t(x_t)\right|\\
&\le B +  \left\|\nabla f_t\left(x_t+\frac{1}{2}\tilde{A}_{t-1}^{-\frac{1}{2}}(v_{t,1}+v_{t,2})\right)\right\|_2 \left\|\frac{1}{2}\tilde{A}_{t-1}^{-\frac{1}{2}}(v_{t,1}+v_{t,2})\right\|_2\\
&\le B+\left(L+C\left\|\frac{1}{2}\tilde{A}_{t-1}^{-\frac{1}{2}}(v_{t,1}+v_{t,2})\right\|_2\right)\left\|\frac{1}{2}\tilde{A}_{t-1}^{-\frac{1}{2}}(v_{t,1}+v_{t,2})\right\|_2\\
&\le B+(L+C\|\tilde{A}_{t-1}^{-\frac{1}{2}}\|_2)\|\tilde{A}_{t-1}^{-\frac{1}{2}}\|_2.
\end{align*} 
By induction hypothesis, with probability at least $1-\frac{t-1}{T^2}$, $\tilde{A}_{t-1}\succeq \frac{1}{2}A_{t-1}\succeq \frac{1}{2}I$, and thus the above inequality implies $|f_t(y_t)|\le B+\sqrt{2}\left(L+\sqrt{2}C\right)=:B^*$. 

Next, consider the following
\begin{align*}
    \tilde{A}_t &= \tilde{A}_{t-1} + \frac{\eta}{\kappa'}\tilde{H}_t = \tilde{A}_{t-1}^{1/2}\left(I + \frac{2d^2 f_t(y_t)\eta}{\kappa'}(v_{t,1}v_{t,2}^{\top}+v_{t,2}v_{t,1}^{\top})\right)\tilde{A}_{t-1}^{1/2}
\end{align*}
For our choice of $\eta$, it is easy to verify that $\frac{1}{2}\preceq \left(I + \frac{2d^2 f_t(y_t)\eta}{\kappa'}(v_{t,1}v_{t,2}^{\top}+v_{t,2}v_{t,1}^{\top})\right) \preceq \frac{3}{2}.$ So $\frac{1}{2}\tilde{A}_{t-1} \preceq \tilde{A}_t \preceq \frac{3}{2}\tilde{A}_{t-1}$. This shows that $\tilde{A}_t$ is invertible, and is very close to $\tilde{A}_{t-1}$, with high probability. We use this inequality often in the subsequent analysis. We now show that the following holds with probability at least $1-\delta$
\[
\sum_{s=1}^t\E_{s-1}[\tilde{H}_s]- g A_{t-1}\preceq\sum_{s=1}^t\tilde{H}_s \preceq \sum_{s=1}^t\E_{s-1}[\tilde{H}_s] +  g A_{t-1},
\]
where $g = \left(96d^2\sqrt{t\log\frac{dTB^*}{\delta}} + 32d^{3}\log\frac{dTB^*}{\delta}\right)B^*$.
Let $\C_{\epsilon}$ be a finite cover of $\sphere = \{x\in\real^d | \|x\|_2 = 1\}$ such that for any $v\in \sphere$ there exists a $u \in \C_{\epsilon}$ such that $\|v-u\|_2 \leq \epsilon.$ By \citet{vershynin2018high}, there exists a $\C_{\epsilon}$ such that $|\C_{\epsilon}| \leq \left(\frac{2}{\epsilon}+1\right)^d.$ Consider any $u\in \C_{\epsilon}$, and the corresponding sequence of random variables $\{u^T\tilde{H}_su\}_{s=1}^t$.  We use Freedman's inequality to bound 
$|\sum_{s=1}^t u^T(\tilde{H}_s-\E_{s-1}[\tilde{H}_s])u|$. To do this, we derive bounds for the first and second moments of $u^T\tilde{H}_su$
\paragraph{Bounding $1^{st}$ moment of $u^T\tilde{H}_su$.}
From the definition of $\tilde{H}_s$, and the fact that  $-2I\preceq v_{t,1}v_{t,2}^{\top}+v_{t,2}v_{t,1}^{\top}\preceq 2I$, we have the following, which holds with high probability
\begin{align*}
    |u^T\tilde{H}_su| \leq 4d^2 B^*\|u\|_{\tilde{A}_{s-1}}^2 \stackrel{(a)}{\leq} 8d^2 B^*\|u\|_{A_{s-1}}^2 \stackrel{(b)}{<} \infty.
\end{align*}
Here, inequality $(a)$ follows from the fact that $\|I-A_{s}^{-\frac{1}{2}}\tilde{A}_sA_{s}^{-\frac{1}{2}}\|_2\le \frac{1}{2}$ $, \forall s\le t-1$ w.h.p. And inequality $(b)$ follows from the fact that the loss functions $f_t$ are smooth.
\paragraph{Bounding $2^{nd}$ moment of $u^T\tilde{H}_su$.} Next, we bound the second moments of $u^T\tilde{H}_su$
\begin{align*}
    &\EBB[(u^T\tilde{H}_su)^2|\FM_{s-1}] = \EBB[u^T\tilde{H}_suu^T\tilde{H}_su|\FM_{s-1}]\\
    &= u^T\tilde{A}_{s-1}^{1/2} \EBB[4d^4f_s(y_s)^2(v_{s,1}v_{s,2}^\top +v_{s,2}v_{s,1}^\top )\tilde{A}_{s-1}^{1/2} uu^T\tilde{A}_{s-1}^{1/2}(v_{s,1}v_{s,2}^\top +v_{s,2}v_{s,1}^\top )|\FM_{s-1}]\tilde{A}_{s-1}^{1/2} u\\
    &\stackrel{(a)}{\leq} \|u\|^2_{\tilde{A}_{s-1} }u^T\tilde{A}_{s-1}^{1/2}\EBB[4d^4f_s(y_s)^2(v_{s,1}v_{s,2}^\top +v_{s,2}v_{s,1}^\top )(v_{s,1}v_{s,2}^\top +v_{s,2}v_{s,1}^\top )|\FM_{s-1}]\tilde{A}_{s-1}^{1/2} u,\\
    &\stackrel{(b)}{\leq}  4d^4(B^*)^2\|u\|^2_{\tilde{A}_{s-1} }u^T\tilde{A}_{s-1}^{1/2}\EBB[(v_{s,1}v_{s,2}^\top +v_{s,2}v_{s,1}^\top )(v_{s,1}v_{s,2}^\top +v_{s,2}v_{s,1}^\top )|\FM_{s-1}]\tilde{A}_{s-1}^{1/2} u,
\end{align*}
where $(a)$ follows from the fact $\tilde{A}_{s-1}^{1/2} uu^T\tilde{A}_{s-1}^{1/2} \preceq\|u\|^2_{\tilde{A}_{s-1} } I$, and $(b)$ follows from the fact that $f_s(y_s) \leq B^*$ w.h.p. Next, we rely on the facts that $\EBB[v_{s,1}v_{s,2}^\top v_{s,1}v_{s,2}^\top] = d^{-2}I,$ and $\EBB[v_{s,1}v_{s,2}^\top v_{s,2}v_{s,1}^\top] = d^{-1}I$ to obtain
\begin{align*}
    &\EBB[(u^T\tilde{H}_su)^2|\FM_{s-1}] \leq 8d^3(1+d^{-1})(B^*)^2 \|u\|^4_{\tilde{A}_{s-1} } \leq 16d^3(B^*)^2 \|u\|^4_{\tilde{A}_{s-1} } \leq 64d^3(B^*)^2 \|u\|^4_{A_{s-1}}.
\end{align*}
\paragraph{Empirical Freedman's Inequality.} Applying the empirical Freedman's inequality from Lemma~\ref{lem:empirical-freedman} on $\{u^T\tilde{H}_su\}_{s=1}^t$ gives us the following bound, which holds with probability at least $1-\delta$
\[
\Big|\sum_{s=1}^tu^T(\tilde{H}_s-\E_{s-1}[\tilde{H}_s])u\Big| \leq \left(48d^{3/2}\sqrt{t\log\frac{dTB^*}{\delta}} + 16d^2\log\frac{dTB^*}{\delta}\right)B^*\|u\|^2_{A_{t-1}}.
\]
Taking a union bound over all $u\in\C_{\epsilon}$, we get the following which holds with probability at least $1-\delta$, for any $u \in \C_{\epsilon}$
\[
\Big|\sum_{s=1}^tu^T(\tilde{H}_s-\E_{s-1}[\tilde{H}_s])u\Big| \leq \left(48d^{3/2}\sqrt{t\log\frac{dTB^*|\C_{\epsilon}|}{\delta}} + 16d^2\log\frac{dTB^*|\C_{\epsilon}|}{\delta}\right)B^*\|u\|^2_{A_{t-1}}.
\]
We now show that the above bound also holds for any $u\in \sphere,$ for appropriate choice of $\epsilon$. Consider any $u \in \sphere$ and let $v$ be the closest point to $u$ that lies in $\C_{\epsilon}$.
\begin{align*}
    &\sum_{s=1}^t\left[u^T(\tilde{H}_s-\E_{s-1}[\tilde{H}_s])u - v^T(\tilde{H}_s-\E_{s-1}[\tilde{H}_s])v\right]  \\
    &\quad = \sum_{s=1}^t(v-u)^T(\tilde{H}_s-\E_{s-1}[\tilde{H}_s])(v-u) + \sum_{s=1}^t2(u-v)^T(\tilde{H}_s-\E_{s-1}[\tilde{H}_s])v\\
    &\quad \leq \sum_{s=1}^t\Big|(v-u)^T(\tilde{H}_s-\E_{s-1}[\tilde{H}_s])(v-u)\Big| + \sum_{s=1}^t2\Big|(u-v)^T(\tilde{H}_s-\E_{s-1}[\tilde{H}_s])v\Big|.
\end{align*}
Using similar arguments as used for bounding the first moment of $u^T\tilde{H}_su$, where we showed that $|u^T\tilde{H}_su| \leq 8d^2B^*\|u\|^2_{A_{t-1}}$, we get
\begin{align*}
    &\sum_{s=1}^t\left[u^T(\tilde{H}_s-\E_{s-1}[\tilde{H}_s])u - v^T(\tilde{H}_s-\E_{s-1}[\tilde{H}_s])v\right]  \\
    &\quad \leq 16td^2B^*\|v-u\|^2_{A_{t-1}} + 32td^2B^*\|v-u\|_{A_{t-1}}\|v\|_{A_{t-1}}\\
    &\quad \stackrel{(a)}{=} O\left(t^{3/2}d^2B^*\epsilon^2 + t^{3/2}d^2B^*\epsilon\right),
\end{align*}
where $(a)$ follows from the fact that $\|A_{t-1}\|_2 = O\left(1 + \frac{\eta t}{\kappa'}\right)$. Choosing $\epsilon = \frac{g}{t^2d^2},$ for some appropriate constant $g$, we have the following which holds with probability at least $1-\delta$
\[
\forall u\in \sphere, \quad \Big|\sum_{s=1}^tu^T(\tilde{H}_s-\E_{s-1}[\tilde{H}_s])u\Big| \leq \left(96d^{2}\sqrt{t\log\frac{dTB^*}{\delta}} + 32d^{3}\log\frac{dTB^*}{\delta}\right)B^*\|u\|^2_{A_{t-1}}.
\]
This shows that, with probability at least $1-\delta$
\[
\sum_{s=1}^t\E_{s-1}[\tilde{H}_s]- g A_{t-1}\preceq\sum_{s=1}^t\tilde{H}_s \preceq \sum_{s=1}^t\E_{s-1}[\tilde{H}_s] +  g A_{t-1},
\]
where $g = \left(96d^2\sqrt{t\log\frac{dTB^*}{\delta}} + 32d^{3}\log\frac{dTB^*}{\delta}\right)B^*$. Rewriting this equation gives us
\[
A_t- \frac{\eta g}{\kappa'} A_{t-1}\preceq \tilde{A}_t \preceq A_t +  \frac{\eta g}{\kappa'} A_{t-1},
\]
Since $A_{t-1} \preceq A_t$, we get
\[
\left(1 - \frac{\eta g}{\kappa'}\right)A_t \preceq \tilde{A}_t \preceq \left(1 + \frac{\eta g}{\kappa'}\right)A_t.
\]
For our choice of $\eta \leq \kappa'\left(200d^2\sqrt{T\log{\frac{dTB^*}{\delta}}}\right)^{-1}, T = \Omega(d^2)$, we get
\[
\frac{1}{2}A_t \preceq \tilde{A}_t \preceq \frac{3}{2}A_t.
\]
This finishes the proof of the Lemma.
\end{proof}

\begin{lemma} [Strengthened Freedman's inequality~\citep{zimmert2022return}] 
\label{lem:empirical-freedman}
Let $\{X_t\}_{t=1,2 \dots }$ be a martingale difference sequence w.r.t a filtration $\FM_1 \subseteq \FM_2 \subseteq \dots $ such that $\E[X_t|\FM_t] = 0$ and assume $\E[|X_t||\FM_t] < \infty$ a.s. Then with probability at least $1-\delta$
\[
\Big|\sum_{t=1}^TX_t\Big| \leq 3 \sqrt{V_T\log\left(\frac{2\max\{U_T, \sqrt{V_T}\}}{\delta}\right)} + 2U_T\log\left(\frac{2\max\{U_T, \sqrt{V_T}\}}{\delta}\right),
\]
where $V_T = \sum_{t=1}^T\E_{t-1}[X_t^2], U_T = \max\{1, \max_{t\in[T]}X_t\}$.
% of random variables adapted to filteration $\FM_t$. Let $\tau$ be a stopping time w.r.t $\{\FM_t\}_{t=1}^T$ with $\tau \leq T$ almost surely. Let $\E_{t}[\cdot] = \E[\cdot|\FM_t]$, and $|X_t|\leq B$ almost surely for all $t \leq T$. Then with probability at least $1-\delta$
% \[
% \Big|\sum_{t=1}^\tau (X_t-\E_{t-1}[X_t])\Big| \leq C\left[\sqrt{V_\tau \log\left(\frac{\log(1+V_{\tau})}{\delta}\right)} + B\log\frac{1}{\delta}\right]
% \]
% where $V_{\tau} = \sum_{t=1}^{\tau} \E_{t-1}[(X_t-\E_{t-1}[X_t])^2]$, and $C>0$ is a sufficiently large absolute constant.
\end{lemma}

\subsection{Proof of Theorem~\ref{thm:expected_regret_bound}}
\label{sec:apx_improper_bco_proof}
Let $\ball = \{x:\|x\|_2 \leq 1\}, \sphere=\{x:\|x\|_2 = 1\}$ be the unit ball and unit sphere in $\mathbb{R}^d$. We define smoothed functions $\smoothfBB{t}, \smoothfBS{t}, \smoothfSS{t}$ as follows
\begin{align*}
&\smoothfBB{t}(x) \coloneqq \E_{u\sim \ball, v \sim \ball}\left[f_t\left(x + \frac{1}{2} \tilde{A}_{t-1}^{-\frac{1}{2}}(u+v)\right)\Big| \FM_{t-1}\right]\\
&\smoothfBS{t}(x) \coloneqq \E_{u\sim \ball, v \sim \sphere}\left[f_t\left(x + \frac{1}{2} \tilde{A}_{t-1}^{-\frac{1}{2}}(u+v)\right)\Big| \FM_{t-1}\right]\\
&\smoothfSS{t}(x) \coloneqq \E_{u\sim \sphere, v \sim \sphere}\left[f_t\left(x + \frac{1}{2} \tilde{A}_{t-1}^{-\frac{1}{2}}(u+v)\right)\Big| \FM_{t-1}\right].
\end{align*}
Observe that
\begin{align*}
\E[f_t(y_t)\mid \FM_{t-1}]=\smoothfSS{t}(x_t),
\end{align*}
where $\FM_t$ denotes the filtration generated by $\{v_{s,1},v_{s,2}\}_{s=1}^{t}$. 
% \eh{here is how I think this proof should be idealy structured:
% \begin{enumerate}
%     \item have a generic theorem about the regret of online newton method with approximate hessian and gradient, that includes a bias term
%     \item proof that the estimators have bias bounded by such and such
%     \item now conclude the regret bound
% \end{enumerate} 
% see the proof for FKM in my book for such a proof but using only gradient. 
% }
By Stokes' theorem~\citep{flaxman2004online}, we have that the gradient and Hessian estimators constructed in Algorithm~\ref{alg:simple-bqo} satisfy 
\begin{align*}
    &\E[\tilde{\nabla}_t\mid \FM_{t-1}] = \nabla \smoothfBS{t}(x_t),\quad \E[\tilde{H}_t\mid \FM_{t-1}] = \nabla^2 \smoothfBB{t}(x_t).
\end{align*}
%  One interesting observation is that $\tilde{H}_t$ is not an unbiased estimate of $\nabla^2\tilde{f}_t(x_t)$. In fact, it is an unbiased estimate of the Hessian of another smoothed function of $f_t$ at $x_t$, if $v$ was also drawn uniformly at random from the unit ball:
% \begin{align*}
% \E[\tilde{H}_t\mid \FM_{t-1}]&=\nabla^2\hat{f}_t(x_t),\\
% \hat{f}_t(x)&=\E_{u,v \sim \ball}\left[f_t(x_t + 2^{-1} A_{t-1}^{-1/2}(u+v))\right].
% \end{align*}
Observe that the Hessian and gradients estimated in Algorithm~\ref{alg:simple-bqo} are not of the same function. Interestingly, despite this mismatch, we can derive $\sqrt{T}$ regret of the algorithm. Define $A_t \coloneqq A_{t-1} + \frac{\eta}{\kappa'} \nabla^2\smoothfBB{t}(x_t)$, with $A_0=I$. Throughout the proof, we assume that the cumulative Hessian concentrates well; that is, the following holds
$$\|A_t^{-\frac{1}{2}}(A_t - \tilde{A}_t)A_t^{-\frac{1}{2}}\|_2 \leq \frac{1}{2}.$$
This assumption is formally proved in Lemma~\ref{lem:concentration-cumulative-hessian}, which says that the above stated inequality holds simultaneously for all $t\in[T]$ with probability at least $1-\frac{1}{T}$. Therefore, we can without loss of generality assume that the above inequality holds deterministically by suffering an additional constant in the regret bound. 

% \noindent Observe that the above assumption implies $A_t$ is invertible only in high probability (this doesn't hold almost surely!). We need to be mindful of this both in our algorithm and when performing the expectation analysis. For now, let's assume the above holds almost surely.  

\noindent For any $x \in \K$, we can decompose the regret as the following:
\begin{align*}
    & \quad \sum_{s=1}^{T}\E[f_s(y_s) -  f_s(x)] \\
    & = \sum_{s=1}^{T}\E[\smoothfSS{s}(x_s) -  f_s(x)]\\
    & = \sum_{s=1}^{T}\E[\smoothfBS{s}(x_s) -  f_s(x)] + \E[\smoothfSS{s}(x_s) - \smoothfBS{s}(x_s)]\\
    & =  \underbrace{\sum_{s=1}^{T}\E[\smoothfBS{s}(x_s) -  \smoothfBS{s}(x)]}_{T_1}  +  \underbrace{\sum_{s=1}^{T}\E[\smoothfBS{s}(x) -  f_s(x)]}_{T_2} +  \underbrace{\sum_{s=1}^{T}\E[\smoothfSS{s}(x_s) - \smoothfBS{s}(x_s)]}_{T_3} ,
\end{align*}
where we will bound $T_1,T_2$, and $T_3$ separately.
\paragraph{Bounding $T_1$.} To bound this term, we rely on Lemma~\ref{lem:stoc_ons_full_info_regret}. In particular, we  instantiate it with the ONS algorithm described in  Algorithm~\ref{alg:ons_full_information}, and use $\Tilde{\Delta}_t, \tilde{H}_t$ as the stochastic estimates of $\nabla \smoothfBS{t}(x_t), \nabla^2 \smoothfBS{t}(x_t)$. This gives us the following bound 
\begin{align*}
    T_1 \leq \frac{\text{diam}(\K)^2}{2\eta} + \sum_{t=1}^T \frac{\eta}{2} \EBB\left[\norm{\tilde{\nabla}_t}_{\tilde{A}_t^{-1}}^2\right] - \sum_{t=1}^T\EBB\left[\Delta_t(x)\right],
\end{align*}
where $\Delta_t(x)$ is defined as
\[
\Delta_t(x) \coloneqq \smoothfBS{t}(x) - \smoothfBS{t}(x_t) - \inp{\nabla \smoothfBS{t}(x_t)}{x-x_t} -  \frac{1}{2\kappa'}(x-x_t)^{\top}\EBB[\tilde{H}_t|\FM_{t-1}] (x-x_t).
\]
For our choice of $\kappa' (\ge \kappa)$, and our assumption on the Hessian of $f_t$, $\Delta_t(x)$ is always greater than $0$. This is because
\begin{align*}
\smoothfBS{t}(x) - \smoothfBS{t}(x_t) - \inp{\nabla \smoothfBS{t}(x_t)}{x - x_t}&\ge \frac{c}{2}\|x-x_{t}\|_{H_t}^2\ge \frac{1}{2\kappa'}(x-x_t)^\top \EBB[\tilde{H}_t|\FM_{t-1}](x-x_t).
\end{align*}
Next, we bound $\norm{\tilde{\nabla}_t}_{\tilde{A}_t^{-1}}$:
\begin{align*}
\norm{\tilde{\nabla}_t}_{\tilde{A}_t^{-1}}^2&=2d^2f_t(y_t)^2 v_{t,1}^{\top} \tilde{A}_{t-1}^{\frac{1}{2}}\tilde{A}_{t}^{-1}\tilde{A}_{t-1}^{\frac{1}{2}}v_{t,1}\\
&\stackrel{(a)}{\le} 6d^2f_t(y_t)^2v_{t,1}^{\top}A_{t-1}^{\frac{1}{2}}A_{t}^{-1}A_{t-1}^{\frac{1}{2}}v_{t,1}\\
&\stackrel{(b)}{\le} 6d^2(B^*)^2,
\end{align*}
where $(a)$ follows from the fact that $\frac{1}{2}A_t\preceq \tilde{A}_{t} \preceq \frac{3}{2}A_t$, and $(b)$ follows from Corollary~\ref{cor:f_bound}. Substituting this in the above upper bound for $T_1$ gives us
\[
T_1 \leq \frac{\text{diam}(\K)^2}{2\eta} + 3\eta d^2(B^*)^2 T.
\]
\paragraph{Bounding $T_2$.}
We now upper bound $T_2$ by $\tilde{O}(\frac{1}{\eta}).$ To see this, consider the following
\begin{align*}
    T_2 &= \sum_{t=1}^T\E[\smoothfBS{t}(x) -  f_t(x)] \\
    &= \sum_{t=1}^T\E_{u_t\sim \ball, v_t\sim \sphere}\left[f_t\left(x + \frac{1}{2}\tilde{A}_{t-1}^{-\frac{1}{2}}(u_t+v_t)\right) -  f_t(x)\right] \\
    & = \sum_{t=1}^T\E_{u_t\sim \ball, v_t\sim \sphere}\left[\frac{1}{2}\nabla f_t(x)^{\top}\tilde{A}_{t-1}^{-\frac{1}{2}}(u_t+v_t)+\frac{1}{8}(u_t+v_t)^{\top}\tilde{A}_{t-1}^{-\frac{1}{2}}\nabla^2 f_t(x(u_t,v_t))\tilde{A}_{t-1}^{-\frac{1}{2}}(u_t+v_t)\right] \\
    & = \frac{1}{8}\sum_{t=1}^T \E_{u_t\sim \ball, v_t\sim \sphere}[(u_t+v_t)^T\tilde{A}_{t-1}^{-\frac{1}{2}}\nabla^2f_t(x(u_t, v_t)) \tilde{A}_{t-1}^{-\frac{1}{2}}(u_t+v_t)],
\end{align*}
where the last equality follows from the fact that conditioning on $\FM_{t-1}$, the first-order term vanishes. Here, $x(u_t,v_t)$ is a point on the line connecting $x$ and $x + \frac{1}{2}\tilde{A}_{t-1}^{-\frac{1}{2}}(u_t+v_t)$. Next, observe that $\nabla^2f_t(x(u_t, v_t)) \preceq C H_t \preceq \frac{C}{c}\E[\tilde{H}_t\mid \FM_{t-1}] \preceq \kappa \E[\tilde{H}_t\mid \FM_{t-1}]$. Substituting this in the previous display gives us
\begin{align*}
    T_2  & \leq \frac{\kappa}{8}\sum_{t=1}^T \E[(u_t+v_t)^T\tilde{A}_{t-1}^{-\frac{1}{2}}\E[\tilde{H}_t\mid \FM_{t-1}] \tilde{A}_{t-1}^{-\frac{1}{2}}(u_t+v_t)].
\end{align*}
Next, since by concentration of Hessian estimate accumulates, we have $\tilde{A}_t\succeq \frac{1}{2}A_t$ and the fact that 
\begin{align*}
A_{t-1}=A_t- \frac{\eta}{\kappa'} \E[\tilde{H}_t\mid \FM_{t-1}]\succeq A_t-c\eta I\succeq \frac{1}{2}A_t,
\end{align*}
we bound $T_2$ as 
\begin{align*}
    T_2 &\leq  \frac{\kappa}{4}\sum_{t=1}^T \E[(u_t+v_t)^{\top}A_{t-1}^{-\frac{1}{2}}\E[\tilde{H}_t\mid \FM_{t-1}] A_{t-1}^{-\frac{1}{2}}(u_t+v_t)]\\
    &\leq  \frac{\kappa}{2}\sum_{t=1}^T \E[(u_t+v_t)^{\top}A_{t}^{-\frac{1}{2}}\E[\tilde{H}_t\mid \FM_{t-1}] A_{t}^{-\frac{1}{2}}(u_t+v_t)].
\end{align*}
Continuing, we have with $|\cdot|$ denoting the determinant of a square matrix,
\begin{align*}
    T_2 & \le \frac{\kappa\kappa'}{2\eta }\sum_{t=1}^T \E[(u_t+v_t)^{\top}A_{t}^{-\frac{1}{2}} (A_{t} - A_{t-1}) A_{t}^{-\frac{1}{2}}(u_t+v_t)]\\
    & = \frac{\kappa\kappa'}{\eta}\sum_{s=1}^T \E[\trace(A_{t}^{-\frac{1}{2}} (A_{t} - A_{t-1}) A_{t}^{-\frac{1}{2}})]\\
    & \stackrel{(c)}{\le} \frac{\kappa\kappa'}{\eta}\E\left[\log\frac{|A_T|}{|A_0|}\right] \\
    & \stackrel{(d)}{\le} \frac{\kappa\kappa' d\log(1+ \eta C T/\kappa')} {\eta},
\end{align*}
where inequality $(c)$ follows 
from the fact that for $A,B\succeq 0$, 
$$A^{-1}\cdot (A-B)\le \log\frac{|A|}{|B|},$$ 
and inequality $(d)$ follows from the fact that the since the largest eigenvalue $\lambda_{\max}(A_T)$ of $A_T$ satisfies $\lambda_{\max}(A_T)\le 1+ \eta CT/\kappa'$,
\begin{align*}
\log |A_T|\le \log\left(\lambda_{\max}(A_T)^d\right)\le d \log(1+\eta CT/\kappa').
\end{align*}
\paragraph{Bounding $T_3$.} 
\begin{align*}
    T_3 &= \sum_{s=1}^{T}\E[\smoothfSS{s}(x_s) - \smoothfBS{s}(x_s)]\\
    & \stackrel{(a)}{\leq} \sum_{s=1}^{T}\E[\smoothfSS{s}(x_s) - f_s(x_s)],
\end{align*}
where $(a)$ follows from the convexity of $f_s$:
\begin{align*}
    \smoothfBS{s}(x_s) &= \EBB_{u_s\sim \ball, v_s\sim \sphere}\left[f_s\left(x_s+\frac{1}{2}\tilde{A}_{t-1}^{-\frac{1}{2}}(u_s+v_s)\right)\right] \\
    & \ge f_s\left(x_s + \frac{1}{2}\EBB_{u_s\sim \ball, v_s\sim \sphere}\left[\tilde{A}_{t-1}^{-\frac{1}{2}}(u_s+v_s)\right]\right)\\
& = f_s(x_s)
\end{align*}
Using similar arguments as $T_2$ to bound $\E[\smoothfSS{s}(x_s) - f_s(x_s)]$, we obtain
\begin{align*}
    T_3 \leq  \frac{\kappa\kappa' d\log(1+ \eta C T/\kappa')} {\eta}.
\end{align*}
Combining the bounds for $T_1, T_2, T_3$ gives us the required regret bound.

\section{Regret of \texttt{BNS-AM} for BQO-AM problems (Section~\ref{subsec:BQO-AM-algo})}
In this section, we prove Theorem~\ref{thm:bqo-m-regret}. We first make two observations on \texttt{BNS-AM}. \texttt{BNS-AM} does improper learning, as the decisions $y_t$'s do not necessarily lie within $\K$ (Line~\ref{line:y-m} in \cref{alg:bqo-memory}). We bound the function value and gradient evaluated at $y_t$'s in Remark~\ref{rmk:improper-bounds}. The dependence of the iterates are explained in Remark~\ref{remark:filtration-and-independence}.

\begin{remark}[Value and gradient bound of $y_t$'s]
\label{rmk:improper-bounds}
Since Algorithm~\ref{alg:bqo-memory} does improper learning, it is essential to show that the loss and gradient at each of the $y_t$ played by the algorithm is bounded. Note that $\forall t$,
\begin{align*}
\|\nabla f_t(y_{t-m+1:t})\|_2&\le \|\nabla f_t(x_{t-m+1:t})\|_2+\beta \|y_{t-m+1:t}-x_{t-m+1:t}\|_2\le L+\beta \sqrt{m}, \\
|f_t(y_{t-m+1:t})|&\le |f_t(x_{t-m+1:t})|+|\nabla f_t(y_{t-m+1:t})^{\top}(y_{t-m+1:t}-x_{t-m+1:t})| \\
&\le B+(L+\beta\sqrt{m})\sqrt{m}. 
\end{align*}
We denote $B^*=B+(L+\beta\sqrt{m})\sqrt{m}$. 
\end{remark}

\begin{remark} [Filtration and independence from delay]
\label{remark:filtration-and-independence}
Denote $\F_t=\sigma(\{u_s\}_{s\le t})$ to be the filtration generated by Algorithm~\ref{alg:bqo-memory} random sampling step. Then by Assumption~\ref{assumption:adversary-adaptivity}, $f_t, H_t, \hat{A}_t$ are $\F_{t-m}$-measurable. By the delayed updates in Algorithm~\ref{alg:bqo-memory}, $x_t$ is $\F_{t-m}$-measurable. $u_t, y_t,\tilde{g}_t$ are $\F_t$-measurable. 
\end{remark}
We now prove Theorem~\ref{thm:bqo-m-regret}.
Following usual procedure of bounding with-memory regret in bandit setting, we note that the regret can be decomposed into three terms, which we will bound separately:
\begin{align*}
\E[\regret_T(x)]&=\underbrace{\E\left[\sum_{t=m}^T f_t(y_{t-m+1:t})-f_t(x_{t-m+1:t})\right]}_{(\text{perturbation loss})}+\underbrace{\E\left[\sum_{t=m}^T f_t(x_{t-m+1:t})-\bar{f}_{t}(x_t)\right]}_{(\text{movement cost})}\\
& \ \ \ \ \ +\underbrace{\E\left[\sum_{t=m}^T \bar{f}_{t}(x_{t})-\bar{f}_{t}(x)\right]}_{(\text{underlying regret})}.
\end{align*}
We will decompose the proof as the following: First, we will show some useful properties of the gradient estimator $\tilde{g}_t$. Then, we will bound the three terms above separately. 

\subsection{Properties of the gradient estimator}

The gradient estimator constructed in Line~\ref{line:grad-est-m}, \cref{alg:bqo-memory} is a conditionally unbiased estimator of the divergence of the loss function $f_t$ evaluated at $(x_{t-m+1},\dots,x_t)$. Together with the smoothness assumption on $f_t$ and stability of the algorithm, this implies that the bias incurred by using $\tilde{g}_t$ as an estimator of $\nabla \bar{f}_t(x_t)$ is small. 

%\begin{lemma}
%$\forall t$, the largest eigenvalue of $\hat{A}_t$ satisfies:
%\begin{align*}
%\lambda_{\max}(\hat{A}_t)\ge m+\frac{\eta c \sigma (t-m+1)}{2d}.
%\end{align*}
%\end{lemma}
%\begin{proof}
%For any $A=\sum_{i=1}^n A_i$, where $A_i\in\mathbb{R}^{d\times d}$ is symmetric, positive semidefinite, and $\lambda_{\max}(A_i)\ge \sigma$, we can denote the eigenvector of $A_i$ corresponding to $\sigma$ as $v_i$. Then,
%\begin{align*}
%\trace(A)=\sigma \sum_{i=1}^n \|v_i\|_2^2=\sigma n,
%\end{align*}
%and thus $\sum_{i=1}^d \lambda_i(A)=\sigma n$ and $\lambda_i(A)\ge 0$, $\forall i$. Thus, $\lambda_{\max}(A)\ge \frac{\sigma n}{d}$.  This implies that
%\begin{align*}
%\lambda_{\max}(\hat{A}_{t})=m+\lambda_{\max}\left(\sum_{t=m}^t \frac{\eta c}{2}H_t\right)\ge m+\frac{\eta c\sigma(t-m+1)}{2d}.
%\end{align*}
%\end{proof}

\begin{lemma} [Gradient estimator]
\label{lem:grad-est}
$\forall t\ge m$, the conditional expectation of the gradient estimator $\tilde{g}_t$ constructed in Line~\ref{line:grad-est-m}, \cref{alg:bqo-memory} is given by
\begin{align*}
\E[\tilde{g}_t\mid \F_{t-m}]=\sum_{i=1}^m [\nabla f_t(x_{t-m+1:t})]_i,
\end{align*}
where $[\nabla f_t(x_{t-m+1:t})]_i$ denotes the $i$-th $d$-vector in the $dm$-vector $\nabla f_t(x_{t-m+1:t})$.
\end{lemma}
\begin{proof}
The proof follows similarly to the proof of Lemma C.3 in \citep{sun2023optimal}. In particular, we note that for a real-valued quadratic function on $\mathbb{R}^n$ $q(x)=\frac{1}{2}x^{\top} Ax+b^{\top}x+c$, a symmetric positive-definite matrix $M\in\mathbb{R}^{d\times d}$, a filtration $\F$, and a random unit vector $u\in\mathbb{R}^n$ satisfying (1) $u$ is symmetric with zero first and third moments conditioning on $\F$, (2) $u$ satisfies $\E[uu^{\top}\mid \F]=\frac{r}{n}I_{n\times n}$, and (3) $A, b, c, x$ are $\F$-measurable, then
\begin{align*}
\E\left[q(x+Mu)M^{-1}u\mid \F\right]&=\frac{1}{2}\E\left[(x+Mu)^{\top}A(x+Mu)M^{-1}u\mid \F\right]+\E\left[b^{\top}(x+Mu)M^{-1}u\mid \F\right]\\
& \ \ \ \ \ +\E[cM^{-1}u\mid \F]\\
&=\frac{1}{2}\E[x^{\top}AMuM^{-1}u\mid \F]+\frac{1}{2}\E[u^{\top}MAxM^{-1}u\mid \F]+\E[b^{\top}MuM^{-1}u\mid \F]\\
&=\frac{1}{2}M^{-1}\E[uu^{\top}\mid \F]MA^{\top}x+\frac{1}{2}M^{-1}\E[uu^{\top}\mid \F]MAx+M^{-1}\E[uu^{\top}\mid\F] Mb\\
&=\frac{r}{2n}(A+A^{\top})x+\frac{r}{n}b\\
&=\frac{r}{n}\nabla q(x).
\end{align*}
Consider the following matrix $M_t\in\mathbb{R}^{dm\times dm}$:
\begin{align*}
M_t =\begin{bmatrix}
\hat{A}_{t-m}^{-\frac{1}{2}} & 0 & \dots & 0 \\
0 & \hat{A}_{t-m+1}^{-\frac{1}{2}} & \dots & 0 \\
\dots & \dots & \dots & \dots \\
0 & 0 & \dots & \hat{A}_{t-1}^{-\frac{1}{2}}
\end{bmatrix}.
\end{align*}
By Remark~\ref{remark:filtration-and-independence}, $f_t, M_t$ is $\F_{t-m}$-measurable, and the concatenated vector $v_t=[u_{t-m+1},\dots u_t]$ satisfies that the first and third moments of $v_t$ is $0$ conditioning on $\F_{t-m}$, and since $u_{t-m+1},\dots,u_t$ are independent conditioning on $\F_{t-m}$, 
\begin{align*}
\E[v_t(v_t)^{\top}\mid \F_{t-m}]=\frac{1}{d}I_{dm\times dm}.
\end{align*}
By taking $r=m$, we have that 
\begin{align*}
\nabla f_t(x_{t-m+1:t})=\E[df_t(y_{t-m+1:t})M_t^{-1}v_t\mid \F_{t-m}]. 
\end{align*}
We note that by definition of the gradient estimator $\tilde{g}_{t}$,
\begin{align*}
\sum_{i=1}^m [\nabla f_t(x_{t-m+1:t})]_i&=\E\left[df_t(y_{t-m+1:t})\sum_{i=1}^m [M_t^{-1}v_t]_i\mid \F_{t-m}\right]\\
&=\E\left[df_t(y_{t-m+1:t})\sum_{s=t-m+1}^t\hat{A}_{s}^{\frac{1}{2}}u_s\mid \F_{t-m}\right]\\
&=\E\left[\tilde{g}_{t}\mid \F_{t-m}\right]. 
\end{align*}
\end{proof}

\subsection{Bounding perturbation loss} 
\label{sec:perturbation-loss-known}
We start with the perturbation loss and prove the following proposition:
\begin{proposition}
\label{prop:perturbation-loss}
The perturbation loss is bounded by
\begin{align*}
\E\left[\sum_{t=m}^T f_t(y_{t-m+1:t})-f_t(x_{t-m+1:t})\right]\le \beta R_G\left[\frac{10mR_GR_Y}{\kappa(G)}\sqrt{T} + \frac{4md\log(\sigma T) }{\eta\alpha\kappa(G)} +  dR_Y^2\sqrt{T}\right].
\end{align*}
\end{proposition}

\begin{proof}
Note that by the quadratic assumption on $f_t$, we have by second-order Taylor expansion:
\begin{align*}
\sum_{t=m}^T \E[f_t(y_{t-m+1:t})-f_t(x_{t-m+1:t})]&=\sum_{t=m}^T \E[\nabla f_t(x_{t-m+1:t})^{\top}\hat{A}_{t-m:t-1}^{-\frac{1}{2}}u_{t-m+1:t}]\\
& \ \ \ \ \ +\frac{1}{2} \sum_{t=m}^T \E[u_{t-m+1:t}^{\top}\hat{A}_{t-m:t-1}^{-\frac{1}{2}}\nabla^2 f_t \hat{A}_{t-m:t-1}^{-\frac{1}{2}}u_{t-m:1:t}].
\end{align*}
The first order term is $0$ since $\forall t$, $f_t, x_{t-m+1},\dots, x_t, \hat{A}_{t-m},\dots,\hat{A}_{t-1}$ are $\F_{t-m}$-measurable, and thus
\begin{align*}
\E[\nabla f_t(x_{t-m+1:t})^{\top}\hat{A}_{t-m:t-1}^{-\frac{1}{2}}u_{t-m+1:t}]=\E[\nabla f_t(x_{t-m+1:t})^{\top}\hat{A}_{t-m:t-1}^{-\frac{1}{2}}\E[u_{t-m+1:t}\mid \F_{t-m}]]=0. 
\end{align*}
The second term can be bounded as following. Note that by denoting $[\nabla^2 f_t]_{ij}$ as the $ij$-th $d\times d$ block of $\nabla^2 f_t$, we have
\begin{align*}
& \quad u_{t-m+1:t}^{\top}\hat{A}_{t-m:t-1}^{-\frac{1}{2}}\nabla^2 f_t \hat{A}_{t-m:t-1}^{-\frac{1}{2}}u_{t-m+1:t}\\
&=[\hat{A}_{t-m}^{-\frac{1}{2}}u_{t-m+1},\dots,\hat{A}_{t-1}^{-\frac{1}{2}}u_t]^{\top}\nabla^2 f_t[\hat{A}_{t-m}^{-\frac{1}{2}}u_{t-m+1},\dots,\hat{A}_{t-1}^{-\frac{1}{2}}u_t]\\
&=\sum_{i,j=1}^m u_{t-m+i}^{\top}\hat{A}_{t-m+i-1}^{-\frac{1}{2}}[\nabla^2 f_t]_{ij}\hat{A}_{t-m+j-1}^{-\frac{1}{2}}u_{t-m+j}.
\end{align*}
By independence between $u_s, u_t$ for $s\ne t$ and taking the expectation, we have
\begin{align*}
& \quad \E[u_{t-m+1:t}^{\top}\hat{A}_{t-m:t-1}^{-\frac{1}{2}}\nabla^2 f_t \hat{A}_{t-m:t-1}^{-\frac{1}{2}}u_{t-m:1:t}]\\
&=\sum_{i=1}^m \E\left[u_{t-m+i}^{\top}\hat{A}_{t-m+i-1}^{-\frac{1}{2}}[\nabla^2 f_t]_{ii}\hat{A}_{t-m+i-1}^{-\frac{1}{2}}u_{t-m+i}\right]\\
&\le \sum_{i=1}^m \E\left[\trace(\hat{A}_{t-m+i-1}^{-\frac{1}{2}}[\nabla^2 f_t]_{ii}\hat{A}_{t-m+i-1}^{-\frac{1}{2}})\right]\\
&\le \E\left[\trace\left(\hat{A}_{t-m}^{-\frac{1}{2}}\sum_{i=1}^m [\nabla^2 f_t]_{ii}\hat{A}_{t-m}^{-\frac{1}{2}}\right)\right],
\end{align*}
where the last steps follow from that $\hat{A}_{s}\preceq \hat{A}_{t}$, $\forall s\le t$, and $\trace(\cdot)$ is linear. Note that since $f_t$ is quadratic, $[\nabla^2 f_t]_{ii}=\nabla^2_{y_{t-m+i}} f_t$, where $\nabla^2_{y_{t-m+i}} f_t$ is the Hessian of $f_t$ w.r.t. $y_{t-m+i}$. $\forall 1\le i\le m-1$, by smoothness assumption on $Q_t$,
\begin{align}
\label{eq:diag-hessian}
[\nabla^2 f_t]_{ii}=\nabla^2_{y_{t-m+i}} f_t=Y_{t-m+i}^{\top}(G^{[m-i]})^{\top}Q_tG^{[m-i]}Y_{t-m+i}\preceq \beta\tilde{R}_{G}Y_{t-m+i}^{\top}Y_{t-m+i}.
\end{align}

Denote $\sigma=\max_{t\in[T]}\lambda_{\max}(H_t)\le R_G^2R_Y^2$. For simplicity, assume that $\sigma\ge 1$ and $\eta\alpha\sigma m\le 2$, which will be satisfied by our choice of $\eta$ and $m$ for large enough $T$. Summing over all iterations, we have
\begin{align*}
& \quad \frac{1}{2} \sum_{t=m}^T \E[u_{t-m+1:t}^{\top}\hat{A}_{t-m:t-1}^{-\frac{1}{2}}\nabla^2 f_t \hat{A}_{t-m:t-1}^{-\frac{1}{2}}u_{t-m:1:t}]\\
&\le \frac{1}{2}\sum_{t=m}^T \E\left[\trace\left(\hat{A}_{t-m}^{-\frac{1}{2}}\sum_{i=1}^m [\nabla^2 f_t]_{ii}\hat{A}_{t-m}^{-\frac{1}{2}}\right)\right]\\
&\le \frac{\beta\tilde{R}_{G}}{2}\sum_{t=m}^{T}\E\left[\trace\left(\hat{A}_{t-m}^{-\frac{1}{2}}\left(\sum_{s=t-m+1}^t Y_{s}^{\top}Y_{s}\right)\hat{A}_{t-m}^{-\frac{1}{2}}\right)\right]\\
&\le\beta\tilde{R}_{G}\sum_{t=m}^{T}\E\left[\trace\left(\hat{A}_{t}^{-\frac{1}{2}}\left(\sum_{s=t-m+1}^tY_{s}^{\top}Y_{s}\right) \hat{A}_{t}^{-\frac{1}{2}}\right)\right],
\end{align*}
where th second inequality follows from \cref{eq:diag-hessian}, and the third inequality follows from that since $\hat{A}_t\succeq mI$, $\forall t$, we have that $\forall s < t$,
\begin{align}
\label{eqn:preconditioner-inequality}
\hat{A}_t=\hat{A}_s+\frac{\eta\alpha}{2}\sum_{r=s+1}^t H_r\preceq \hat{A}_s+\frac{\eta\alpha\sigma m(t-s)}{2}I\preceq \left(1+\frac{\eta\alpha\sigma(t-s)}{2}\right)\hat{A}_s,
\end{align}
which implies that assuming $\eta\alpha\sigma m\le 2$, we have $\hat{A}_{t}\preceq \max\{2, \eta\alpha\sigma m\}\hat{A}_{t-m}\preceq 2\hat{A}_{t-m}$. 

Here, we use a similar ``blocking" technique used in \cite{simchowitz2020making} to bound this term. In particular, for some $\tau\in\mathbb{Z}_{++}$ to be determined later (we will take $\tau=\lfloor\sqrt{T}\rfloor$), consider endpoints $k_j=\tau(j-1)+m$, $j=1,\dots,J$, where $J=\lfloor \frac{T-m}{\tau}\rfloor$. Then, we can rewrite the sum on the right hand side as
\begin{align*}
\sum_{t=m}^{T}\trace\left(\hat{A}_{t}^{-\frac{1}{2}}\left(\sum_{s=t-m+1}^t Y_{s}^{\top}Y_{s}\right)\hat{A}_{t}^{-\frac{1}{2}}\right)&=\sum_{j=1}^{J}\sum_{t=k_j}^{k_{j+1}-1}\trace\left(\left(\sum_{s=t-m+1}^t Y_{s}^{\top}Y_{s}\right)\cdot \hat{A}_{t}^{-1}\right)\\
& \ \ \ \ \ +\sum_{t=J\tau+m}^T \trace\left(\left(\sum_{s=t-m+1}^t Y_{s}^{\top}Y_{s}\right)\cdot \hat{A}_{t}^{-1}\right)\\
&\le_{(1)}\sum_{j=1}^{J}\sum_{t=k_j}^{k_{j+1}-1}\trace\left(\left(\sum_{s=t-m+1}^t Y_{s}^{\top}Y_{s}\right)\cdot \hat{A}_{t}^{-1}\right)\\
& \ \ \ \ \ +\tau dR_Y^2,
\end{align*}
where $(1)$ follows from that $T-m-J\tau<\tau$ and since $\hat{A}_t\succeq mI$ and $\trace(A)\le d\|A\|_2$ for $A\in\mathbb{R}^{d\times d}$ such that $A\succeq 0$ is symmetric,
\begin{align*}
\trace\left(\left(\sum_{s=t-m+1}^t Y_{s}^{\top}Y_{s}\right)\cdot \hat{A}_{t}^{-1}\right)\le \frac{d}{m}\sum_{s=t-m+1}^t\|Y_s^{\top}Y_s\|_2\le dR_Y^2. 
\end{align*}
To bound the first term on the right hand side, since $\hat{A}_t\succeq \hat{A}_{k_j}$ for $t\ge k_j$, we have
\begin{align*}
\sum_{j=1}^{J}\sum_{t=k_j}^{k_{j+1}-1}\trace\left(\left(\sum_{s=t-m+1}^t Y_{s}^{\top}Y_{s}\right)\cdot \hat{A}_{t}^{-1}\right)&\le\frac{4m}{\kappa(G)}\underbrace{\sum_{j=1}^{J}\sum_{t=k_j}^{k_{j+1}-1}\left(\frac{\kappa(G)}{4m}\trace\left(\sum_{s=t-m+1}^t Y_{s}^{\top}Y_{s}\right)\cdot \hat{A}_{k_{j}}^{-1}\right)}_{(a)},
\end{align*}
To bound $(a)$, we further decompose $(a)$ as
\begin{align*}
(a)=\underbrace{\sum_{j=1}^{J}\sum_{t=k_j}^{k_{j+1}-1}\trace\left(\left(-\frac{H_t}{2}+\frac{\kappa(G)}{4m}\sum_{s=t-m+1}^t Y_{s}^{\top}Y_{s}\right)\cdot \hat{A}_{k_{j}}^{-1}\right)}_{(i)}+\underbrace{\frac{1}{2}\sum_{j=1}^{J}\trace\left(\left(\sum_{t=k_j}^{k_{j+1}-1} H_t\right)\cdot \hat{A}_{k_{j}}^{-1}\right)}_{(ii)}.
\end{align*}
To see the bound on the first term, we use Proposition 4.8 from \cite{simchowitz2020making}. We include the proof of this result in Section~\ref{sec:proof-ytht-inequality} for completion.
\begin{lemma}[Proposition 4.8 in \cite{simchowitz2020making}]
\label{lem:YtHt-inequality}
$\forall Y_1,\dots,Y_T$, we have that:
\begin{align*}
\sum_{t=m}^T H_t\succeq \frac{\kappa(G)}{2}\sum_{t=1}^T Y_t^{\top}Y_t-5mR_GR_YI. 
\end{align*}
\end{lemma}

If we let $\tilde{Y}_t=Y_{t+k_j-m}$, and $\tilde{H}_t=H_{t+k_j-m}$, then Lemma~\ref{lem:YtHt-inequality} implies that
\begin{align*}
\sum_{t=k_j}^{k_{j+1}-1}\left(\frac{\kappa (G)}{4m}\sum_{s=t-m+1}^t Y_s^{\top}Y_s\right)&\preceq \frac{\kappa(G)}{4}\sum_{t=k_j-m+1}^{k_{j+1}-1}Y_t^{\top}Y_t\\
&\preceq \frac{\kappa(G)}{4}\sum_{t=1}^{\tau+m-1}\tilde{Y}_t^{\top}\tilde{Y_t}\\
&\preceq \sum_{t=m}^{\tau+m-1}\frac{\tilde{H}_t}{2}+\frac{5mR_GR_Y}{2}I\\
&=\sum_{t=k_j}^{k_{j+1}-1}\frac{H_t}{2}+\frac{5mR_GR_Y}{2}I.
\end{align*}
Thus, $(i)$ is bounded by
\begin{align*}
(i)&=\sum_{j=1}^J\underbrace{\left[\left(\sum_{t=k_j}^{k_{j+1}-1}-\frac{H_t}{2}+\frac{\kappa(G)}{4m}\sum_{s=t-m+1}^t Y_{s}^{\top}Y_{s}\right)-\frac{5mR_GR_Y}{2}I\right]}_{\preceq 0}\cdot \hat{A}_{k_{j}}^{-1}+\frac{5mR_GR_Y}{2}\sum_{j=1}^J\trace(\hat{A}_{k_{j}}^{-1})\\
&\le \frac{5dR_GR_Y}{2}\left\lfloor\frac{T}{\tau}\right\rfloor.
\end{align*}
Since by Eq.~\ref{eqn:preconditioner-inequality}, $\hat{A}_{k_{j+1}-1}\preceq \max\{2, \eta\alpha\sigma \tau\} \hat{A}_{k_j}$, $(ii)$ is bounded by 
\begin{align*}
(ii)&=\frac{1}{2}\sum_{j=1}^{J}\left(\sum_{t=k_j}^{k_{j+1}-1} H_t\right)\cdot \hat{A}_{k_{j}}^{-1}\\
&=\frac{1}{2\eta \alpha}\sum_{j=1}^{J}\left(\hat{A}_{k_{j+1}-1}-\hat{A}_{k_j-1}\right)\cdot \hat{A}_{k_{j}}^{-1}\\
&\le \frac{\max\{2, \eta\alpha\sigma \tau\}}{2\eta\alpha}\sum_{j=1}^{J}\left(\hat{A}_{k_{j+1}-1}-\hat{A}_{k_j-1}\right)\cdot \hat{A}_{k_{j+1}-1}^{-1}\\
&\le \frac{\max\{2, \eta\alpha\sigma \tau\}}{2\eta\alpha}\sum_{j=1}^{J}\log\left(\frac{|\hat{A}_{k_{j+1}-1}|}{|\hat{A}_{k_{j}-1}|}\right)\\
&\le \frac{\max\{2, \eta\alpha\sigma \tau\}}{2\eta\alpha}\log(|\hat{A}_T|)\\
&\le_{(2)} \frac{\max\{2, \eta\alpha\sigma \tau\}}{2\eta\alpha} d\log(\sigma T),
\end{align*}
where (2) follows from $|\hat{A}_T|\le \|\hat{A}_T\|_2^d$. 

%To bound $(b)$, note that for $A, A+B\succ 0$, we have that
%\begin{align*}
%(A+B)^{-1}-A^{-1}=-(I+A^{-1}B)^{-1}A^{-1}BA^{-1}.
%\end{align*}
%Thus, $(b)$ can be bounded by
%\begin{align*}
%(b)&\le \sum_{j=1}^{J}\sum_{t=k_j}^{k_{j+1}-1}\sum_{s=0}^{t-k_j-1}\left\|\sum_{s=t-m+1}^t Y_{s}^{\top}Y_{s}\right\|_*\left\|\hat{A}_{t-s-1}^{-1}\left(\hat{A}_{t-s}-\hat{A}_{t-s-1}\right)\hat{A}_{t-s-1}^{-1}\right\|_2\\
%&\le_{(3)}  dmR_Y^2\sum_{j=1}^{J}\sum_{t=k_j}^{k_{j+1}-1}\sum_{s=0}^{t-k_j-1}\ \hat{A}_{t-s-1}^{-1}\cdot\left(\hat{A}_{t-s}-\hat{A}_{t-s-1}\right)\\
%&\le dmR_Y^2\sigma\sum_{j=1}^{J}\sum_{t=k_j}^{k_{j+1}-1}\sum_{s=0}^{t-k_j-1} \hat{A}_{t-s}^{-1}\cdot\left(\hat{A}_{t-s}-\hat{A}_{t-s-1}\right)\\
%&\le dmR_Y^2\sigma \sum_{t=m}^T \sum_{s=0}^{(\tau-1)\wedge (t-m)}\log\left(\frac{|\hat{A}_{t-s}|}{|\hat{A}_{t-s-1}|}\right)\\
%&\le \tau dmR_Y^2\sigma \sum_{t=m}^T \log\left(\frac{|\hat{A}_{t}|}{|\hat{A}_{t-1}|}\right)\\
%&\le  \tau dmR_Y^2\sigma\log(|\hat{A}_T|)\\
%&\le \tau d^2mR_Y^2\sigma\log (\sigma T),
%\end{align*}
%where $|A|$ denotes the determinant of $A$, and (3) follows from the nuclear norm is bounded by $\|A\|_*\le \mathrm{rank}(A)\|A\|_2$. 

Combining, we have
\begin{align*}
(\text{perturbation loss})\le \beta R_G\left[\frac{4m}{\kappa(G)}\left(\frac{5dR_GR_Y}{2}\left\lfloor\frac{T}{\tau}\right\rfloor+ \frac{\max\{2, \eta\alpha\sigma \tau\}}{2\eta\alpha} d\log(\sigma T)\right) + \tau dR_Y^2\right],
\end{align*}
by setting $\tau=\lfloor \sqrt{T}\rfloor$ and assuming $\eta\le \frac{2}{\alpha \sigma \sqrt{T}}\le \frac{2}{\alpha \sigma \tau}$, we have
\begin{align*}
(\text{perturbation loss})\le \beta R_G\left[\frac{10mR_GR_Y}{\kappa(G)}\sqrt{T} + \frac{4md\log(\sigma T) }{\eta\alpha\kappa(G)} +  dR_Y^2\sqrt{T}\right].
\end{align*}
\end{proof}

\subsection{Bounding movement cost}
\label{sec:movement-cost}
The movement cost depends on the stability of the algorithm and is bounded by the following lemma:
\begin{lemma}
\label{lem:movement-cost}
The movement cost is bounded by
\begin{align*}
\E\left[\sum_{t=m}^T f_t(x_{t-m+1:t})-\bar{f}_{t}(x_t)\right]\le  \eta dB^*Lm^{2}T. 
\end{align*}
\end{lemma}

\begin{proof}
The movement cost is bounded by the Lipschitz constant $L$ of $f_t$'s and the Euclidean distances between neighboring iterates.
\begin{lemma} [Generalized Pythagorean Theorem]
\label{lem:generalized-pythagorean}
For any positive definite, symmetric matrix $A$ and its induced norm $\|\cdot\|_{A}$ over $\mathbb{R}^d$, let $C$ be a convex, closed, nonempty subset of $\mathbb{R}^d$, $y\in\mathbb{R}^d$, and $x=\Pi_{C}(y)$ w.r.t. $\|\cdot\|_A$. Then $\forall z\in C$,
\begin{align*}
\|x-z\|_A\le \|y-z\|_A.
\end{align*}
\end{lemma}

Note that by the above lemma, take $y=x_{t-1}-\eta\hat{A}_{t-m}^{-1}\tilde{g}_{t-m}$, $x=x_t$, $z=x_{t-1}$, and $\|\cdot\|=\|\cdot\|_{\hat{A}_{t-m}}$, we have that $\forall t$,
\begin{align*}
\|x_t-x_{t-1}\|_{\hat{A}_{t-m}}^2&\le \|\eta\hat{A}_{t-m}^{-1}\tilde{g}_{t-m}\|_{\hat{A}_{t-m}}^2=\eta^2 \|\tilde{g}_{t-m}\|_{\hat{A}_{t-m}^{-1}}^2,
\end{align*}
where since $\hat{A}_{t-m}\succeq \hat{A}_{t-m-1}$,
\begin{align*}
\|\tilde{g}_{t-m}\|_{\hat{A}_{t-m}^{-1}}^2&\le d^2 (B^*)^2\sum_{i,j=0}^{m-1}u_{t-i}^{\top} \hat{A}_{t-m-1-i}^{\frac{1}{2}}\hat{A}_{t-m}^{-1}\hat{A}_{t-m-1-j}^{\frac{1}{2}}u_{t-j}\\
&\le d^2(B^*)^2m^2\left\|\hat{A}_{t-m-1}^{\frac{1}{2}}\hat{A}_{t-m}^{-1}\hat{A}_{t-m-1}^{\frac{1}{2}}\right\|_{\mathrm{op}}\\
&\le d^2(B^*)^2m^2,
\end{align*}
Thus,
\begin{align*}
\|x_t-x_{t-1}\|_2\le \frac{1}{\sqrt{m}}\|x_t-x_{t-1}\|_{\hat{A}_{t-m}}\le \sqrt{m}\eta dB^*.
\end{align*}
By the Lipschitz assumption on $f_t$ over $\K$, the movement cost is bounded by
\begin{align*}
\sum_{t=m}^T \E[f_t(x_{t-m+1:t})-\bar{f}_t(x_t)]&\le L\sum_{t=m}^T \E[\|(x_{t-m+1},\dots, x_t)-(x_t,\dots, x_t)\|_2]\\
&= L \sum_{t=m}^T \E\left[\left(\sum_{s=t-m+1}^{t-1}\|x_s-x_t\|_2^2\right)^{\frac{1}{2}}\right]\\
&\le  L \sum_{t=m}^T \E\left[ \left(\sum_{s=t-m+1}^{t-1}(t-s)\sum_{r=s+1}^{t}\|x_r-x_{r-1}\|_2^2\right)^{\frac{1}{2}}\right]\\
&\le \eta dB^*Lm^{2}T. 
\end{align*}
\end{proof}

\subsection{Bounding underlying regret} 
\label{sec:ons-regret-known}
The underlying regret can be bounded by the following lemma:
\begin{lemma}
\label{lem:bound-underlying-regret}
The underlying regret is bounded by
\begin{align*}
\E\left[\sum_{t=m}^T \bar{f}_{t}(x_{t})-\bar{f}_{t}(x)\right]\le \frac{mD^2}{2\eta}+2\eta \sigma \max\{\alpha,1\} d^2(B^*)^2m^3T+\eta \beta \tilde{R}_{G}R_Y^2 d^2 B^* D m^{\frac{9}{2}}T+mB. 
\end{align*}
\end{lemma}

\begin{proof}
By projection onto convex set, we have that $\forall x\in\K$,
\begin{align*}
\|x_t-x\|_{\hat{A}_{t-m}}^2&\le \|x_{t-1}-x-\eta\hat{A}_{t-m}^{-1}\tilde{g}_{t-m}\|_{\hat{A}_{t-m}}^2\\
&= \|x_{t-1}-x\|_{\hat{A}_{t-m-1}}^2+\frac{1}{2}\|x_{t-1}-x\|_{\eta \alpha H_{t-m}}^2-2\eta \tilde{g}_{t-m}^{\top}(x_{t-1}-x)+\eta^2 \|\tilde{g}_{t-m}\|_{\hat{A}_{t-m}^{-1}}^2\\
&\le \|x_{t-1}-x\|_{\hat{A}_{t-m-1}}^2+\|x_{t-m}-x\|_{\eta\alpha H_{t-m}}^2+\|x_{t-m}-x_{t-1}\|_{\eta\alpha H_{t-m}}^2-2\eta \tilde{g}_{t-m}^{\top}(x_{t-m}-x)\\
& \ \ \ \ \ + 2\eta\tilde{g}_{t-m}^{\top}(x_{t-m}-x_{t-1})+\eta^2\|\tilde{g}_{t-m}\|_{\hat{A}_{t-m}^{-1}}^2,
\end{align*}
where the last inequality follows from that for a square matrix $H\succeq 0$, $\|x+y\|_{H}^2\le 2(\|x\|_{H}^2+\|y\|_{H}^2)$. 
Rearranging, we have
\begin{align*}
& \quad \tilde{g}_{t-m}^{\top}(x_{t-m}-x)-\frac{1}{2}\|x_{t-m}-x\|_{\alpha H_{t-m}}^2\\
&\le \frac{1}{2\eta}(\|x_{t-1}-x\|_{\hat{A}_{t-m-1}}^2-\|x_t-x\|_{\hat{A}_{t-m}}^2)+\frac{\alpha}{2}\|x_{t-m}-x_{t-1}\|_{H_{t-m}}^2+\tilde{g}_{t-m}^{\top}(x_{t-m}-x_{t-1})\\
& \quad +\frac{\eta}{2}\|\tilde{g}_{t-m}\|_{\hat{A}_{t-m}^{-1}}^2.
\end{align*}
We can bound the terms on the right hand side as the following: by using the bounds on $\|x_t-x_{t-1}\|_2$ and $\|\tilde{g}_{t}\|_{\hat{A}_{t}^{-1}}$ in Section~\ref{sec:movement-cost}, we have
\begin{align*}
\frac{\alpha}{2}\|x_{t-m}-x_{t-1}\|_{H_{t-m}}^2&\le \frac{\alpha\sigma m}{2}\sum_{s=t-m+1}^{t-1}\|x_{s}-x_{s-1}\|_2^2\le \frac{1}{2}\eta^2\alpha\sigma d^2 (B^*)^2 m^3,\\
\tilde{g}_{t-m}^{\top}(x_{t-m}-x_{t-1})&\le \|\tilde{g}_{t-m}\|_{\hat{A}_{t-m}^{-1}}\|x_{t-m}-x_{t-1}\|_{\hat{A}_{t-m}}\\
&\le \|\tilde{g}_{t-m}\|_{\hat{A}_{t-m}^{-1}}\sum_{s=t-m+1}^{t-1}\|x_s-x_{s-1}\|_{\hat{A}_{t-m}}\\
&\le 2\|\tilde{g}_{t-m}\|_{\hat{A}_{t-m}^{-1}}\sum_{s=t-m+1}^{t-1}\|x_s-x_{s-1}\|_{\hat{A}_{s-m}}\\
&\le 2\eta d^2(B^*)^2m^3.
\end{align*}
Therefore, substituting these bounds into the right hand side of the inequality, we have
\begin{align*}
\tilde{g}_{t-m}^{\top}(x_{t-m}-x)-\frac{1}{2}\|x_{t-m}-x\|_{\alpha H_{t-m}}^2&\le \frac{1}{2\eta}(\|x_{t-1}-x\|_{\hat{A}_{t-m-1}}^2-\|x_t-x\|_{\hat{A}_{t-m}}^2)+\frac{1}{2}\eta^2\alpha\sigma d^2 (B^*)^2 m^3\\
& \ \ \ \ \ +2\eta d^2(B^*)^2m^3+\frac{1}{2}\eta d^2(B^*)^2 m^2\\
&\le \frac{1}{2\eta}(\|x_{t-1}-x\|_{\hat{A}_{t-m-1}}^2-\|x_t-x\|_{\hat{A}_{t-m}}^2)\\
& \ \ \ \ \ +2\eta \sigma \max\{\alpha,1\} d^2(B^*)^2m^3. 
\end{align*}
Note that since by Lemma~\ref{lem:grad-est},
\begin{align*}
\E[\tilde{g}_{t-m}\mid \F_{t-2m}]=\sum_{i=1}^m \nabla_i f_{t-m}(x_{t-2m+1:t-m}).
\end{align*}
When bounding perturbation loss, we showed that $\forall t\in[T]$, $\forall i\in[m]$,
\begin{align*}
[\nabla^2 f_t]_{ii}\preceq \beta\tilde{R}_{G}Y_{t-m+i}^{\top}Y_{t-m+i}\preceq \beta\tilde{R}_{G}R_Y^2I_{d\times d}.
\end{align*}
Thus, since $\nabla^2 f_t\succeq 0$, we have
\begin{align*}
\|\nabla^2 f_t\|_2\le \trace(\nabla^2 f_t)=\sum_{i=1}^m \trace([\nabla^2 f_t]_{ii})\le \sum_{i=1}^m d \|[\nabla^2 f_t]_{ii}\|_2\le dm\beta\tilde{R}_{G}R_Y^2. 
\end{align*}
By Lemma A.2 in \cite{gradu2020non}, $\forall x$,
\begin{align*}
\nabla \bar{f}_{t}(x)=\sum_{i=0}^{m-1} \nabla_i f_t(x,\dots,x).
\end{align*}
By smoothness of $f_{t-m}$,
\begin{align*}
\|\E[\tilde{g}_{t-m}\mid \F_{t-2m}]-\nabla \bar{f}_{t-m}(x_{t-m})\|_2&\le \sqrt{m}\|\nabla f_{t-m}(x_{t-2m+1:t-m})-\nabla f_{t-m}(x_{t-m},\dots,x_{t-m})\|_2\\
&\le d\beta\tilde{R}_{G}R_Y^2 m^2 \sum_{s=t-2m+1}^{t-m-1}\|x_s-x_{t-m}\|_2\\
&\le d\beta\tilde{R}_{G}R_Y^2 m^2 \sum_{s=t-2m+1}^{t-m-1}\sum_{r=s+1}^{t-m}\|x_s-x_{s-1}\|_2\\
&\le \eta \beta \tilde{R}_{G}R_Y^2 d^2 B^* m^{\frac{9}{2}}. 
\end{align*}
Thus,
\begin{align*}
\E\left[\nabla \bar{f}_{t-m}(x_{t-m})^{\top}(x_{t-m}-x)-\frac{1}{2}\|x_{t-m}-x\|_{\alpha H_{t-m}}^2\right]&\le \frac{1}{2\eta}\E\left[\|x_{t-1}-x\|_{\hat{A}_{t-m-1}}^2-\|x_t-x\|_{\hat{A}_{t-m}}^2\right]\\
& \ \ \ \ \ +\E[(\nabla \bar{f}_{t-m}(x_{t-m})-\tilde{g}_{t-m})^{\top}(x_{t-m}-x)]\\
& \ \ \ \ \ + 2\eta \sigma \max\{\alpha,1\} d^2(B^*)^2m^3\\
&\le \frac{1}{2\eta}\E\left[\|x_{t-1}-x\|_{\hat{A}_{t-m-1}}^2-\|x_t-x\|_{\hat{A}_{t-m}}^2\right]\\
& \ \ \ \ \ +2\eta \sigma \max\{\alpha,1\} d^2(B^*)^2m^3\\
& \ \ \ \ \ +\eta \beta \tilde{R}_{G}R_Y^2 d^2 B^* Dm^{\frac{9}{2}}. 
\end{align*}
Since $\nabla^2\bar{f}_t=G_t^{\top}Q_tG_t\succeq \alpha H_t$, summing over all iterations, we have that $\forall x\in\K$,
\begin{align*}
\E\left[\sum_{t=m}^{T}\bar{f}_t(x_t)-\bar{f}_t(x)\right]&\le \sum_{t=2m}^{T} \E\left[\bar{f}_{t-m}(x_{t-m})-\bar{f}_{t-m}(x)\right]+mB\\
&\le \sum_{t=2m}^{T} \E\left[\nabla \bar{f}_{t-m}(x_{t-m})^{\top}(x_{t-m}-x)-\frac{1}{2}\|x_{t-m}-x\|_{\alpha H_{t-m}}^2\right]+mB\\
&\le \frac{1}{2\eta} \sum_{t=2m}^T \E\left[\|x_{t-1}-x\|_{\hat{A}_{t-m-1}}^2-\|x_t-x\|_{\hat{A}_{t-m}}^2\right]+2\eta \sigma \max\{\alpha,1\} d^2(B^*)^2m^3T\\
& \ \ \ \ \ +\eta \beta \tilde{R}_{G}R_Y^2 d^2 B^* m^{\frac{9}{2}}T+mB\\
&\le \frac{1}{2\eta}\E\left[\|x_{2m-1}-x\|_{\hat{A}_{m-1}}^2\right] +2\eta \sigma \max\{\alpha,1\} d^2(B^*)^2m^3T\\
& \ \ \ \ \ +\eta \beta \tilde{R}_{G}R_Y^2 d^2 B^* m^{\frac{9}{2}}T+mB\\
&\le \frac{mD^2}{2\eta}+2\eta \sigma \max\{\alpha,1\} d^2(B^*)^2m^3T+\eta \beta \tilde{R}_{G}R_Y^2 d^2 B^* Dm^{\frac{9}{2}}T+mB. 
\end{align*}
\end{proof}

\subsection{Assembling regret bound}

Combining bounds in Section~\ref{sec:perturbation-loss-known}, Section~\ref{sec:movement-cost}, Section~\ref{sec:ons-regret-known} and using that $\sigma\le R_G^2R_Y^2$, the overall expected regret w.r.t. any $x\in\K$ is given by
\begin{align*}
\E[\regret_T(x)]&\le  \beta R_G\left[\frac{10mR_GR_Y}{\kappa(G)}\sqrt{T} + \frac{4md\log(\sigma T) }{\eta\alpha\kappa(G)} +  dR_Y^2\sqrt{T}\right] +  \eta dB^*Lm^{2}T \\
& \ \ \ \ \ +\frac{mD^2}{2\eta}+2\eta \sigma \max\{\alpha,1\} d^2(B^*)^2m^3T+\eta \beta \tilde{R}_{G}R_Y^2 d^2 B^* D m^{\frac{9}{2}}T+mB. 
\end{align*}
By taking $\eta=\tilde{\Theta}\left(\frac{1}{\alpha\sqrt{T}}\right)$, we have
\begin{align*}
\E[\regret_T(x)]=\tilde{O}\left(\frac{\beta}{\alpha}B^*\sqrt{T}\right),
\end{align*}
where $\tilde{O}(\cdot)$ hides logarithmic factors in $T$ and all natural parameters $R_G, \tilde{R}_G, R_Y, \kappa(G), B, L, D, d$.

\subsection{Proofs of Technical Lemmas}

\subsubsection{Proof of Lemma~\ref{lem:YtHt-inequality}}
\label{sec:proof-ytht-inequality}
Fix a sequence of $Y_1,\dots,Y_T$ and $G$. Let $u$ be any unit vector, and by definition of $H_t$. Denote $u_t=Y_tv$, and let $Y_t=0$, $\forall t\le 0$ and $t>T$. Then, 
\begin{align*}
v^{\top}\left(\sum_{t=m}^T H_t\right)v&=\sum_{t=m}^T \left\|\sum_{i=0}^{m-1}G^{[i]}u_{t-i}\right\|_2^2\\
&\ge \sum_{t=1}^{T+m-1}\left\|\sum_{i=0}^{m-1}G^{[i]}u_{t-i}\right\|_2^2-2mR_G^2R_Y^2\\
&\ge_{(1)} \frac{1}{2} \sum_{t=1}^{T+m-1}\left\|\sum_{i=0}^{\infty}G^{[i]}u_{t-i}\right\|_2^2-\sum_{t=1}^{T+m-1}\left\|\sum_{i=m}^{\infty}G^{[i]}u_{t-i}\right\|_2^2-2mR_G^2R_Y^2\\
&\ge_{(2)} \frac{1}{2} \sum_{t=1}^{T+m-1}\left\|\sum_{i=0}^{\infty}G^{[i]}u_{t-i}\right\|_2^2-\frac{R_G^2R_Y^2}{T}-4mR_G^2R_Y^2\\
&=_{(3)}\frac{1}{2} \sum_{t=1}^{\infty}\left\|\sum_{i=0}^{t}G^{[i]}u_{t-i}\right\|_2^2-5mR_G^2R_Y^2\\
&\ge_{(4)} \frac{\kappa(G)}{2}\sum_{t=1}^{T}\|Y_tv\|_2^2-5mR_G^2R_Y^2\\
&=\frac{\kappa(G)}{2}v^{\top}\left(\sum_{t=1}^T Y_t^{\top}Y_t\right)-5mR_G^2R_Y^2.
\end{align*}
where we use the inequality $\|x-y\|_2^2\ge \frac{1}{2}\|x\|_2^2-\|y\|_2^2$ in $(1)$, the decaying assumption on $G$ such that $\sum_{i=m}^{\infty}\|G^{[i]}\|_{\mathrm{op}}\le \frac{R_G}{T}$ in $(2)$, $Y_t=0$ $\forall t\le 0$ and $t>T$ in $(3)$, and definition of $\kappa(G)$ in Assumption~\ref{assumption:convolution-invertibility-modulus} in $(4)$.

\subsubsection{Proof of Lemma~\ref{lem:generalized-pythagorean}}
Consider the functions $f_{y}(v)=\|v-y\|_A^2$ on $C$ and $g_{x,z}(t)=f(tz+(1-t)x)$ defined on $t\in[0,1]$. Since $C$ is convex, we have $tz+(1-t)x\in C$, $\forall t\in[0,1]$. By assumption, $x=\Pi_{C}(y)$, $g_{x,z}$ attains minimum at $t=0$. $g_{x,z}'(0)=2(z-x)^{\top}A(x-y)\ge 0$. Then,
\begin{align*}
\|y-z\|_{A}^2=\|y-x+x-z\|_A^2=\|y-x\|_A^2+\|x-z\|_A^2+2(x-y)^{\top}A(z-x)\ge \|x-z\|_A^2. 
\end{align*}

\section{Regret of \texttt{NBPC} (Section~\ref{subsec:reduction-control})}

\subsection{Reduction from LQ control to BQO-AM}
\label{sec:reduction-control}
Note that by letting $G^{[i]}=\begin{bmatrix}C \\ 
KC\end{bmatrix}(A+BKC)^{i-1}B\in\mathbb{R}^{(d_{\y}+d_{\uv})\times d_{\uv}}$, we have that the observation-control pair $(\y_t,\uv_t)$ reached by playing DRC matrices $(M_1,\dots, M_{t-1})$ can be expressed as
\begin{align*}
\begin{bmatrix}
\y_t\\
\uv_t
\end{bmatrix}=\begin{bmatrix}
\ynat_t\\
K\ynat_t
\end{bmatrix}+\sum_{i=1}^t G^{[i]}(\uv_{t-i}-K\y_{t-i})=\begin{bmatrix}
\ynat_t\\
K\ynat_t
\end{bmatrix}+\sum_{i=1}^t G^{[i]}\left(\sum_{j=0}^{m-1}M_{t-i}^{[j]}\ynat_{t-i-j}\right). 
\end{align*}
The sequence of matrices $G=\{G^{[i]}\}_{i\ge 1}$ is called the Markov operator. For simplicity, we assume that the system dynamics and the stabilizing linear policy $K$ is known in this section. With these knowledge, the above equation implies that the signals $\y_t^K$ can be directly computed by the learner from the observations. 

Moreover, let $d=md_{\y}d_{\uv}$, and let $\embed:(\mathbb{R}^{d_{\uv}\times d_{\y}})^{m}\rightarrow\mathbb{R}^d$ denote the natural embedding of a DRC controller $M$ in $\mathbb{R}^d$, with inverse $\embed^{-1}$, i.e. for $M^{[0]},\dots,M^{[m-1]}\in\mathbb{R}^{d_{\uv}\times d_{\y}}$, $ k \in [m-1]$, $i\in [d_{\uv}]$, $j\in [d_{\y}]$,
\begin{align*}
\embed(M^{[0]},\dots, M^{[m-1]})_{kd_{\uv}d_{\y}+(i-1)d_{\y}+j}=M^{[k]}_{ij}.
\end{align*}

Denote $\embed_{\y}:(\mathbb{R}^{d_{\y}})^{m}\rightarrow \mathbb{R}^{d_{\uv}\times d}$ such that $\forall M^{[0]},\dots,M^{[m-1]}\in\mathbb{R}^{d_{\uv}\times d_{\y}}$,
\begin{align*}
\embed_{\y}(\ynat_{t-m+1},\dots,\ynat_{t})\embed(M^{[0]},\dots,M^{[m-1]})=\sum_{j=0}^{m-1}M^{[j]}\ynat_{t-j}.
\end{align*}
We shorthand $Y_t=\embed_{\y}(\ynat_{t-m+1},\dots,\ynat_{t})$, and $\embed(M)=\embed(M^{[0]},\dots,M^{[m-1]})$. An explicit formulation of $\embed_{\y}$ is given by
\begin{align*}
Y_{i,j+1:j+d_{\y}}=\ynat_{t-k+1} \ \ \ \text{if } \ \ \ j=(k-1)d_{\uv}d_{\y}+(i-1)d_{\y}, \ \ \ \forall i\in[d_{\uv}], k\in[m]. 
\end{align*}
Therefore, for any controller that plays DRC policies, the cost function can be re-written as
\begin{align*}
c_t(\y_t,\uv_t)&=\left(B_t+\sum_{i=0}^{m-1}G^{[i]}Y_{t-i}\embed(M_{t-i})\right)^{\top}\begin{bmatrix}
Q_t & \mathbf{0}_{d_{\y}\times d_{\uv}}\\
\mathbf{0}_{d_{\uv\times d_{\y}}} & R_t
\end{bmatrix}\left(B_t+\sum_{i=0}^{m-1}G^{[i]}Y_{t-i}\embed(M_{t-i})\right)\\
&=:f_t(\embed(M_{t-m+1}),\dots,\embed(M_t)),
\end{align*}
where $B_t=(\ynat_t,K\ynat_t)+\sum_{i=m}^t G^{[i]}Y_{t-i}\embed(M_{t-i})$. Let $\bar{f}_t$ denote its induced unary form, i.e. $\bar{f}_t(\embed(M))=f_t(\embed(M),\dots,\embed(M))$. Note that $f_t$ is an adaptive function of the learners decision before the time $t-m+1$ through $B_t$. $B_t$ is independent of the algorithm's decisions $M_{t-m+1:t}$, and $Y_t$'s are independent of the algorithm's decisions. Therefore, the adversarial model in Assumption~\ref{assumption:adversary-adaptivity} is satisfied. 

\subsection{Regularity conditions}
\label{sec:regualrity-conditions-known}
\begin{lemma}[Bounds on LQR/LQG induced with-memory loss functions] 
\label{lem:regularity-bounds}
$\forall M_1,\dots, M_T\in\M(m,R_{\M})$, 
\begin{itemize}
\item For $m\ge 1+\left(\log\frac{1}{1-\gamma}\right)^{-1}\log T$,
\begin{align*}
&\sum_{i=0}^{\infty} \|G^{[i]}\|_{\mathrm{op}}\le 1+\frac{c\kappa\sqrt{1+\kappa^2}\kappa_C\kappa_B}{\gamma}=: R_G, \ \ \ \sum_{i=m}^{\infty} \|G^{[i]}\|_{\mathrm{op}}\le \frac{R_G}{T},  \\
&\sum_{i=1}^{\infty} \|C(A+BKC)^{i-1}B\|_{\mathrm{op}}\le \frac{c\kappa_C\kappa_B\kappa}{\gamma}=: R_{G_{\y}}, \ \ \ \sum_{i=1}^{\infty} \|C(A+BKC)^{i-1}\|_{\mathrm{op}}\le \frac{c\kappa_C\kappa}{\gamma}=: R_{G_{\y},K},\\
&\max_{1\le i\le m-1}\|(G^{[i]})^{\top}G^{[i]}\|_2\le(1+\kappa^2)\kappa_C^2\kappa_B^2\kappa^2 =: \tilde{R}_G.
\end{align*}
\item $\max_{t\in[T]}\|\ynat_t\|_2\le (R_{G,K}+1)R_{\w,\e}=:R_{\mathrm{nat}}$. %$\max_{t\in[T]}\|\uv_t^{M_t}\|_2\le R_{\M}R_{\mathrm{nat}}.$
\item $\max_{t\in[T]}|c_t(\y_t^{M_{1:t-1}}, \uv_{t}^{M_{1:t}})|=\max_{t\in[T]}|f_t(\embed(M_{t-m+1}),\dots,\embed(M_t))|\le \beta R_{\mathrm{nat}}^2((1+R_{G_{\y}}R_{\M})^2(1+2\kappa^2)+2R_{\M}^2)=:B$. 
\item $\max_{M^1,M^2\in\M(m,R_{\M})}\|\embed(M^1)-\embed(M^2)\|_2\le \sqrt{m\max\{d_{\uv},d_{\y}\}} R_{\M}=:D$. 
\item $\max_{t\in[T]}\|Y_t\|_2\le \sqrt{md_{\y}d_{\uv}^2}R_{\mathrm{nat}}=:R_Y$. 
\item $\max_{t\in[T]}\|\nabla f_t(\embed(M_{t-m+1}),\dots,\embed(M_t))\|\le 2\beta \sqrt{m} R_YR_G (R_YR_GD+2R_{\mathrm{nat}})=:L$.
\item $\kappa (G)\ge \frac{1}{4}\min\{1,\kappa^{-2}\}$. 
\end{itemize}
\end{lemma}
\begin{proof}
By Assumption~\ref{assumption:bounded-dynamics-and stability}, we know that
\begin{align*}
\sum_{i=0}^{\infty}\|G^{[i]}\|_{\mathrm{op}}&\le 1+\sum_{i=1}^{\infty}\sqrt{\|C(A+BKC)^{i-1}B\|_{\mathrm{op}}^2+\|KC(A+BKC)^{i-1}B\|_{\mathrm{op}}^2}\\
&\le 1+c\sqrt{1+\kappa^2}\kappa_C\kappa_B\sum_{i=1}^{\infty}\|HL^{i-1}H^{-1}\|_{2}\\
&\le 1+c\kappa\sqrt{1+\kappa^2}\kappa_C\kappa_B \sum_{i=0}^{\infty} (1-\gamma)^i\\
&=1+\frac{c\kappa\sqrt{1+\kappa^2}\kappa_C\kappa_B}{\gamma}\\
&=R_G,
\end{align*}
where $c$ is some constant possibly depending on the dimension of $H,L$. By similar argument and the choice of $m$, we can establish the other inequalities. Moreover,
\begin{align*}
\max_{1\le i\le m-1}\|(G^{[i]})^{\top}G^{[i]}\|_2\le (1+\kappa^2)\kappa_C^2\kappa_B^2\kappa^2(1-\gamma)^{2(i-i)}\le(1+\kappa^2)\kappa_C^2\kappa_B^2\kappa^2=\tilde{R}_G.  
\end{align*}
Therefore, denote $\xnat_t$ as the would-be state if the linear policy $K$ was played from the beginning of the time, we have
\begin{align*}
\max_{t\in[T]}\|\ynat_t\|_2=\max_{t\in[T]}\left\|C\xnat_t+\e_t\right\|=\max_{t\in[T]}\left\|C\sum_{i=0}^{t-1}(A+BKC)^{i}\w_{t-i}+\e_t\right\|_2\le (R_{G,K}+1)R_{\w,\e}, 
\end{align*}
and
\begin{align*}
\max_{t\in[T]}\left\|\sum_{j=0}^{m-1}M_t^{[j]}\ynat_{t-j}\right\|_2\le \max_{t\in[T]} \|M_t\|_{\ell_1,\mathrm{op}}\max_{t\in[T]}\|\ynat_t\|_2\le R_{\M}R_{\mathrm{nat}},
\end{align*}
which implies that $\max_{t\in[T]}\left\|\y_t^{M_{1:t-1}}\right\|\le R_{\mathrm{nat}}(1+R_{G_{\y}}R_{\M})$. Thus, $c_t$'s are bounded by
\begin{align*}
\max_{t\in[T]} c_t(\y_t^{M_{1:t-1}},\uv_t^{M_{1:t}})&\le \beta\left(\|\y_t^{M_{1:t-1}}\|_2^2+\left\|K\y_t^{M_{1:t-1}}+\sum_{j=0}^{m-1}M_t^{[j]}\ynat_{t-j}\right\|_2^2\right)\\
&\le \beta R_{\mathrm{nat}}^2((1+R_{G_{\y}}R_{\M})^2(1+2\kappa^2)+2R_{\M}^2)\\
&=B. 
\end{align*}
The diameter of $\M(m,R_{\M})$ is given by
\begin{align*}
\max_{M_1,M_2\in\M(m,R_{\M})}\|\embed(M_1)-\embed(M_2)\|_2&\le \sqrt{m} \max_{M_1,M_2\in\M(m,R_{\M})}\max_{0\le j\le m-1}\|M_1^{[j]}-M_2^{[j]}\|_F\\
&\le \sqrt{m\max\{d_{\uv},d_{\y}\}}\max_{M_1,M_2\in\M(m,R_{\M})}\max_{0\le j\le m-1}\|M_1^{[j]}-M_2^{[j]}\|_{\mathrm{op}}\\
&\le \sqrt{m\max\{d_{\uv},d_{\y}\}} R_{\M},
\end{align*}
where the second inequality follows from that $\|A\|_F\le \sqrt{\mathrm{rank}(A)}\sigma_{\max}(A)\le \sqrt{\mathrm{rank}(A)}\|A\|_{\mathrm{op}}$. 

To bound the gradient of $f_t$, denote $\nabla_i f_t$ as the gradient of $f_t$ w.r.t. $\embed(M_{t-m+i})$, and then we have $\forall t\in[T], i\in[m]$,
\begin{align*}
\left\|\nabla_i f_t(\embed(M_{t-m+1}),\dots,\embed(M_t))\right\|_2&\le 2\beta\left\|Y_{t-m+i}^{\top} (G^{[m-i]})^{\top}\left(B_t+\sum_{i=0}^{m-1}G^{[i]}Y_{t-i}\embed(M_{t-i})\right)\right\|_2 \\
&\le 2\beta \max_{t\in[T]}\|Y_t\|_{\mathrm{op}} R_G\left(\max_{t\in[T]}\|B_t\|_2+R_G\max_{t\in[T]}\|Y_t\|_{\mathrm{op}}D\right)
\end{align*}
We can bound $B_t$ and $Y_t$ as follows:
\begin{align*}
&\max_{t\in[T]}\|B_t\|_2=\max_{t\in[T]}\left\|\ynat_t+\sum_{i=m}^t G^{[i]}\left(\sum_{j=0}^{m-1}M_{t-i}^{[j]}\ynat_{t-j}\right)\right\|_2\le R_{\mathrm{nat}}\left(1+\frac{R_GR_{\M}}{T}\right),\\
&\max_{t\in[T]}\|Y_t\|_2\le \sqrt{md_{\y}d_{\uv}^2}\max_{t\in[T]}\|\y_t^{K}\|_2\le \sqrt{md_{\y}d_{\uv}^2}R_{\mathrm{nat}}=:R_Y,
\end{align*}
where the second inequality is given by the definition of $\embed_{\y}$. Thus, assuming $\frac{R_GR_{\M}}{T}\le 1$,
\begin{align*}
\max_{t\in[T]}\|\nabla f_t(\embed(M_{t-m+1}),\dots,\embed(M_t))\|_2&\le \sqrt{m} \max_{t\in[T], 1\le i\le m}\|\nabla_i f_t(\embed(M_{t-m+1}),\dots,\embed(M_t))\|_2\\
&\le 2\beta \sqrt{m} R_YR_G (R_YR_GD+2R_{\mathrm{nat}})\\
&=:L.
\end{align*}
The last inequality on $\kappa(G)\ge \frac{1}{4}\min\{1,\kappa^{-2}\}$ is given by Lemma 3.1 in \cite{simchowitz2020making}. 
\end{proof}

\subsection{\texttt{NBPC}: Newton Bandit Perturbation Controller Algorithm}
\label{sec:nbpc-algorithm}
\begin{algorithm}[H]
\caption{Newton Bandit Perturbation Controller (\texttt{NBPC})}
\label{alg:bandit-control-known}
\begin{flushleft}
{\bf Input:} DRC policy class $\M(m,R_{\M})$, step size $\eta>0$, time horizon $T$, system Markov operator $G$ from system parameters $(A,B,C)$, $(\kappa,\gamma)$-strongly stable linear policy $K$, strong convexity parameter $\alpha>0$.
\end{flushleft}
\begin{algorithmic}[1]
\STATE Initialize: $M_{1}^{[j]}=\dots=M_{m}^{[j]}=\mathbf{0}_{d_{\uv}\times d_{\y}}$, $\forall j\in[m]$, $\tilde{g}_{0:m-1} = \mathbf{0}_{md_{\uv}d_{\y}}$, $ \hat{A}_{0:m-1} = mI_{md_{\uv}d_{\y}\times md_{\uv}d_{\y}}$.
\STATE Sample $\eps_t\sim \sphere^{md_{\uv}d_{\y}-1}$ i.i.d. uniformly at random for $t=1,\dots,m$. 
\STATE Set $\widetilde{M}_t = \embed^{-1}(\embed(M_t) + \hat{A}_{t-1}^{-\frac{1}{2}}\eps_t)$, $t=1,\dots,m$. 
\STATE Play control $\uv_t=K\y_t$, incur cost $c_t(\y_t,\uv_t)$ for $t=1,\dots, m$. 
\FOR{$t = m, \ldots, T$}
%\STATE Play control $\uv_t=\uv_t^{\widetilde{M}_t}=\sum_{j=0}^{m-1}\widetilde{M}_t^{[j]}\ynat_{t-j}$, incur cost $c_t(\y_t,\uv_t)=f_t(\embed(\widetilde{M}_{t-m+1:t}))$. 
\STATE Play control $\uv_t=\uv_t^{\widetilde{M}_t}=K\y_t+\sum_{j=0}^{m-1}\widetilde{M}_t^{[j]}\y_{t-j}^K$, incur cost $c_t(\y_t,\uv_t)=f_t(\embed(\widetilde{M}_{t-m+1:t}))$. 
\STATE System involves as $\x_{t+1}=A\x_t+B\uv_t+\w_t$ and $\y_{t+1}=C\x_{t+1}+\e_{t+1}$. Receive new observation $\y_{t+1}$, compute signal $$\ynat_{t+1}=\y_{t+1}-\sum_{i=1}^{t+1} G^{[i]}\left(\sum_{j=0}^{m-1}\widetilde{M}_{t+1-i}^{[j]}\ynat_{t+1-i-j}\right),$$
where $\forall t\le 0$, $\ynat_t\defeq 0$ and $\widetilde{M}_t\defeq 0$.
\STATE Compute Hessian information matrix $H_t\in\mathbb{R}^{md_{\y}d_{\uv}\times md_{\y}d_{\uv}}$:
\begin{align*}
H_t=G_t^{\top}G_t, \ \ \ G_t=\sum_{i=0}^{m-1}G^{[i]}Y_{t-i}, \ \ \ Y_t=\embed_{\y}(\ynat_{t-m+1},\dots,\ynat_t).
\end{align*}
\STATE Update $\hat{A}_t=\hat{A}_{t-1}+\frac{\eta\alpha}{2}H_t$. 
\STATE Create gradient estimate: $\tilde{g}_{t}=md_{\uv}d_{\y}c_t(\y_t,\uv_t) \sum_{j=0}^{m-1}\hat{A}_{t-1-j}^{\frac{1}{2}}\eps_{t-j}$.
\STATE Update
$ M_{t+1} = \embed^{-1}\left(\prod_{\embed(\M(m,R_{\M})}^{\hat{A}_{t-m+1}} \left[ \embed(M_t) - \eta \hat{A}_{t-m+1}^{-1}  \tilde{g}_{t-m+1}  \right]\right)$.
\STATE Sample $\eps_{t+1}\sim \sphere^{md_{\uv}d_{\y}-1}$ uniformly at random, independent of previous steps.  
\STATE Set $\widetilde{M}_{t+1}=\embed^{-1}(\embed(M_{t+1})+\hat{A}_{t-m+1}^{-\frac{1}{2}}\eps_{t+1})$. 
\ENDFOR
\end{algorithmic}
\end{algorithm}

\subsection{Proof of Theorem~\ref{thm:nbpc-regret}}
Let $R_G, R_Y, B, D, L$ be consistent with the notations in Lemma~\ref{lem:regularity-bounds}, and 
$$M^*=\argmin_{M\in\M(m,R_{\M})}\sum_{t=1}^T c_t(\y_t^M,\uv_t^M).$$  
Note that treating the first $2m-2$ steps as burn-in loss, we have
\begin{align*}
\regret_T(\texttt{NBPC})&\le 2mB+\left(\sum_{t=2m-1}^Tf_t(\embed(\widetilde{M}_{t-m+1:t}))-\bar{f}_t(\embed(M^*))\right)\\
& \ \ \ \ \ +\left(\sum_{t=2m-1}^T\bar{f}_t(\embed(M^*))-c_t(\y_t^{M^*},\uv_t^{M^*})\right).
\end{align*}
By Theorem~\ref{thm:bqo-m-regret} and Lemma~\ref{lem:regularity-bounds}, by taking $\eta=\tilde{\Theta}\left(\frac{1}{\alpha\sqrt{T}}\right)$, the second term is bounded by $\tilde{O}\left(\frac{\beta^2}{\alpha}\sqrt{T}\right)$, where $\tilde{O}$ hides logarithmic terms in $T$ and natural parameters as those derived in Lemma~\ref{lem:regularity-bounds}. The third term can be bounded as the following:
\begin{align}
\label{eq:comparator-loss-known}
\sum_{t=2m-1}^T\bar{f}_t(\embed(M^*))-c_t(\y_t^{M^*},\uv_t^{M^*})
\le L\sum_{t=2m-1}^T \left\|\sum_{i=m}^t G^{[i]}Y_{t-i}(\embed(M^*)-\embed(M_{t-i}))\right\|_2\le L R_GR_YD. 
\end{align}

\section{Extension to control of unknown systems (Section~\ref{subsec:reduction-control})}
\label{appendix:unknown}

Consider the setting the dynamics of the system are unknown to the learner. In this case, the learner needs to first estimate the system dynamics and obtain an estimated Markov operator $\hat{G}$, and then use $\hat{G}$ in place of $G$ in the control algorithm.

This section will be organized as the following: Section~\ref{subsubsec:known-stabilizing-controller} will analyze the regret guarantee if instead of the true Markov operator $G$, the learner receives an estimate $\hat{G}$ satisfying $\|\hat{G}-G\|_{\ell_1,\mathrm{op}}\le \eps_G$. 

\subsection{Known Stabilizing Controller}
\label{subsubsec:known-stabilizing-controller}
First, we consider the case where a $(\kappa, \gamma)$-stabilizing linear policy $K$ is known to the learner, although the learner has no information to the system's dynamics $(A,B,C)$ and therefore cannot compute $G$. \cite{simchowitz2020making} has studied this problem in the full information setting, where they proved a quadratic sensitivity of the control regret to the estimation error $\eps_G$ of $\hat{G}$. The main difference between our setting and and that of \cite{simchowitz2020making} comes from the bandit gradient estimation step and the delayed updates in Algorithm~\ref{alg:bqo-memory} to establish conditional independence. 

We start with the assumption that there exists an estimation algorithm that returns a $\hat{G}$ sufficiently close to the true Markov operator $G$:
\begin{assumption} [Accurate system estimator]
\label{assumption:estimation-markov}
Assume that the learner has access to an estimate $\hat{G}$ of the Markov operator $G$ satisfying
\begin{align*}
\|G-\hat{G}\|_{\ell_1,\mathrm{op}}\le \eps_G, \ \ \ \hat{G}^{[i]}=\mathbf{0}_{(d_{\y}+d_{\uv})\times d_{\uv}}, \forall i\ge m.
\end{align*}
\end{assumption}

\begin{remark}
The system estimation via least-square algorithm used in \citep{simchowitz2020improper, simchowitz2020making, sun2023optimal} satisfies the conditions in Assumption~\ref{assumption:estimation-markov} after $N=O(1/\eps_G^2)$ iterations with high probability (see, for instance, Theorem 6b in \citep{simchowitz2020improper}). Since we pay at most $BN$-regret during the estimation step, $\tilde{O}(\sqrt{T})$ is preserved if $\eps_G$ propagates quadratically in the regret bound (i.e. the learner's regret suffers an additional $\tilde{O}(T\eps_G^2)$ due to estimation error). 
\end{remark}

\paragraph{Preliminaries.} During the course of the algorithm, the learner uses the Markov operator to compute $\hat{\y}_t^K$. Therefore, with an estimated Markov operator, the $\y_t^K$'s computed by the learner are no longer accurate. When reducing to BQO-AM, the loss functions that the learner sees become
\begin{align*}
\hat{F}_t(\embed(M_{t-m+1:t}))=\left(\hat{B}_t+\sum_{i=0}^{m-1}\hat{G}^{[i]}\hat{Y}_{t-i}\embed(M_{t-i})\right)^{\top}\begin{bmatrix}
Q_t & \mathbf{0}_{d_{\y}\times d_{\uv}} \\
\mathbf{0}_{d_{\uv}\times d_{\y}} & R_t
\end{bmatrix}
\left(\hat{B}_t+\sum_{i=0}^{m-1}\hat{G}^{[i]}\hat{Y}_{t-i}\embed(M_{t-i})\right),
\end{align*}
where $\hat{B}_t=(\hat{\y}_t^K, K\hat{\y}_t^K), \hat{Y}_t$ are the counterparts of $B_t, Y_t$ defined in Section~\ref{subsec:reduction-control} using estimates $\hat{G}$ and $\hat{\y}_t^K$. Note that we do not have the additional terms in $\hat{B}_t$ that depends on the learner's past controls since $\hat{G}^{[i]}=\mathbf{0}_{(d_{\y}+d_{\uv})\times d_{\uv}}$ for $i\ge m$. First, we consider a pseudo-loss function:
\begin{align*}
\tilde F_t(\embed(M_{t-m+1:t}))=\left(\tilde{B}_t+\sum_{i=0}^{m-1}G^{[i]}\hat{Y}_{t-i}\embed(M_{t-i})\right)^{\top}\begin{bmatrix}
Q_t & \mathbf{0}_{d_{\y}\times d_{\uv}} \\
\mathbf{0}_{d_{\uv}\times d_{\y}} & R_t
\end{bmatrix}
\left(\tilde{B}_t+\sum_{i=0}^{m-1}G^{[i]}\hat{Y}_{t-i}\embed(M_{t-i})\right),
\end{align*}
where $\tilde{B}_t=(\y_t^K, K\y_t^K)$.
Let $\hat{f}_t,\tilde{f}_t$ denote their induced unary form, respectively. Although the learner has no access to $\tilde{F}_t$ (which requires the knowledge of $G$), Proposition D.8 in \cite{simchowitz2020making} suggests that the control regret for unknown system is implied by the bound for some $\nu>0$ on the following quantity:
\begin{align}
\label{eqn:reduced-control-regret}
\E\left[\sum_{t=m}^T\tilde{F}_t(\embed(\widetilde{M}_{t-m+1:t}))-\tilde{f}_t(\embed(M))+\nu\sum_{t=m}^T\|Y_t(\embed(M_{t-m+1})-\embed(M))\|_2^2\right],
\end{align}
where $\widetilde{M}_t$ is the DRC policy played by the learner at time $t$, and $M_t$ is the projected DRC policy by Algorithm~\ref{alg:bandit-control-known} at time $t$.
Therefore, in the rest of this section, we focus on bounding the quantity in \cref{eqn:reduced-control-regret}.

\subsection{High-level proof}

Throughout this section, we assume that Assumption~\ref{assumption:estimation-markov} holds, i.e. we have access to a sufficiently accurate estimator $\hat{G}$ of the Makrov operator $G$. For a comparator $x\in\embed(\M(m,R_{\M}))$, we denote
\begin{align*}
\widetilde{\regret}_T(x)=\sum_{t=m}^T \tilde{F}_t(y_{t-m+1:t})-\tilde{f}_t(x), \ \ \ \widehat{\regret}_T(x)=\sum_{t=m}^T \hat{F}_t(y_{t-m+1:t})-\hat{f}_t(x),
\end{align*}
where $y_t=\embed(\widetilde{M}_t)$. Similarly, we denote $x_t=\embed(M_t)$. Suppose Algorithm~\ref{alg:bqo-memory} is run with $\hat{F}_t$ with a slightly modified cumulative Hessian $\hat{A}_t=mI+\frac{\eta\alpha}{6}\sum_{s=m}^t\hat{H}_s$, where $\hat{H}_t=\hat{G}_t^{\top}\hat{G}_t$, $\hat{G}_t=\sum_{i=0}^{m-1}\hat{G}^{[i]}\hat{Y}_{t-i}$. We start by noting that $\widehat{\regret}_T(x)$ can be bounded via the regret bound for known systems by treating $\hat{G}$ as the true Markov operator. 

\begin{lemma}
\label{lem:bco-regret-unknown}
$\forall x\in\embed(\M(m,R_{\M}))$, $\E\left[\widehat{\regret}_T(x)\right]\le \tilde{O}(\sqrt{T})$. 
\end{lemma}

Lemma~\ref{lem:control-regret-unknown} and Lemma~\ref{lem:regret-negative-movement} relates the regret of controlling an unknown system to $\widehat{\regret}_T(x)$. 
\begin{lemma}
\label{lem:control-regret-unknown}
$\forall \nu>0$, $\exists x^*\in\embed(\M(\frac{m}{3}, \frac{R_{\M}}{2}))$ such that
\begin{align*}
\E\left[\mathrm{ControlRegret}_T\right]\le \E[\widetilde{\regret}_T(x^*)]+\tilde{O}(1)\cdot T\eps_G^2\left(1+\frac{1}{\nu}\right)+\nu\sum_{t=m}^T \E[\|\hat{Y}_t(x_t-x^*)\|_2^2]+\tilde{O}(1),
\end{align*}
where $\mathrm{ControlRegret}_T:=\sum_{t=1}^T c_t(\y_t,\uv_t)-\min_{M\in\M(\frac{m}{3},\frac{R_{\M}}{2})}\sum_{t=1}^T c_t(\y_t^M, \uv_t^M)$.
\end{lemma}

Lemma~\ref{lem:regret-negative-movement} is an analogous result to Lemma 5.1-5.4 in \cite{simchowitz2020making}, providing a relationship between the two quantities defined above. 

\begin{lemma}
\label{lem:regret-negative-movement}
$\forall \nu>0$, the regret for any $x\in\embed(\M(m,R_{\M}))$ for $\tilde{F}_t$ is bounded by
\begin{align*}
\E[\widetilde{\regret}_T(x)]&\le 2B(\E[\widehat{\regret}_T(x)])+\frac{\nu}{m}\sum_{t=m}^T\sum_{i=0}^{m-1}\E\left[ \|\hat{Y}_{t-i}(x_{t}-x)\|_2^2\right]-\frac{\alpha}{12}\sum_{t=m}^T \E\left[\|x_t-x\|_{\tilde{H}_t}^2\right]\\
& \ \ \ \ \ + \left(\beta R_{\hat{Y}}^2+\frac{\alpha R_{\hat{Y}}^2D^2}{3}+\frac{12}{\alpha}\beta^2D^2+\frac{(\beta DR_G)^2m}{\nu}\right)\eps_G^2T,
\end{align*}
where $B(\cdot)$ denotes any upper bound on a given quantity, and $\tilde{H}_t$ is defined analogously to $\hat{H}_t$ for $\tilde{F}_t$'s.
\end{lemma}

\begin{lemma}
\label{lem:HtYt-movement-inequality}
For $\nu=\frac{\alpha \kappa(G)}{192}$, we have
\begin{align*}
\frac{2\nu}{m}\sum_{t=m}^T\sum_{i=0}^{m-1}\E\left[ \|\hat{Y}_{t-i}(x_{t}-x)\|_2^2\right]-\frac{\alpha}{12}\sum_{t=m}^T \E\left[\|x_t-x\|_{\tilde{H}_t}^2\right]\le \tilde{O}(\sqrt{T}), 
\end{align*}
and therefore by Lemma~\ref{lem:regret-negative-movement},
\begin{align*}
\E[\widetilde{\regret}_T(x)]&\le 2B(\E[\widehat{\regret}_T(x)])-\frac{\alpha\kappa(G)}{192m}\sum_{t=m}^T\sum_{i=0}^{m-1}\E\left[ \|\hat{Y}_{t-i}(x_{t}-x)\|_2^2\right]+\tilde{O}(1)\cdot \eps_G^2T.
\end{align*}
\end{lemma}
Combining Lemma~\ref{lem:bco-regret-unknown}-\ref{lem:HtYt-movement-inequality}, we conclude that
\begin{align*}
\E\left[\mathrm{ControlRegret}_T\right]\le \tilde{O}(\sqrt{T})+O(T\eps_G^2).
\end{align*}

\subsection{Omitted proofs in Section~\ref{subsubsec:known-stabilizing-controller}}

\subsubsection{Proof of Lemma~\ref{lem:bco-regret-unknown}}
To see the desired regret bound, it suffices to derive analogous bounds to Lemma~\ref{lem:regularity-bounds} to bound the natural parameters necessary for analysis in Section~\ref{sec:bqo-AM}. First, we have by Assumption~\ref{assumption:estimation-markov},
\begin{align*}
R_{\hat{G}}\le R_G+\eps_G, \ \ \ \tilde{R}_{\hat{G}}\le \tilde{R}_G\eps_G(R_{\hat{G}}+R_G), \ \ \ \kappa(\hat{G})\ge \frac{1}{4}\min\{1,\kappa^{-2}\}.
\end{align*}
It is left to establish bounds on $\hat{B}, \hat{L}, R_{\hat{Y}}$. Note that assuming $\eps_G\le \frac{1}{2R_{\M}}$,
\begin{align*}
\max_{t\in[T]}\left\|\sum_{j=0}^{m-1}M_t^{[j]}\hat{\y}_{t-j}^K\right\|_2&\le R_{\M}\max_{t\in[T]}\max_{t-m+1\le s\le t} \|\hat{\y}_s^K\|_2\\
&\le R_{\M}\left(R_{\mathrm{nat}}+\max_{t\in[T]}\max_{t-m+1\le s\le t}\left\|\sum_{i=0}^{s}(G_{\y}^{[i]}-\hat{G}_{\y}^{[i]})\sum_{j=0}^{m-1}M_{s-i}^{[j]}\hat{\y}_{s-i-j}^K\right\|_2\right)\\
&\le R_{\M}\left(R_{\mathrm{nat}}+\eps_G\max_{t\in[T]}\left\|\sum_{j=0}^{m-1}M_t^{[j]}\hat{\y}_{t-j}^K\right\|_2\right),
\end{align*}
and thus $\max_{t\in[T]}\left\|\sum_{j=0}^{m-1}M_t^{[j]}\hat{\y}_{t-j}^K\right\|_2\le 2R_{\M}R_{\mathrm{nat}}$, which implies $\max_{t\in[T]}\|\hat{\y}_t^K\|_2\le 2R_{\mathrm{nat}}$. Therefore, following the analysis of Lemma~\ref{lem:regularity-bounds}, $\hat{B}\le 4B$, $R_{\hat{Y}}\le 2R_Y$, $\hat{L}\le 8L$. 

\subsubsection{Proof of Lemma~\ref{lem:control-regret-unknown}}
Denote 
\begin{align*}
M^*=\argmin_{M\in\M(m,R_{\M})}\sum_{t=1}^T c_t(\y_t^M, \uv_t^M), \  \ \ x^*=\embed(M^*).
\end{align*}
When the system is unknown and the control algorithm is run with an estimated Markov operator $\hat{G}$, we have that
\begin{align*}
c_t(\y_t,\uv_t)=\left\|\tilde{B}_t+\sum_{i=0}^{t} G^{[i]}\hat{Y}_{t-i}\embed(\widetilde{M}_{t-i})\right\|_{P_t}^2.
\end{align*}

The control regret for unknown system with an estimated Markov operator $\hat{G}$ can be decomposed as
\begin{align*}
\mathrm{ControlRegret}_T&=\underbrace{\sum_{t=1}^{2m-2}c_t(\y_t,\uv_t)}_{\le 2m\hat{B}}+\underbrace{\sum_{t=2m-1}^Tc_t(\y_t,\uv_t)-\tilde{F}_t(y_{t-m+1:t})}_{(\text{loss approximation error})}\\
& \ \ \ \ \ + \underbrace{\sum_{t=2m-1}^T\tilde{F}_t(y_{t-m+1:t})-\tilde{f}_t(x^*)}_{\widetilde{\regret}_T(x^*)} + \underbrace{\sum_{t=2m-1}^T \tilde{f}_t(x^*)-\bar{f}_t(x^*)}_{(\text{comparator approximation error})} \\
& \ \ \ \ \ + \underbrace{\sum_{t=2m-1}^T \bar{f}_t(x^*)-\sum_{t=2m-1}^T c_t(\y_t^{M^*},\uv_t^{M^*})}_{(\text{control approximation error})}.
\end{align*}
We bound each of the terms as follows: since
\begin{align*}
\|\Delta_t\|_2:=\left\|\sum_{i=m}^t G^{[i]}\hat{Y}_{t-i}\embed(\widetilde{M}_{t-i})\right\|_2\le \frac{R_GR_{\hat{Y}}\hat{D}}{T},
\end{align*}
thus
\begin{align*}
(\text{loss approximation error})&\le \beta \sum_{t=2m-1}^T\|\Delta_t\|_2\left(2\left\|\tilde{B}_t+\sum_{i=0}^{t}G^{[i]}\hat{Y}_{t-i}x_{t-i}\right\|_2+\|\Delta_t\|_2\right)\\
&=\tilde{O}(1).
\end{align*}
The control approximation error is bounded by $\tilde{O}(1)$ by Eq.~\ref{eq:comparator-loss-known}. By Lemma D.10 in \cite{simchowitz2020making}, 
\begin{align*}
(\text{comparator approximation error})\le \nu\sum_{t=m}^T \|\hat{Y}_t(x_t-x^*)\|_2^2+ \tilde{O}(1)\cdot T\eps_G^2\left(1+\frac{1}{\nu}\right)+\tilde{O}(1).
\end{align*}

\subsubsection{Proof of Lemma~\ref{lem:regret-negative-movement}}
\begin{lemma} [Gradient error]
\label{lem:gradient-error}
Let $c(\beta,\kappa,G,\hat{G},D,\hat{Y})=2\beta R_{\hat{Y}}^2\left(\sqrt{1+\kappa^2}+D(R_G+R_{\hat{G}})\right)$. Then, there holds
\begin{align*}
&\max_{t\in[T]}\max_{M_{t-m+1:t}\in\M(m,R_{\M})}\|\nabla \hat{F}_t(\embed(M_{t-m+1:t})) - \nabla\tilde{F}_t(\embed(M_{t-m+1:t}))\|_2\le c(\beta,\kappa,G,\hat{G},D,\hat{Y})\eps_G\sqrt{m},\\
&\nabla \tilde{f}_t(\embed(M_t))-\nabla \hat{f}_t(\embed(M_t))=2\tilde{G}_t^{\top}P_t(\tilde{G}_t-\hat{G}_t)\embed(M_t)+2(\tilde{G}_t-\hat{G}_t)^{\top} P_t\hat{G}_t\embed(M_t),
\end{align*}
where $P_t=\begin{bmatrix}
Q_t & \mathbf{0}_{d_{\y}\times d_{\uv}}\\
\mathbf{0}_{d_{\uv}\times d_{\y}} & R_t
\end{bmatrix}$. 
\end{lemma}
\begin{proof}
The second equality follows directly by taking the gradient of the functions. To bound the difference of the gradients of $\hat{F}_t$ and $\tilde{F}_t$, Note that $\forall t$ and $\forall M_{t-m+1:t}$,
\begin{align*}
\|\nabla_i \hat{F}_t(\embed(M_{t-m+1:t})) - \nabla_i \tilde{F}_t(\embed(M_{t-m+1:t}))\|_2&\le 2\beta R_{\hat{Y}}\|(\hat{G}^{[m-i]}-G^{[m-i]})\hat{B}_t\|_2\\
& \ \ \ \ \ +2\beta R_{\hat{Y}}^2 D\left\|\sum_{j=0}^{m-1}(\hat{G}^{[m-i]})^{\top}\hat{G}^{[j]}-(G^{[m-i]})^{\top}\hat{G}^{[j]}\right\|\\
&\le 2\beta R_{\hat{Y}}^2(\sqrt{1+\kappa^2}+D(R_G+R_{\hat{G}}))\eps_G. 
\end{align*}
\end{proof}

Let $y_t=\embed(\widetilde{M}_t)$. We again decompose the regret into perturbation loss, movement cost, and underlying regret and bound each separately. We define the following quantities: denote
\begin{align*}
&\widehat{\mathrm{PerturbLoss}}:=\E\left[\sum_{t=m}^T \hat{F}_t(y_{t-m+1:t})-\hat{F}_t(x_{t-m+1:t})\right], \\ &\widetilde{\mathrm{PerturbLoss}}:=\E\left[\sum_{t=m}^T \tilde{F}_t(y_{t-m+1:t})-\tilde{F}_t(x_{t-m+1:t})\right], \\
&\widehat{\mathrm{MoveCost}}:=\E\left[\sum_{t=m}^T \hat{F}_t(x_{t-m+1:t})-\hat{f}_t(x_{t})\right], \ \widetilde{\mathrm{MoveCost}}:=\E\left[\sum_{t=m}^T \tilde{F}_t(x_{t-m+1:t})-\tilde{f}_t(x_{t})\right], \\
&\widehat{\mathrm{ONSRegret}(x)}:=\E\left[\sum_{t=m}^T \hat{f}_t(x_{t})-\hat{f}_t(x)\right], \ \widetilde{\mathrm{ONSRegret}(x)}:=\E\left[\sum_{t=m}^T \tilde{f}_t(x_{t})-\tilde{f}_t(x)\right].
\end{align*}

\paragraph{Bounding perturbation loss.} Similar to the analysis in Section~\ref{sec:perturbation-loss-known}, we have
\begin{align*}
\widetilde{\mathrm{PerturbLoss}}&= \frac{1}{2}\sum_{t=m}^T\sum_{i=1}^m \E\left[u_{t-m+i}^{\top}\hat{A}_{t-m+i-1}^{-\frac{1}{2}}[\nabla^2 \tilde{F}_t]_{ii}\hat{A}_{t-m+i-1}^{-\frac{1}{2}}u_{t-m+i}\right]\\
&\le \frac{\beta}{2} \sum_{t=m}^T\sum_{i=1}^m \E\left[\left\|G^{[m-i]}\hat{Y}_{t-m+i}\hat{A}_{t-m+i-1}^{-\frac{1}{2}}u_{t-m+i}\right\|_2^2\right]\\
&\le \underbrace{\beta \sum_{t=m}^T\sum_{i=1}^m \E\left[\left\|\hat{G}^{[m-i]}\hat{Y}_{t-m+i}\hat{A}_{t-m+i-1}^{-\frac{1}{2}}u_{t-m+i}\right\|_2^2\right]}_{\le 2B(\widehat{\mathrm{PerturbLoss}})}\\
& \ \ \ \ \ +\beta \sum_{t=m}^T\sum_{i=1}^m \E\left[\left\|\left(\hat{G}^{[m-i]}-G^{[m-i]}\right)\hat{Y}_{t-m+i}\hat{A}_{t-m+i-1}^{-\frac{1}{2}}u_{t-m+i}\right\|_2^2\right].
\end{align*}
The bound on the first term can be derived identically as in Proposition~\ref{prop:perturbation-loss}. The second term can be bounded as
\begin{align*}
\beta \sum_{t=m}^T\sum_{i=1}^m \E\left[\left\|\left(\hat{G}^{[m-i]}-G^{[m-i]}\right)\hat{Y}_{t-m+i}\hat{A}_{t-m+i-1}^{-\frac{1}{2}}u_{t-m+i}\right\|_2^2\right]&\le \eps_G^2\beta R_{\hat{Y}}^2T. 
\end{align*}

\paragraph{Bounding movement cost.} The bound for movement cost is simply the Lipschitz constant multiplying the Euclidean distance between iterates. Therefore, let $\hat{L}$ denote the Lipschitz constant for $\hat{F}_t$, respectively, we have
\begin{align*}
\widetilde{\mathrm{MoveCost}}\le \frac{\max_{t\in[T], M_{t-m+1:t}}\|\nabla \tilde{F}_t(M_{t-m+1:t})\|_2}{\hat{L}}B(\widehat{\mathrm{MoveCost}})\le 2B(\widehat{\mathrm{MoveCost}}),
\end{align*}
where the last inequality follows from Lemma~\ref{lem:gradient-error} and assuming $\eps_G\le (c(\beta,\kappa,G,\hat{G},D,\hat{Y})\sqrt{m})^{-1}L$.

\paragraph{Bounding underlying regret.}
By construction, we have $\alpha\tilde{H}_t\preceq \nabla^2\tilde{f}_t\preceq \beta\tilde{H}_t$ and $\alpha\hat{H}_t\preceq\nabla^2\hat{f}_t\preceq \beta\hat{H}_t$, and thus $\forall x\in\K$,
\begin{align*}
\widetilde{\mathrm{ONSRegret}(x)}&\le \sum_{t=m}^T\E\left[\nabla \tilde{f}_{t}(x_{t})^{\top}(x_{t}-x)-\frac{1}{2}\|x_{t}-x\|_{\alpha\tilde{H}_{t}}^2\right]\\
&=\sum_{t=m}^T\E\left[\nabla \hat{f}_{t}(x_{t})^{\top}(x_{t}-x)-\frac{1}{6}\|x_{t}-x\|_{\alpha\hat{H}_{t}}^2\right] \\
& \ \ \ \  \ + \sum_{t=m}^T\E\left[ (\nabla\tilde{f}_{t}(x_t)-\nabla\hat{f}_{t}(x_{t}))^{\top}(x_{t}-x)\right]+\frac{\alpha}{6}\sum_{t=m}^T\E\left[\|x_t-x\|_{\hat{H}_t-3\tilde{H}_t}^2\right]\\
&\le B(\widehat{\mathrm{ONSRegret}(x)})+\sum_{t=m}^T\E\left[ (\nabla\tilde{f}_{t}(x_t)-\nabla\hat{f}_{t}(x_{t}))^{\top}(x_{t}-x)\right]-\frac{\alpha}{6}\sum_{t=m}^T \E\left[\|x_t-x\|_{\tilde{H}_t}^2\right]\\
& \ \ \ \ \ + \frac{\alpha}{3}\sum_{t=m}^T \E\left[\|(\hat{G}_t-\tilde{G}_t)^{\top}(x_t-x)\|_2^2\right]\\
&\le B(\widehat{\mathrm{ONSRegret}(x)})+\sum_{t=m}^T\E\left[ (\nabla\tilde{f}_{t}(x_t)-\nabla\hat{f}_{t}(x_{t}))^{\top}(x_{t}-x)\right]-\frac{\alpha}{6}\sum_{t=m}^T \E\left[\|x_t-x\|_{\tilde{H}_t}^2\right]\\
& \ \ \ \ \ + \frac{\alpha \eps_G^2R_{\hat{Y}}^2D^2T}{3}.
\end{align*}

By Lemma~\ref{lem:gradient-error}, we have $\forall t$,
\begin{align*}
\frac{1}{2}(\nabla \tilde{f}_t(x_t)-\nabla \hat{f}_t(x_t))^{\top}(x_t-x)&=\underbrace{x_t^{\top}(\tilde{G}_t-\hat{G}_t)^{\top}P_t^{\top}\tilde{G}_t(x_t-x)}_{(1)}+\underbrace{x_t^{\top}\hat{G}_t^{\top}P_t^{\top}(\tilde{G}_t-\hat{G}_t)(x_t-x)}_{(2)}\\
&\le \underbrace{\frac{6}{\alpha}\beta^2D^2\eps_G^2+\frac{\alpha}{24}\|x_t-x\|_{\tilde{H}_t}^2}_{(1)}+\underbrace{\frac{(\beta DR_G)^2m}{2\nu}\eps_G^2+\frac{\nu}{2m}\sum_{i=0}^{m-1}\|\hat{Y}_{t-i}(x_t-x)\|_2^2}_{(2)},
\end{align*}
where the inequality follows by applying Cauchy-Schwarz $u^{\top}v\le \frac{1}{2\lambda}\|u\|_2^2+\frac{\lambda}{2}\|v\|_2^2$ with $\lambda=\frac{\alpha}{12}$ for $(1)$ and $\lambda=\frac{\nu}{m}$ for $(2)$.

The desired inequality follows from combining the bounds.

\subsubsection{Proof of Lemma~\ref{lem:HtYt-movement-inequality}}
Using the same blocking method and definitions of $k_j$'s as in Section~\ref{sec:perturbation-loss-known}, we have
\begin{align*}
\frac{2\nu}{m}\sum_{t=m}^T\sum_{i=0}^{m-1}\E\left[ \|\hat{Y}_{t-i}(x_{t}-x)\|_2^2\right]&\le \frac{2\nu}{m}\sum_{j=1}^J \sum_{t=k_j}^{k_{j+1}-1}\sum_{i=0}^{m-1}\E\left[\|\hat{Y}_{t-i}(x_t-x)\|_2^2\right]+2\nu R_{\hat{Y}}^2D^2\tau\\
&\le \frac{4\nu}{m}\sum_{j=1}^J \sum_{t=k_j}^{k_{j+1}-1}\sum_{i=0}^{m-1}\E\left[\|\hat{Y}_{t-i}(x_{k_j}-x)\|_2^2\right]\\
& \ \ \ \ \ +\frac{4\nu}{m}\sum_{j=1}^J\sum_{t=k_j}^{k_{j+1}-1}\sum_{i=0}^{m-1}\E\left[\|\hat{Y}_{t-i}(x_{t}-x_{k_j})\|_2^2\right]+2\nu R_{\hat{Y}}^2D^2\tau.
\end{align*}
Applying similar analysis as in Section~\ref{sec:perturbation-loss-known}, 
\begin{align*}
\frac{4\nu}{m}\sum_{j=1}^J \sum_{t=k_j}^{k_{j+1}-1}\sum_{i=0}^{m-1}\E\left[\|\hat{Y}_{t-i}(x_{k_j}-x)\|_2^2\right]&\le \frac{8\nu}{\kappa(G)}\sum_{j=1}^J\left[\left(\sum_{t=k_j}^{k_{j+1}-1} \E\left[\|x_{k_j}-x\|_{\tilde{H}_t}^2\right]\right)+5mR_GR_{\hat{Y}}D^2\right]\\
&\le \frac{8\nu}{\kappa(G)}\sum_{j=1}^J \sum_{t=k_j}^{k_{j+1}-1}\E\left[\|x_{k_j}-x\|_{\tilde{H}_t}^2\right]+\frac{40\nu mR_GR_{\hat{Y}}D^2}{\kappa(G)}\left\lfloor \frac{T}{\tau}\right\rfloor.
\end{align*}
On the other hand,
\begin{align*}
\frac{\alpha}{12}\sum_{t=m}^T \E\left[\|x_t-x\|_{\tilde{H}_t}^2\right]&\ge \frac{\alpha}{24}\sum_{j=1}^J\sum_{t=k_j}^{k_{j+1}-1}\E\left[\|x_{k_j}-x\|_{\tilde{H}_t}^2\right]-\frac{\alpha}{12}\sum_{j=1}^J\sum_{t=k_j}^{k_{j+1}-1}\E\left[\|x_{t}-x_{k_j}\|_{\tilde{H}_t}^2\right].
\end{align*}
By choice of $\nu=\frac{\alpha\kappa(G)}{192}$, the left hand side of the desired inequality is upper bounded by
\begin{align*}
\frac{\alpha\kappa(G)}{48m}\underbrace{\sum_{j=1}^J\sum_{t=k_j}^{k_{j+1}-1}\sum_{i=0}^{m-1}\E\left[\|\hat{Y}_{t-i}(x_{t}-x_{k_j})\|_2^2\right]}_{(1)}+\frac{\alpha}{12}\underbrace{\sum_{j=1}^J\sum_{t=k_j}^{k_{j+1}-1}\E\left[\|x_{t}-x_{k_j}\|_{\tilde{H}_t}^2\right]}_{(2)}+2\nu R_{\hat{Y}}^2D^2\tau\\
+\frac{40\nu mR_GR_{\hat{Y}}D^2}{\kappa(G)}\left\lfloor \frac{T}{\tau}\right\rfloor.
\end{align*}
We bound the first two terms. First, note that assuming $\frac{\eta\alpha\eps_GR_{\hat{Y}}T}{6}\le 1$, we have
\begin{align*}
\hat{H}_t&\preceq \tilde{H}_t+\sum_{i=0}^{m-1}(\hat{G}^{[i]}-G^{[i]})\hat{Y}_{t-i}\preceq \tilde{H}_t+\eps_GR_{\hat{Y}}I, \\
\hat{A}_t&=mI+\frac{\eta\alpha}{6}\sum_{s=m}^t \hat{H}_t\preceq mI+\frac{\eta\alpha}{6}\sum_{s=m}^t \tilde{H}_t+\frac{\eta\alpha}{6}\cdot \eps_GR_{\hat{Y}}tI\preceq 2\tilde{A}_t.
\end{align*}
First, we bound $(2)$. Note that
\begin{align*}
\|x_t-x_{k_j}\|_{\tilde{H}_t}^2&\le 2\|\tilde{G}_t\tilde{A}_{t}^{-\frac{1}{2}}\hat{A}_t^{\frac{1}{2}}(x_t-x_{k_j})\|_2^2\\
&\le 2\trace(\tilde{A}_t^{-\frac{1}{2}}\tilde{H}_t\tilde{A}_t^{-\frac{1}{2}})\|x_t-x_{k_j}\|_{\hat{A}_t}^{2}\\
&\le \frac{12}{\eta\alpha} \frac{\log(|\tilde{A}_t|)}{\log(|\tilde{A}_{t-1}|)}\cdot \tau \sum_{s=k_j+1}^t \|x_s-x_{s-1}\|_{\hat{A}_t}^2\\
&\le \frac{12}{\eta\alpha} \frac{\log(|\tilde{A}_t|)}{\log(|\tilde{A}_{t-1}|)}\cdot \tau \max\{2,\eta\alpha\hat{\sigma}(\tau+m)\} \sum_{s=k_j+1}^t \|x_s-x_{s-1}\|_{\hat{A}_{s-m}}^2\\
&\le \frac{12}{\eta\alpha} \frac{\log(|\tilde{A}_t|)}{\log(|\tilde{A}_{t-1}|)}\cdot \tau^2 \max\{2,\eta\alpha\hat{\sigma}(\tau+m)\} \eta^2d^2(\hat{B}^*)^2m^2.
\end{align*}
Now, by letting $\tau=\lfloor \sqrt{T}\rfloor$, and assuming that $\eta \le \frac{1}{\alpha\hat{\sigma}\sqrt{T}}$, we have that
\begin{align*}
\|x_t-x_{k_j}\|_{\tilde{H}_t}^2&\le \frac{24\eta \tau^2d^2(\hat{B}^*)^2m^2}{\alpha}\cdot\frac{\log(|\tilde{A}_t|)}{\log(|\tilde{A}_{t-1}|)}.
\end{align*}
Thus,
\begin{align*}
(2)&\le \sum_{t=m}^T \E\left[\|x_{t}-x_{k_j}\|_{\tilde{H}_t}^2\right]\le \frac{24\eta \tau^2d^2(\hat{B}^*)^2m^2}{\alpha}\sum_{t=m}^T \frac{\log(|\tilde{A}_t|)}{\log(|\tilde{A}_{t-1}|)}\le \frac{24\eta d^3(\hat{B}^*)^2m^2 T}{\alpha}\log(\tilde{\sigma}T). 
\end{align*}

To see the bound on $(1)$: $\forall 0\le i\le m-1$,
\begin{align*}
\|\hat{Y}_{t-i}(x_t-x_{k_j})\|_2^2&= \|\hat{Y}_{t-i}\hat{A}_{t}^{-\frac{1}{2}}\hat{A}_{t}^{\frac{1}{2}}(x_t-x_{k_j})\|_2^2\\
&\le \trace (\hat{A}_{t}^{-\frac{1}{2}} \hat{Y}_{t-i}^{\top}\hat{Y}_{t-i} \hat{A}_{t}^{-\frac{1}{2}})\|x_t-x_{k_j}\|_{\hat{A}_t}^2\\
&\le \trace (\hat{A}_{t-i}^{-\frac{1}{2}} \hat{Y}_{t-i}^{\top}\hat{Y}_{t-i} \hat{A}_{t-i}^{-\frac{1}{2}})\cdot 2\eta^2\tau^2d^2(\hat{B}^*)^2m^2,
\end{align*}
where the last inequality follows similarly from before. Summing over, we have that
\begin{align*}
(1)&\le 2\eta^2\tau^2d^2(\hat{B}^*)^2m^2\sum_{t=m}^T \sum_{i=0}^{m-1} \trace (\hat{A}_{t-i}^{-\frac{1}{2}} \hat{Y}_{t-i}^{\top}\hat{Y}_{t-i} \hat{A}_{t-i}^{-\frac{1}{2}})\\
&\le  4\eta^2\tau^2d^2(\hat{B}^*)^2m \sum_{t=m}^T \trace \left(\hat{A}_{t}^{-\frac{1}{2}} \left(\sum_{i=0}^{m-1}\hat{Y}_{t-i}^{\top}\hat{Y}_{t-i}\right) \hat{A}_{t}^{-\frac{1}{2}}\right).
\end{align*}
By the same analysis as in Section~\ref{sec:perturbation-loss-known},
\begin{align*}
\sum_{t=m}^T \trace \left(\hat{A}_{t}^{-\frac{1}{2}} \left(\sum_{i=0}^{m-1}\hat{Y}_{t}^{\top}\hat{Y}_{t}\right) \hat{A}_{t}^{-\frac{1}{2}}\right)&\le \frac{10mR_{\hat{G}}R_{\hat{Y}}}{\kappa(\hat{G})}\sqrt{T}+\frac{4md\log(\hat{\sigma}T))}{\eta\alpha\kappa(\hat{G})}+dR_{\hat{Y}}^2\sqrt{T}.
\end{align*}
The result follows from taking $\eta=\tilde{\Theta}(\frac{1}{\alpha\sqrt{T}})$, $\tau=\lfloor \sqrt{T}\rfloor$, in which case
\begin{align*}
(1) = \tilde{O}(\sqrt{T}), \ \ \ (2)=\tilde{O}(\sqrt{T}), \ \ \ \ 2\nu R_{\hat{Y}}^2D^2\tau+\frac{40\nu mR_GR_{\hat{Y}}D^2}{\kappa(G)}\left\lfloor \frac{T}{\tau}\right\rfloor=\tilde{O}(\sqrt{T}).
\end{align*}

\subsection{Discussion: unknown stabilizing controllers}
To further generalize the results in the previous sections, we believe that the same regret bound can be attained for unknown system with unknown stabilizing controller $K$. Here, we sketch the idea to arrive at the same regret bound but omit the proof. In this case, we further assume that the system is strongly controllable and observable.

\begin{definition}[Strong controllability and observability] A system $(A,B)$ is $(k, \kappa)$-strongly controllable if the matrix 
\begin{align*}
C_k=\begin{bmatrix}B & AB & \dots & A^{k-1}B\end{bmatrix}\in\mathbb{R}^{d_{\x}\times kd_{\uv}}
\end{align*}
has full row rank, and $\|(C_kC_k^{\top})^{-1}\|_2\le \kappa$. A partially observable system $(A, B, C)$ is observable if the matrix 
\begin{align*}
O=\begin{bmatrix}C^{\top} & (CA)^{\top} & \dots & (CA^{d_{\x}-1})^{\top}\end{bmatrix}^{\top}\in\mathbb{R}^{d_{\x}d_{\y}\times d_{\x}}
\end{align*}
has full column rank. 
\end{definition}

We first need a Markov operator estimator, which we can run a variant of the system estimation algorithm 1 in \cite{chen2021black} to obtain a good estimator for $G_{\y}$. Then, we run the Ho-Kalman algorithm to extract $\hat{A}, \hat{B}, \hat{C}$ that are close to $A, B, C$ up to linear transformations by unitary matrices. With the estimated $\hat{A}, \hat{B}, \hat{C}$, we can run a variant of the controller recovery algorithm in \cite{chen2021black} to obtain a a controller $\hat{K}$. With the $\hat{G}, \hat{K}$, we can run Algorithm~\ref{alg:bandit-control-known}.

%\begin{algorithm}[H]
%\caption{Black-Box Newton Bandit Perturbation Controller (\texttt{Black-Box NBPC})}
%\label{alg:bandit-control-unknown}
%\begin{flushleft}
%{\bf Input:} DRC policy class $\M(m,R_{\M})$, step size $\eta>0$, time horizon $T$, strong convexity parameter $\alpha>0$. System estimation length $T_1$. Controller recovery length $T_2$. 
%\end{flushleft}
%\begin{algorithmic}[1]
%\STATE \textbf{Step 1.} Run \texttt{SysEst}$(T_1, m)$. Obtain estimates for dynamics and observation matrices $\hat{A},\hat{B},\hat{C}$. 
%\STATE \textbf{Step 2.} Run  \texttt{ControllerRecovery}$(\hat{A},\hat{B},\hat{C},T_2)$. Obtain $\hat{K}$.
%\STATE \textbf{Step 3.} Run \texttt{NBPC} (Algorithm~\ref{alg:bandit-control-known}) with $T\leftarrow T-T_1-T_2$, $G\leftarrow \hat{G}$, $K\leftarrow\hat{K}$.
%\end{algorithmic}
%\end{algorithm}

\section{BCO-M lower bound: Proof of Theorem~\ref{thm:lower-bound}  (Section~\ref{sec:bco-m})}

In this section, we prove Theorem~\ref{thm:lower-bound}. We start by constructing the sequence of loss functions. On a high level, the loss function $f_t$ has memory length of $2$ and has three different components: a scaled linear loss of the difference between the current iterate and the iterate two steps before, a random variable from a random process indexed at time $t$, and a quadratic moving cost between the current and previous iterates. Concretely, the loss function takes the following form:
\begin{align}
\label{eq:loss-moving}
f_t(x_{t-1},x_t)=\eps\ell^*(x_t-x_{t-2})+\bar{n}_t+(x_t-x_{t-1})^2,
\end{align}
where $\ell^*$ takes values in $\{-1,1\}$ with equal probability and is fixed ahead of time but unknown to the learner; $\eps$ is some scaling factor; $\bar{n}_t$ is a multi-scaled random walk indexed at time $t$. We view $f_t$ as an adaptive loss function to the learner's decision at time $t-2$.

One important distinction we make is the notion of regret in the presence of an adaptive adversary. In the presence of an adaptive adversary described in Assumption~\ref{assumption:adversary-adaptivity}, a function $f_t$ of memory length $m$ is itself a function of the decision $x_{1:t-m}$. Therefore, we can write $f_t(x_{t-m+1:t})=F_t(x_{1:t})$. We are interested in the following regret definition: for $x\in\K$, the regret w.r.t. $x$ is given by
\begin{align}
\label{eq:bco-m-regret}
\sum_{t=1}^T F_t(x_{1:t})-\sum_{t=1}^T F_t(x_{1:t-m},x,\dots,x).
\end{align}
This is the common regret notion in literature (e.g. \cite{auer2002nonstochastic, mcmahan2004online}). \cite{cesa2013online} compares this regret notion with a stronger \textit{policy regret} defined as
\begin{align*}
\sum_{t=1}^T F_t(x_{1:t})-\sum_{t=1}^T F_t(x,\dots,x),
\end{align*}
and showed that the latter has a lower bound of $\Omega(T)$. Therefore, we focus on establishing a lower bound with respect to the regret defined in Eq.~(\ref{eq:bco-m-regret}).

With the notion of regret in \cref{eq:bco-m-regret}, the regret of loss functions in \cref{eq:loss-moving} is given by $\sum_{t=1}^T f_t(x_{t-1},x_t)-\min_{x}\sum_{t=1}^T f_t(x,x)$. The optimal strategy is thus determined by the value of $\ell^*$. In the bandit setting, the learner has no access to $\ell^*$ and for sufficiently small $\eps$, the information regarding the sign of $\ell^*$ through the scaler feedback is limited. The tradeoff between moving decisions to learn $\ell^*$ and incurring moving cost results in the suboptimal bounds derived in \cref{thm:lower-bound}.

This section will be organized as the following: \cref{sec:lower-contruction} spells out the construction of the multi-scaled random walk used in the construction of $f_t$ in \cref{eq:loss-moving} and its useful properties; \cref{sec:bound-lower-loss} establishes the regularity of the constructed loss function (in particular it is bounded with high probability); \cref{sec:lower-proof} proves the regret lower bound in \cref{thm:lower-bound}.

\subsection{Construction of multi-scaled random walk and loss functions}
\label{sec:lower-contruction}
We consider the multi-scaled random walk constructed in \cite{dekel2014bandits}.
\begin{definition} [Multi-scaled random walk] 
\label{def:mrw}
Given a time horizon $T\in\mathbb{Z}_{++}$, let $\xi_1,\dots, \xi_T$ be i.i.d. $N(0,\sigma^2)$ random variables for some $\sigma\in\mathbb{R}_{++}$. The multi-scaled random walk is a random process $\{\bar{n}_{t}\}_{t=1}^T$ defined by
\begin{align*}
\bar{n}_t&=\bar{n}_{\rho(t)}+\xi_t,\\
\rho(t)&=t-\max\{2^i: i\ge 0, \ 2^i \text{ divides }t\}.
\end{align*}
\end{definition}
Given the definition of the process $\{\bar{n}_t\}_{t=1}^T$, consider the sequence of loss functions that is adaptive to the learner's decisions at time $t-2$ for every $t$.
\begin{definition} [Loss instance] 
\label{def:loss-instance}
Consider the $(t-2)$-adaptive loss function $L_t(\cdot;\eps):[0, 1]\rightarrow\mathbb{R}$ parameterized by $\eps>0$ at time $t$: 
\begin{align*}
L_t(x;\eps)&=\eps\ell^*(x-x_{t-2})+\bar{n}_t,
\end{align*}
where $\ell^*$ is a Rademacher random variable chosen before the game starts unknown to the learner, i.e. $\ell^*=\pm 1$ with probability $\frac{1}{2}$ equally.  
\end{definition}
We define the with-memory loss functions as the sum of the loss instance defined in Definition~\ref{def:loss-instance} and the quadratic moving cost between $x_t$ and $x_{t-1}$. 
\begin{definition} [Regret, moving cost, and performance metric]
\label{def:with-memory-loss-and-regret}
The additional moving cost takes the form 
\begin{align*}
M=\sum_{t=1}^T (x_t-x_{t-1})^2.
\end{align*}
We assume that the learner has full access to $M$. Note that this only makes the problem easier. Define the with-memory quadratic loss functions over $[0,1]\times [0,1]$ as
\begin{align*}
f_t(x_{t-1}, x_t)=L_t(x_t,\eps)+(x_t-x_{t-1})^2. 
\end{align*}
Let $\delta>0$ be the improper learning parameter. The regret w.r.t. any $x\in[\delta, 1-\delta]$ is measured by
\begin{align*}
\E[\regret_T(x)]&=\E\left[\sum_{t=1}^T f_t(x_{t-1},x_t)-\sum_{t=1}^T f_t(x,x)\right]\\
&=\E\left[\sum_{t=1}^T L_t(x_t;\eps)+M-\sum_{t=1}^T L_t(x;\eps)\right]. 
\end{align*}
\end{definition} 

\begin{remark}
Note that the $f_t$ defined in Definition~\ref{def:with-memory-loss-and-regret} is quadratic, smooth, and adaptive to the learner's decision at time $t-2$. In particular, this fits the assumption on the adversary model in Section~\ref{sec:bqo-AM}. 
\end{remark}

The following property of the multi-scaled random walk (MRW) constructed in Definition~\ref{def:mrw} is proved in Lemma 2 in \cite{dekel2014bandits}. To summarize the proof, we note that the depth and width defined in Lemma~\ref{lem:depth-mrw} is bounded by the number of $1$'s and $0$'s in the binary representation of $t$ respectively, which are bounded by $\lfloor \log_2 T\rfloor+1$.
\begin{lemma} [Depth and width of MRW, Lemma 2 in \cite{dekel2014bandits}]
\label{lem:depth-mrw}
The parent set of $t$ is defined by $\rho^*(t)=\{\rho(t)\}\cup \rho^*(\rho(t))$, with $\rho^*(1)=\emptyset$. For a fixed $t$, the cut of $t$ is defined by the number of integer $s$ such that $t$ lies between $s$ and its parent, i.e. $\mathrm{cut}(t)=\{s\in[T]: \rho(s)<t\le s\}$. The depth $d(\rho)$ and width $w(\rho)$ of the multi-scaled random walk defined in Definition~\ref{def:mrw} are defined and bounded by
\begin{align*}
d(\rho)&=\max_{t\in[T]}|\rho^*(t)|\le \lfloor\log_2 T\rfloor+1,\\
w(\rho)&=\max_{t\in[T]}|\mathrm{cut}(t)|\le \lfloor\log_2 T\rfloor+1.
\end{align*}
\end{lemma}

\subsection{Bound on loss functions}
\label{sec:bound-lower-loss}

To establish a meaningful lower bound, we need to have $f_t$ be bounded. It is unclear that $f_t$'s are bounded due to the component of $\bar{n}_t$ in $L_t(x_t,\eps)$. However, by the bound on the depth of $\{\bar{n}_t\}_{t=1}^T$, we argue in this section that by setting $\sigma \sim \frac{1}{\log T} $, we can without loss of generality assume that $f_t\in[-3,3]$. In particular, Lemma~\ref{lem:depth-mrw} implies a tail bound on process $\{\bar{n}_t\}_{t=1}^T$, which directly follows from tail bounds for the sum of i.i.d. Gaussian random variables.
\begin{lemma}
\label{lem:gaussian-bounded}
The random process $\{\bar{n}_t\}_{t=1}^T$ obeys the tail bound that $\forall \gamma\in (0,1)$,
\begin{align*}
\mathbb{P}\left(\max_{1\le t\le T} |\bar{n}_t|>\sigma \cdot  2 \log\left(\frac{T}{\gamma}\right)\right) \le \gamma. 
\end{align*}
\end{lemma}
\begin{proof}
By Gaussian tail bound we have
\begin{align*}
&\mathbb{P}\left(\max_{1\le t\le T} |\bar{n}_t|>\sigma \cdot 2 \log\left(\frac{T}{\gamma}\right)\right)\\
&\le \sum_{t=1}^T \mathbb{P}\left(\left|\sum_{s\in\{\rho^*(t)\cup\{t\}\}}\xi_s\right|>\sigma \cdot 2 \log\left(\frac{T}{\gamma}\right) \right)\\
&\le \sum_{t=1}^T \exp\left(-\log\frac{T}{\gamma}\right)\\
&=\gamma,
\end{align*}
since $\sum_{s\in\{\rho^*(t)\cup\{t\}\}}\xi_s$ is $N(0, (|\rho^*(t)|+1)\sigma^2)$ distributed. 
\end{proof}
Note that by definition of $L_t$, Lemma~\ref{lem:gaussian-bounded} implies that by setting
\begin{align*}
\sigma=\frac{1}{2 \log T},
\end{align*}
and $\eps \le 1$, $\gamma=\frac{1}{T}$, we have with probability at least $1-\frac{1}{T}$,
\begin{align*}
\max_{t\in[T]}|L_t(x_t;\eps)|\le 1 + \max_{t\in[T]}|\bar{n}_t|\le 2. 
\end{align*}
Thus, we can without loss of generality assume that $L_t(x_t;\eps)$ is indeed bounded between $[-3,3]$ for all $t$: note that since the addition $\bar{n}_t$ to $L_t(x;\eps)$ cancels for the learner and the comparator, the regret is always bounded by
\begin{align*}
\regret_T=\sum_{t=1}^T \eps \ell^*(x_t-x^*)+M\le (\eps+1)T\le 2T. 
\end{align*}
Let $A$ denote the event when $L_t(x_t;\eps)\in[-3,3]$, $\forall t\in[T]$, then if $\E[\regret_T]=\tilde{\Omega}(T^{\frac{2}{3}})$, since
\begin{align*}
2\gamma T+\E[\regret_T\mid A]\ge \E[\regret_T] = \tilde{\Omega}(T^{\frac{2}{3}}),
\end{align*}
we have that $\E[\regret_T\mid A]=\tilde{\Omega}(T^{\frac{2}{3}})$.

\subsection{$T^{\frac{2}{3}}$-lower bound}
\label{sec:lower-proof}
We prove a lower bound against any randomized algorithm. For any randomized algorithm, the decision random variable $X_t$ played at time $t$ is a random function of the observations $Y_1,\dots,Y_{t-1}$, where $Y_t=L_t(x_t;\eps)$. 
%Since $\ell^*$ is chosen ahead of the starting time, we can define $\mathbb{Q}_0, \mathbb{Q}_1$ as the joint distribution over $Y_1,\dots,Y_T$ conditioning on $\ell^*=1, -1$, respectively. (Note that we name $\mathbb{Q}_0, \mathbb{Q}_1$ in this way since when $\ell^*=1$, $\argmin_{x\in[0,1]} \sum_{t=1}^T L_t(x;\eps)=0$, and when $\ell^*=-1$, $\argmin_{x\in[0,1]}\sum_{t=1}^T L_t(x;\eps)=1$). 
%Let $\mathbb{Q}_0^t, \mathbb{Q}_1^t$ be the marginal distribution of $Y_t$ under $\mathbb{Q}_0, \mathbb{Q}_1$, respectively.  
Note that $Y_t$ decomposes as
\begin{align*}
Y_t=Y_{\rho(t)}+\eps\ell^*(X_t-X_{t-2})-\eps\ell^*(X_{\rho(t)}-X_{\rho(t)-2})+\xi_t.
\end{align*}
Then, conditioning on $Y_{1:t-1}$ and $X_{1:t}$,
\begin{align*}
Y_t\mid (Y_{1:t-1}, X_{1:t}, \ell^*=1)&\overset{D}{=}N\left(Y_{\rho(t)}+\eps(X_t-X_{t-2})-\eps(X_{\rho(t)}-X_{\rho(t)-2}), \sigma^2\right),\\
Y_t\mid (Y_{1:t-1},X_{1:t},\ell^*=-1)&\overset{D}{=}N\left(Y_{\rho(t)}-\eps(x_t-x_{t-2})+\eps(x_{\rho(t)}-x_{\rho(t)-2}), \sigma^2\right).
\end{align*}
We are interested in the KL divergence of the distribution $\mathbb{Q}_0^{t\mid t-1}$ and $\mathbb{Q}_1^{t\mid t-1}$ of the random variable $Y_t$. The KL divergence between two uni-variate Gaussian distributions with the same variance $\sigma^2$ and different expectation $\mu_1$, $\mu_2$ is given by 
\begin{align*}
D_{\mathrm{KL}}(N(\mu_1,\sigma^2)\parallel N(\mu_2,\sigma^2))=\frac{(\mu_1-\mu_2)^2}{2\sigma^2}.
\end{align*}
For finite sequence of random variables $\{X_s\}_{s\in S_X}, \{Y_s\}_{s\in S_Y}$ mapping from $(\Omega, \Sigma, \mathbb{P})\rightarrow (\mathbb{R},\mathcal{B}_{\mathbb{R}}, \mathrm{Leb}_{\mathbb{R}})$, we slightly abuse notation and denote as $\mathbb{P}(\{X_s\}_{s\in S_X}\mid \{Y_{s}\}_{s\in S_Y})$ the joint distribution of $\{X_s\}_{s\in S_X}$ conditioning on the sub product $\sigma$-algebra of $\Sigma$ generated by the sequence of random variables $\{Y_s\}_{s\in S_Y}$. We have then
\begin{align}
\label{eq:kl-divergence-single}
&D_{\mathrm{KL}}(\mathbb{P}(Y_t\mid (Y_{1:t-1}, X_{1:t}, \ell^*=1))\parallel \mathbb{P}(Y_t\mid (Y_{1:t-1}, X_{1:t}, \ell^*=-1))) \notag\\
&=\frac{(2\eps(X_t-X_{t-2})-2\eps(X_{\rho(t)}-X_{\rho(t)-2}))^2}{2\sigma^2} \notag\\
&\le \frac{8\eps^2(X_t-X_{t-2})^2+8\eps^2(X_{\rho(t)}-X_{\rho(t)-2})^2}{2\sigma^2} \notag\\
&\le \frac{8\eps^2}{\sigma^2}[(X_t-X_{t-1})^2+(X_{t-1}-X_{t-2})^2+(X_{\rho(t)}-X_{\rho(t)-1})^2+(X_{\rho(t)-1}-X_{\rho(t)-2})^2],
\end{align}
where we use the fact that $(a+b)^2\le 2(a^2+b^2)$. The KL divergence satisfies the chain rule stated in the following lemma.

\begin{lemma}[Chain rule for KL divergence]
\label{lem:KL-divergence-chain}
Let $X,Y$ be two random variables, and let $\mathbb{P}$, $\mathbb{Q}$ be two joint distribution over $X, Y$. Let $\mathbb{P}_X$, $\mathbb{Q}_X$ be the marginal distribution of $X$ under $\mathbb{P}$, $\mathbb{Q}$, respectively, and let $\mathbb{P}_{Y\mid X}$, $\mathbb{Q}_{Y\mid X}$ be the conditional distribution of $Y$ given $X$ under $\mathbb{P}$, $\mathbb{Q}$. Then,
\begin{align*}
D_{\mathrm{KL}}(\mathbb{P}\parallel \mathbb{Q})= D_{\mathrm{KL}}(\mathbb{P}_X\parallel \mathbb{Q}_X)+\E_{X\sim \mathbb{P}_X}[D_{\mathrm{KL}}(\mathbb{P}_{Y\mid X}\parallel \mathbb{Q}_{Y\mid X})].
\end{align*}
\end{lemma}
By Lemma~\ref{lem:KL-divergence-chain}, we can bound the KL divergence of the joint distribution of $Y_{1:T}, X_{1:T}$.
\begin{align}
&D_{\mathrm{KL}}(\mathbb{P}(Y_{1:T},X_{1:T}\mid \ell^*=1)\parallel \mathbb{P}(Y_{1:T},X_{1:T}\mid \ell^*=-1))\\ 
&=\underbrace{D_{\mathrm{KL}}(\mathbb{P}(Y_{1:T-1},X_{1:T}\mid \ell^*=1)\parallel \mathbb{P}(Y_{1:T-1},X_{1:T}\mid \ell^*=-1))}_{(1)}\\
& \ \ \ \ \ +\E_{Y_{1:T-1}, X_{1:T}\mid \ell^*=1}[D_\mathrm{KL}(\mathbb{P}(Y_T\mid Y_{1:T-1, },X_{1:T},\ell^*=1)\parallel\mathbb{P}(Y_T\mid Y_{1:T-1, },X_{1:T},\ell^*=-1)] \notag,
\end{align}
where $(1)$ can be further decomposed as
\begin{align*}
(1)\\
&=D_{\mathrm{KL}}(\mathbb{P}(Y_{1:T-1},X_{1:T-1}\mid \ell^*=1)\parallel \mathbb{P}(Y_{1:T-1},X_{1:T-1}\mid \ell^*=-1))\\
& \ \ \ \ \ + \E_{Y_{1:T-1},X_{1:T-1}\mid \ell^*=1}[\underbrace{D_{\mathrm{KL}}(\mathbb{P}(X_T\mid Y_{1:T-1}, X_{1:T-1},\ell^*=1)\parallel \mathbb{P}(X_T\mid Y_{1:T-1},X_{1:T-1},\ell^*=-1))}_{(2)}].
\end{align*}
Note that the distribution of $X_T\mid Y_{1:T-1},X_{1:T-1}$ is independent of $\ell^*$ since $X_T$ depends on $\ell^*$ only through $Y_{1:T-1}$, thus we have $(2)=0$. Recursively apply this argument, we have that 
\begin{align}
&D_{\mathrm{KL}}(\mathbb{P}(Y_{1:T},X_{1:T}\mid \ell^*=1)\parallel \mathbb{P}(Y_{1:T},X_{1:T}\mid \ell^*=-1))\\ 
&=\sum_{t=1}^T \E_{Y_{1:t-1}, X_{1:t}\mid \ell^*=1}[D_\mathrm{KL}(\mathbb{P}(Y_t\mid Y_{1:t-1, },X_{1:t},\ell^*=1)\parallel\mathbb{P}(Y_t\mid Y_{1:t-1, },X_{1:t},\ell^*=-1)] \notag \\
&\le \frac{8\eps^2}{\sigma^2}\sum_{t=1}^T\E_{X_{1:t}}[(X_t-X_{t-1})^2+ (X_{t-1}-X_{t-2})^2+(X_{\rho(t)}-X_{\rho(t)-1})^2 \notag \\
& \ \ \ \ \  \ \ \ \ \ \ \ \  \ \ \ \ \ \ +(X_{\rho(t)-1}-X_{\rho(t)-2})^2\mid \ell^*=1] \notag\\
&\le \frac{16\eps^2}{\sigma^2}\sum_{t=1}^T \E_{X_{1:t}}[(X_t-X_{t-1})^2 +(X_{\rho(t)}-X_{\rho(t)-1})^2\mid \ell^*=1] \notag\\
&\le \frac{16(\lfloor \log_2 T\rfloor+2)\eps^2}{\sigma^2}\E_{X_{1:T}}[M\mid \ell^*=1], \label{eqn:shalom}
\end{align}
where the first inequality follows from Eq.~(\ref{eq:kl-divergence-single}). To see the last inequality, note that we can bound
\begin{align*}
&\E_{X_{1:T}}\left[\sum_{t=1}^T (X_{\rho(t)}-X_{\rho(t)-1})^2\mid \ell^*=1\right]\\
&\le \left(\max_{t\in[T]} |\{s\in[T]:\rho(s)=t\}|\right)\E_{X_{1:T}}\left[\sum_{t=1}^T(X_t-X_{t-1})^2\mid \ell^*=1\right]\\
&\le w(\rho)\E_{X_{1:T}}\left[\sum_{t=1}^T(X_t-X_{t-1})^2\mid \ell^*=1\right],
\end{align*}
where the last inequality follows since if for some $t$, $|\{s\in[T]:\rho(s)=t\}|=n\ge 1$, then let $\{s_1,\dots,s_n\}$ be set of all $s\in[T]$ satisfying $\rho(s)=t$ in increasing order. We note that $\mathrm{cut}(s_1)\ge n$ since $\forall i\in\{1,\dots, n\}$, $t=\rho(s_i)=\rho(s_1)<s_1\le s_i$. Therefore, we have $w(\rho)\ge \max_{t\in[T]} |\{s\in[T]:\rho(s)=t\}|$. 

The KL divergence gives a bound on the probability measure of any measurable event $A$ through Pinsker's inequality, stated below:

\begin{lemma}[Pinsker's inequality]
\label{lemma:pinsker}
Let $\mathbb{P}, \mathbb{Q}$ be two probability distributions on a measurable space $(Y,\F)$, then we have
\begin{align*}
\sup_{A\in\F}\{|\mathbb{P}(A)-\mathbb{Q}(A)|\}\le \sqrt{\frac{1}{2}D_{\mathrm{KL}}(\mathbb{P}\parallel \mathbb{Q})}. 
\end{align*}
\end{lemma}
Let $\F$ be the $\sigma$-algebra generated by the random variables $Y_{1:T},X_{1:T}$. Then, Pinsker's inequality applied to \eqref{eqn:shalom},
\begin{align*}
& \quad \sup_{A\in\F} \{|\mathbb{P}(A\mid \ell^*=1)-\mathbb{P}(A\mid \ell^*=-1)|\}\\
&\le \sqrt{\frac{1}{2}D_{\mathrm{KL}}(\mathbb{P}(Y_{1:T},X_{1:T}\mid \ell^*=1)\parallel\mathbb{P}(Y_{1:T},X_{1:T}\mid \ell^*=-1))}\\
&\le \frac{\sqrt{8(\lfloor \log_2 T\rfloor +2)}\eps}{\sigma}\sqrt{\E[M\mid \ell^*=1]}.
\end{align*}
Similar to the above analysis, if we switch $\mathbb{Q}_0$ and $\mathbb{Q}_1$, we will have a symmetrical inequality:
\begin{align*}
& \quad \sup_{A\in\F} \{|\mathbb{P}(A\mid \ell^*=-1)-\mathbb{P}(A\mid \ell^*=1)|\}\\
&\le \sqrt{\frac{1}{2}D_{\mathrm{KL}}(\mathbb{P}(Y_{1:T},X_{1:T}\mid \ell^*=-1)\parallel\mathbb{P}(Y_{1:T},X_{1:T}\mid \ell^*=1))}\\
&\le \frac{\sqrt{8(\lfloor \log_2 T\rfloor +2)}\eps}{\sigma}\sqrt{\E[M\mid \ell^*=-1]}.
\end{align*}
Combining the two inequalities and by concavity of the function $f(x)=\sqrt{x}$ for $x\ge 0$, 
\begin{align*}
& \quad \sup_{A\in\F} \{|\mathbb{P}(A\mid \ell^*=1)-\mathbb{P}(A\mid \ell^*=-1)|\}\\
&\le \frac{\sqrt{8(\lfloor \log_2 T\rfloor +2)}\eps}{\sigma}\left(\frac{1}{2}\sqrt{\E[M\mid \ell^*=1]}+\frac{1}{2}\sqrt{\E[M\mid \ell^*=-1]}\right)\\
&\le \frac{\sqrt{8(\lfloor \log_2 T\rfloor +2)}\eps}{\sigma}\sqrt{\E[M]}.
\end{align*}
Since we view the dependence on $x_{t-2}$ in $L_t$ as an adaptive adversary, and therefore the optimal fixed point $x^*\in[\delta, 1-\delta]$ would be $x^*=\delta$ if $\ell^*=1$ and $x^*=1-\delta$ if $\ell^*=-1$, we can express the regret as
\begin{align*}
\E[\mathrm{Regret}_T]&\ge\left[\sum_{t=1}^T \frac{1}{2} \E[\eps (X_t-\delta)\mid \ell^*=1]+\frac{1}{2}\E[\eps(1-\delta-X_t)\mid \ell^*=-1]\right]+\E[M]\\
&=\frac{\eps T}{2}-\eps\delta T-\frac{\eps}{2}\left[\sum_{t=1}^T\E[X_t\mid \ell^*=1]-\E[X_t\mid \ell^*=-1]\right]+\E[M].
\end{align*}
Denote $Z=\sum_{t=1}^T X_t$. $Z$ is by definition measurable with respect to $\F$. 
Since $X_t\in[0,1]$, $\forall t$, $Z\in[0, T]$. We can further express the sum of differences between the expectation of $x_t$ under $\ell^*=1$ and $\ell^*=-1$ by
\begin{align*}
\sum_{t=1}^T\E[X_t\mid \ell^*=1]-\E[X_t\mid \ell^*=-1]&=\E\left[Z\mid \ell^*=1\right]-\E\left[Z\mid \ell^*=-1\right]\\
&=\int_{0}^T  z(\mathbb{P}(Z\in dz\mid \ell^*=1)-\mathbb{P}(Z\in dz\mid \ell^*=-1))\\
&\le \frac{\sqrt{8(\lfloor\log_2 T\rfloor+2)}\eps}{\sigma}\sqrt{\E[M]} T,
\end{align*}
%where the last inequality follows from Lemma~\ref{lemma:pinsker} and that the event $Z\in dz$ is $\F$-measurable. 
Therefore, the regret is lower bounded by
\begin{align*}
\E[\regret_T]&\ge\frac{\eps (1-2\delta) T}{2}-\frac{\eps^2 T}{2}\frac{\sqrt{8(\lfloor \log_2 T\rfloor +2)}}{\sigma}\sqrt{\E[M]}+\E[M],
\end{align*}
where the right hand side is a quadratic function of $\sqrt{\E[M]}$ minimized at
\begin{align*}
\sqrt{\E[M]}=\frac{\eps^2 T}{4}\frac{\sqrt{8(\lfloor \log_2 T\rfloor +2)}}{\sigma}
\end{align*}
with the minimal value equal to 
\begin{align*}
\frac{\eps (1-2\delta) T}{2}-\frac{\eps^4 T^2(\lfloor \log_2 T\rfloor +2)}{\sigma^2} \geq \frac{\eps (1- 2 \delta)T}{2} - \frac{2 \eps^4 T^2 \log T}{\sigma^2} . 
\end{align*}
Recall that $\sigma$ is chosen to be 
$ \sigma =\frac{1}{2 \log T}$,
and choosing $\eps = \Theta(T^{-1/3} / \log T) $ %$ \eps=\frac{(1-2\delta)^{\frac{1}{3}}}{3}(\lfloor \log_2 T\rfloor +2)^{-1}(\log T)^{-\frac{1}{2}}T^{-\frac{1}{3}}$, \eh{nope, just choose it to be $T^{-1/3}$} 
we have
\begin{align*}
\E[\regret_T]&\ge \frac{\eps (1-2\delta) T}{2}-\frac{2 \eps^4 T^2\log T }{\sigma^2} = \tilde{\Omega}(T^{2/3}) .
%&=\frac{(1-2\delta)^{\frac{4}{3}}}{3} (\lfloor \log_2 T\rfloor +2)^{-1}(\log T)^{-\frac{1}{2}}T^{\frac{2}{3}}-\frac{4(1-2\delta)^{\frac{4}{3}}}{81}(\lfloor \log_2 T\rfloor+2)^{-2}(\log T)^{-1}T^{\frac{2}{3}}\\
%&\ge \left(\frac{(1-2\delta)^{\frac{4}{3}}}{3}-\frac{4(1-2\delta)^{\frac{4}{3}}}{81}\right)(\lfloor \log_2 T\rfloor +2)^{-1}(\log T)^{-\frac{1}{2}}T^{\frac{2}{3}}\\
%&=\frac{23}{81} (1-2\delta)^{\frac{4}{3}} (\lfloor \log_2 T\rfloor +2)^{-1}(\log T)^{-\frac{1}{2}}T^{\frac{2}{3}}.
\end{align*}

%\subsection{Extension to Randomized Player}
%In the previous section, we assumed that the learning algorithm is deterministic. We extend the result to randomized algorithm. In this case, the decision $X_t$ played by the learner becomes a random variable. However, note that the loss of $L_t$'s in this case is given by
%\begin{align*}
%\E_{\ell^*,\bar{n}_t,X_{1:t}}[L_t(X_t;\eps)]&=\frac{1}{2}\int_{-\infty}^{\infty}\E[\eps\ell^*\E[X_t-X_{t-2}\mid \ell^*=1, \bar{n}_t\in dz]+\bar{n}_t]\mathbb{P}(\bar{n}_t\in dz)\\
%& \ \ \ \ \ + \frac{1}{2}\int_{-\infty}^{\infty}\E[\eps\ell^*\E[X_t-X_{t-2}\mid \ell^*=-1, \bar{n}_t\in dz]+\bar{n}_t]\mathbb{P}(\bar{n}_t\in dz),
%\end{align*}
%and the lower bound follows from the lower bound for the deterministic setting since $\E[X_t-X_{t-2}\mid\ell^*=\pm 1,\bar{n}_t\in dz]$ is a deterministic quantity for any $z\in\mathbb{R}$.

\section{BCO-M upper bound}
\label{sec:bco-m-upper-bound}
In this section, we provide a first-order, proper learning algorithm (\cref{alg:bqo-memory-first-order}) for BCO-M problems that achieves $\tilde{O}(T^{2/3})$ regret upper bound for convex quadratic, smooth functions that can have adaptivity described in Assumption~\ref{assumption:adversary-adaptivity}, which matches the lower bound established in \cref{thm:lower-bound}. The algorithm is based on the FTRL (Follow-the-Regularized-Leader, \cite{shalev2007primal}) algorithm and makes use of self-concordant barriers over compact convex sets to ensure proper learning. 

\subsection{Preliminaries on Self-Concordant Barriers}
First, we introduce self-concordant barriers, which we will make use of in Algorithm~\ref{alg:bqo-memory-first-order}.

\begin{definition} [Self-concordant barrier]
A $\mathcal{C}^3$ function $R(\cdot)$ over a closed convex set $\K\subset\mathbb{R}^d$ with non-empty interior is a $\nu$-self-concordant barrier of $\K$ if it satisfies the following two properties:
\begin{enumerate}
\item (Boundary property) For any sequence $\{x_n\}_{n\in\mathbb{N}}\subset \mathrm{int}(\K)$ such that $\lim_{n\rightarrow\infty}x_n=x\in\partial\K$, $\lim_{n\rightarrow\infty}R(x_n)=\infty$. 
\item (Self-concordant) $\forall x\in\mathrm{int}(\K)$, $h\in\mathbb{R}^d$, 
\begin{enumerate}
\item $|\nabla^3R(x)[h,h,h]|\le 2|\nabla^2 R(x)[h,h]|^{3/2}$.
\item $|\langle \nabla R(x),h\rangle|\le \sqrt{\nu}|\nabla^2 R(x)[h,h]|^{1/2}$. 
\end{enumerate}
\end{enumerate}
\end{definition} 
It is well known that self-concordant functions satisfy the following properties:

\begin{proposition}[Properties of self-concordant functions, \cite{suggala2021efficient}]
$\nu$-self-concordant barriers over $\K$ satisfy the following properties:
\label{prop:self-concordant-barrier}
\begin{enumerate}
\item Sum of two self-concordant functions is self-concordant. Linear and quadratic functions are self-concordant.
\item If $x,y\in\K$ satisfies $\|x-y\|_{\nabla^2 R(x)}<1$, then the following inequality holds:
\begin{align}
(1-\|x-y\|_{\nabla^2 R(x)})^2 \nabla^2 R(x)\preceq \nabla^2 R(y)\preceq \frac{1}{(1-\|x-y\|_{\nabla^2 R(x)})^2}\nabla^2 R(x). \label{eqn:scb-hessian-inequality}
\end{align}
\item The Dikin ellipsoid centered at any point in the interior of $\K$ w.r.t. a self-concordant barrier $R(\cdot)$ over $\K$ is completely contained in $\K$. Namely,
\begin{align}
\{y\in\mathbb{R}^d\mid \|y-x\|_{\nabla^2 R(x)}\le 1\}\subset \K, \ \ \forall x\in \text{int}(\K). \label{eqn:dikin-ellipsoid}
\end{align}
where 
\begin{align*}
    \|v\|_{\nabla^2 R(x)}\defeq\sqrt{v^\top \nabla^2 R(x) v}
\end{align*}
\item $\forall x,y\in\text{int}(\K)$:
\begin{align*}
R(y)-R(x)\le \nu\log\frac{1}{1-\pi_x(y)},
\end{align*}
where $\pi_x(y)\defeq\inf\{t\ge0:x+t^{-1}(y-x)\in\K\}$.
\end{enumerate}
\end{proposition}

\subsection{BQO-M algorithm and analysis}
We introduce \cref{alg:bqo-memory-first-order}, the assumptions, and guarantees.

\begin{assumption}
\label{assumption:bqo-first-order}
We assume the convex compact constraint set $\K\subset\mathbb{R}^d$ and the convex, quadratic loss function $f_t:\K^m\rightarrow\mathbb{R}$ satisfies Assumption~\ref{assumption:diameter-grad-bound} and \ref{assumption:adversary-adaptivity} in \cref{subsec:bqo-am-prelims}. Moreover, the Hessian is bounded above and satisfies $\nabla^2 f_t\preceq\beta I_{dm}$.
\end{assumption}

\begin{algorithm}[H]
\caption{Bandit Quadratic Optimization with Memory}\label{alg:bqo-memory-first-order}
\begin{flushleft}
  {\bf Input:} convex compact set $\K$, step size $\eta>0$, perturbation parameter $\delta\in[0,1]$, $\alpha$-strongly convex $\nu$-self-concordant barrier $R$ over $\K$, memory length $m\in\mathbb{N}$, time horizon $T\in\mathbb{N}$. 
\end{flushleft}
\begin{algorithmic}[1]
\STATE Initialize: $x_{1}=\dots=x_{m}=\argmin_{x\in\K}R(x)$, $\tilde{g}_{0:m-1} = \mathbf{0}_{d}$, $ A_1=\dots=A_m = \nabla^2R(x_1)$. \label{line:initialize-first-order-bqo-m}
\STATE Sample $u_t\sim \sphere^{d-1}$ i.i.d. uniformly at random for $t=1,\dots,m$. 
\STATE Set $y_t = x_t + \delta A_{t}^{-\frac{1}{2}}u_t$, $t=1,\dots,m$. 
\FOR{$t = m, \ldots, T$}
\STATE Play $y_t$, observe $f_t(y_{t-m+1:t})$.
\STATE Create gradient estimator:
\begin{align*}
\tilde{g}_{t}&=\frac{d}{\delta}f_t(y_{t-m+1:t}) \sum_{j=0}^{m-1}A_{t-j}^{\frac{1}{2}}u_{t-j}\in\mathbb{R}^{d}.
\end{align*}
\STATE Update
$ x_{t+1} = \argmin_{x\in\K} \eta\sum_{s=m}^{t-m+1} \langle \tilde{g}_{s}, x\rangle + R(x)$, $A_{t+1}=\nabla^2 R(x_{t+1})$. \label{line:update-first-order-bqo-m}
\STATE Independently from previous steps, sample $u_{t+1}\sim \sphere^{d-1}$ uniformly at random. 
\STATE Set $y_{t+1}=x_{t+1}+\delta A_{t+1}^{-\frac{1}{2}}u_{t+1}$. 
\ENDFOR
\end{algorithmic}
\end{algorithm}

Before stating the regret guarantee, we note that \cref{alg:bqo-memory-first-order} satisfies two properties: proper learning and delayed dependence, given by the following two remarks. 

\begin{remark} [Correctness]
\label{remark:correctness}
All $y_t$'s played by Algorithm~\ref{alg:bqo-memory-first-order} satisfy that $y_t\in\K$ since $\|y_t-x_t\|_{\nabla^2 R(x_t)}^2=\delta^2 u_t^{\top}A_{t}^{-\frac{1}{2}}A_tA_t^{-\frac{1}{2}}u_t=\delta^2\le 1$. By Eq.~(\ref{eqn:dikin-ellipsoid}) in Proposition~\ref{prop:self-concordant-barrier}, $y_t\in\K$. 
\end{remark}

\begin{remark} [Delayed dependence]
\label{remark:delayed-dependence}
In Algorithm~\ref{alg:bqo-memory-first-order}, $y_t,\tilde{g}_{t}$ are $\F_t$-measurable, and $x_{t},A_{t}$ are $\F_{t-m}$-measurable.
\end{remark}

Similar to \cref{lem:grad-est}, the constructed gradient estimators $\tilde{g}_t$ is a conditionally unbiased estimator of the divergence of $f_t$ evaluated at $(x_{t-m+1},\dots,x_t)$. 

\begin{lemma} [Unbiased gradient estimator]
\label{lem:unbiased-grad-estimator}
$\tilde{g}_{t}$ is a conditionally unbiased estimator of the following statistics: 
\begin{align*}
\E[\tilde{g}_{t}\mid \F_{t-m}]= \nabla\cdot f_t(x_{t-m+1:t}), 
\end{align*}
where $\F_t=\sigma(\{u_s\}_{s\le t})$ is the filtration generated by the algorithm's random sampling up to time $t$, and $\nabla\cdot f_t(x_{t-m+1:t})$ denotes the divergence of $f_t$ evaluated at $x_{t-m+1:t}$. 
\end{lemma}
\begin{proof}
Similar as in the proof of Lemma~\ref{lem:grad-est}.
\end{proof}

Finally, the regret guarantee of \cref{alg:bqo-memory-first-order} is given by the following proposition.

\begin{proposition}
\label{prop:bco-m-upper-bound}
For $d, m, T\in\mathbb{N}$, convex compact set $\K\subset\mathbb{R}^d$ and $\{f_t\}_{t=m}^T$ satisfying Assumption~\ref{assumption:bqo-first-order}, $\alpha$-strongly convex $\nu$-self-concordant barrier function $R$ over $\K$, suppose \cref{alg:bqo-memory-first-order} is run with input ($\K, \eta, \delta,\alpha, R, m, T$).
The regret can be decomposed as
\begin{align*}
\E[\regret_T(x)]&=\underbrace{\sum_{t=m}^T \E[f_t(y_{t-m+1:t})-f_t(x_{t-m+1:t})]}_{(1: \text{exploration  loss})} + \underbrace{\sum_{t=m}^T \E[f_t(x_{t-m+1:t})-\bar{f}_{t}(x_t)]}_{(2: \text{movement loss})}\\
& \ \ \ \ \ +\underbrace{\sum_{t=m}^T\E[\bar{f}_{t}(x_t)-\bar{f}_{t}(x)]}_{(3: \text{underlying regret})},
\end{align*}
and each of the above terms can be bounded as following
\begin{align*}
(1)\le C_1\delta^2 T, \ \ \ (2)\le C_2\frac{\eta^2T}{\delta^2}, \ \ \ (3)\le  C_3\frac{\eta T}{\delta^2}+C_4\frac{1}{\eta}+C_5,
\end{align*}
where $C_1, \dots, C_5$ are given by
\begin{align*}
C_1=\frac{\beta m}{2\alpha}, \ C_2=\frac{8\beta d^2B^2m^5}{\alpha}, \ C_3=16d^2B^2m^3+\frac{4d\beta DBm^4}{\sqrt{\alpha}}, \ C_4=\nu\log T, \ C_5=mB.
\end{align*}
In particular, by setting $\delta=T^{-\frac{1}{6}}$, $\eta=T^{-\frac{2}{3}}$, and $m=\mathrm{poly}(\log T)$, we have $\forall x\in\K$, $\E[\regret_T(x)]\le \tilde{O}(T^{\frac{2}{3}})$. 
\end{proposition}

\begin{proof}
Proposition~\ref{prop:bco-m-upper-bound} can be established similarly as Theorem~\ref{thm:bqo-m-regret}, where we bound each of the decomposed loss as follows.
\paragraph{Perturbation loss.} The perturbation loss can be bounded by
\begin{align*}
\sum_{t=m}^T \E[f_t(y_{t-m+1:t})-f_t(x_{t-m+1:t})]&\le \sum_{t=m}^{T}\delta\E[\nabla f_t(x_{t-m+1:t})^{\top} A_{t-m+1:t}^{-\frac{1}{2}}u_{t-m+1:t}] \\
& \ \ \ \ \ + \frac{\delta^2}{2} \E[u_{t-m+1:t}^{\top}A_{t-m+1:t}^{-\frac{1}{2}}(\beta I)A_{t-m+1:t}^{-\frac{1}{2}}u_{t-m+1:t}]\\
&=_{(1)}\frac{\delta^2\beta}{2}\sum_{t=m}^T\sum_{s=t-m+1}^t \E[u_s^{\top}(\nabla^2 R(x_s))^{-1}u_s]\\
&\le_{(2)} \frac{\delta^2\beta m}{2\alpha} T\\
&=C_1\delta^2 T,
\end{align*}
where (1) follows from the $m$-step delayed dependence and independence of the perturbations,
and (2) follows from the $\alpha$-strong convexity assumption on the self-concordant barrier $R(\cdot)$.

\paragraph{Movement loss.} To bound the movement loss, it is necessary and sufficient to bound the $\ell_2$-distance between neighboring iterates, i.e. $\|x_{t+1}-x_t\|_2$. 

Denote $\Phi_t(x)=\eta\sum_{s=m}^{t-m+1}\langle \tilde{g}_s,x\rangle +R(x)$. By specification of Algorithm~\ref{alg:bqo-memory-first-order},  we have that $x_{t+1}=\argmin_{x\in\K}\Phi_t(x)$. By Taylor's theorem, $\exists \alpha\in[0,1]$ such that $z=\alpha x_t+(1-\alpha)x_{t+1}$ satisfies
\begin{align*}
\frac{1}{2}\|x_t-x_{t+1}\|_{\nabla^2 R(z)}^2&=_{(3)}\Phi_t(x_t)-\Phi_t(x_{t+1})-\nabla \Phi_t(x_{t+1})^{\top}(x_t-x_{t+1})\\
&\le_{(4)} (\Phi_{t-1}(x_t)-\Phi_{t-1}(x_{t+1}))+\eta\langle \tilde{g}_{t-m+1}, x_{t}-x_{t+1} \rangle\\
&\le_{(5)} \eta \|\tilde{g}_{t-m+1}\|_{\nabla^2 R(x_t)}^{*}\|x_t-x_{t+1}\|_{\nabla^2 R(x_t)},
\end{align*}
where (3) follows from $\nabla^2\Phi_t(x)=\nabla^2 R(x)$, $\forall x\in\K$; (4) follows from $\nabla \Phi_t(x_{t+1})^{\top}(x_t-x_{t+1})\ge 0$ by optimality condition; (5) follows from $x_t=\argmin_{x\in\K}\Phi_{t-1}(x)$ and generalized Cauchy-Schwarz. It is clear from this expression that it suffices to bound $\|\tilde{g}_{t-m+1}\|_{\nabla^2 R(x_t)}^*$ since 
\begin{align*}
\|x_t-x_{t+1}\|_2\le \frac{2\eta}{\sqrt{\alpha}}\|\tilde{g}_{t-m+1}\|_{\nabla^2 R(x_t)}^*.
\end{align*}
\begin{lemma}
\label{lem:gradient-est-local-norm-bound}
Provided that $\eta\le \frac{\delta}{4dBm}\left(1-\exp\left(-\frac{\log 2}{m}\right)\right)$,  $\forall t\ge m$, it holds that 
\begin{align*}
\|\tilde{g}_{t-m+1}\|_{\nabla^2 R(x_t)}^*\le \frac{2dBm}{\delta}.
\end{align*}
\end{lemma}
\begin{proof}
The base case follows by construction since $\tilde{g}_1=\dots=\tilde{g}_{m-1}=0$, the claimed inequality holds for any $m\le t\le 2m-2$. 

\paragraph{Induction hypothesis.} Suppose for some $t\ge 2m-1$, $\|\tilde{g}_{s-m+1}\|_{\nabla^2 R(x_s)}^*\le \frac{2dBm}{\delta}$ holds $\forall s\le t-1$.  

\paragraph{Induction.} Consider $\|\tilde{g}_{t-m+1}\|_{\nabla^2 R(x_t)}^*$:
\begin{align}
\left(\|\tilde{g}_{t-m+1}\|_{\nabla^2 R(x_t)}^*\right)^2&\le \frac{d^2B^2}{\delta^2}\sum_{j,k=0}^{m-1}u_{t-j}^{\top}A_{t-j}^{\frac{1}{2}}A_t^{-1}A_{t-k}^{\frac{1}{2}}u_{t-k} \nonumber\\
&\le \frac{d^2B^2m^2}{\delta^2}\max_{0\le j\le m-1}\left\|A_{t-j}^{\frac{1}{2}}A_t^{-1}A_{t-j}^{\frac{1}{2}}\right\|_2. \label{eqn:induction-grad-bound}
\end{align}
By induction hypothesis and choice of $\eta$, $\forall s\le t$, we have
\begin{align*}
\|x_s-x_{s-1}\|_{\nabla^2R(x_{s-1})}\le 2\eta \|\tilde{g}_{s-m}\|_{\nabla^2 R(x_{s-1})}^*\le \frac{4\eta dBm}{\delta}\le 1. 
\end{align*}
Then, by \cref{eqn:scb-hessian-inequality} in Proposition~\ref{prop:self-concordant-barrier}, we have that $\forall s\le t$,
\begin{align*}
A_s\succeq \left(1-\frac{4\eta dBm}{\delta}\right)^2 A_{s-1} \ \Rightarrow \ A_t\succeq \left(1-\frac{4\eta dBm}{\delta}\right)^{2m} A_{t-i}, \ \forall i\in\{0,\dots,m-1\}. 
\end{align*}
By assumption on $\eta$, we have that 
\begin{align*}
A_t\succeq \left(\exp\left(-\frac{\log 2}{m}\right)\right)^{2m}A_{t-i}=\frac{1}{4}A_{t-i}, \ \ \ \forall i\in\{0,\dots,m-1\}.
\end{align*}
Thus, Eq.~(\ref{eqn:induction-grad-bound}) implies that 
\begin{align*}
\|\tilde{g}_{t-m+1}\|_{\nabla^2 R(x_t)}^*\le \frac{2dBm}{\delta}.
\end{align*}
\end{proof}
Lemma~\ref{lem:gradient-est-local-norm-bound} directly implies that $\|x_t-x_{t+1}\|_2\le \frac{4\eta dBm}{\delta\sqrt{\alpha}}$, $\forall t$, which together with the assumption that $f_t$ is $\beta$-smooth, gives that
\begin{align*}
\sum_{t=m}^T \E[f_t(x_{t-m+1:t})-\bar{f}_{t}(x_t)]&\le\frac{\beta}{2}\sum_{t=m}^T\sum_{s=t-m+1}^t\E\|x_s-x_t\|_2^2]\\
&\le\frac{\beta m}{2}\sum_{t=m}^T\sum_{s=t-m+1}^t\sum_{r=s+1}^t \E[\|x_r-x_{r-1}\|_2^2]\\
&\le \frac{\beta m^3 T}{2}\cdot\frac{16\eta^2d^2B^2m^2}{\delta^2\alpha}\\
&=\frac{8\eta^2\beta d^2B^2m^5}{\delta^2\alpha}\\
&=C_2\frac{\eta^2}{\delta^2}T. 
\end{align*}

\paragraph{Underlying regret.} To see the bound on the underlying regret, note that by Lemma~\ref{lem:unbiased-grad-estimator} and analysis in Section~\ref{sec:ons-regret-known}, we have that
\begin{align*}
\left\|\E[\tilde{g}_t\mid \F_{t-m}]-\nabla \bar{f}_{t}(x_{t})\right\|_2\le m\beta \sum_{s=t-m+1}^t\sum_{r=s+1}^{t} \|x_r-x_{r-1}\|_2\le \frac{4\eta d\beta B m^4}{\delta\sqrt{\alpha}}.
\end{align*}
By the above bias bound,
\begin{align*}
& \quad \sum_{t=m}^T\E[\bar{f}_{t}(x_{t})-\bar{f}_t(x)]\\
&\le\sum_{t=m}^{T-m+1}\left(\E\left[\tilde{g}_{t}^{\top}(x_{t}-x)\right]+\E[(\nabla \bar{f}_{t}(x_{t})-\E[\tilde{g}_t\mid\F_{t-m}])^{\top}(x_{t}-x)]\right)+mB\\
&\le \underbrace{\sum_{t=m}^{T-m+1}\E\left[\tilde{g}_{t}^{\top}(x_{t+m}-x)\right]}_{(a)}+\underbrace{\sum_{t=m}^{T-m+1}\E\left[\tilde{g}_{t}^{\top}(x_{t}-x_{t+m})\right]}_{(b)}+\frac{4\eta d\beta DB m^4T}{\delta\sqrt{\alpha}}+mB.
\end{align*}
We will bound (a) and (b). (a) is given by a variant of the FTL-BTL lemma (Lemma 5.2.2, \cite{hazan2022oco}). 
\begin{lemma}[FTL-BTL, with memory]
\label{lem:ftl-btl-m}
$\forall t\ge 2m-2$, $t\in\mathbb{N}$, $\forall x\in\K$, the $x_t,\tilde{g}_t$ given by \cref{alg:bqo-memory-first-order} with step size $\eta>0$ and self-concordant barrier $R(\cdot)$ satisfies the following inequality:
\begin{align}
\sum_{s=m}^{t-m+1}\eta\langle \tilde{g}_{s},x\rangle + R(x)\ge \sum_{s=m}^{t-m+1} \eta\langle \tilde{g}_s, x_{s+m}\rangle + R(x_{m}).\label{eqn:ftl-btl}
\end{align}
\end{lemma}
\begin{proof}
We prove \cref{eqn:ftl-btl} by induction. For $t=2m-2$, the right hand side equals to $R(x_m)$, and $R(x_m)\le R(x)$, $\forall x\in\K$ since $x_m=\argmin_{x\in\K}R(x)$ (Line~\ref{line:initialize-first-order-bqo-m}, \cref{alg:bqo-memory-first-order}). Suppose \cref{eqn:ftl-btl} holds for some $t\ge 2m-2$. Consider $t+1$. We have $\forall x\in\K$, by Line~\ref{line:update-first-order-bqo-m},
\begin{align*}
\sum_{s=m}^{t-m+2} \eta\langle \tilde{g}_s, x_{t+2}\rangle + R(x_{t+2})&\le \sum_{s=m}^{t-m+2}\eta\langle \tilde{g}_{s},x\rangle + R(x).
\end{align*}
By induction hypothesis, we have $\forall x\in\K$,
\begin{align*}
\sum_{s=m}^{t-m+2}\eta\langle \tilde{g}_{s},x\rangle + R(x)&\ge \sum_{s=m}^{t-m+1} \eta\langle \tilde{g}_s, x_{s+m}\rangle + R(x_m)+\eta\langle\tilde{g}_{t-m+2}, x_{t+2}\rangle\\
&=\sum_{s=m}^{t-m+2}\eta \langle \tilde{g}_s, x_{s+m}\rangle + R(x_m).
\end{align*}
\end{proof}

Lemma~\ref{lem:ftl-btl-m} implies a bound on (a): we can without loss of generality assume that the comparator $x$ satisfies $\pi_{2m-1}(x)> 1-\frac{1}{T}$ at a cost of $O(1)$, and then by Proposition~\ref{prop:self-concordant-barrier},
\begin{align*}
(a)&\le \frac{1}{\eta}(R(x)-R(x_m))\le \frac{\nu\log T}{\eta}.
\end{align*}
By bounds established for the local norm of $(x_{t+1}-x_t)$ and $\tilde{g}_t$, we have
\begin{align*}
(b)&\le \sum_{t=m}^{T-m+1} \E\left[\|\tilde{g}_t\|_{\nabla^2 R(x_t)}^*\left(\sum_{i=0}^{m-1}\|x_{t+i}-x_{t+i+1}\|_{\nabla^2 R(x_t)}\right)\right]\le \frac{16\eta d^2B^2m^3T}{\delta^2},
\end{align*}
since $\nabla^2R(x_t)\preceq 4 \nabla^2R(x_{t+i})$ for all $0\le i\le m-1$. Combining the bounds, we have that
\begin{align*}
\sum_{t=m}^T\E[\bar{f}_{t}(x_{t})-\bar{f}_t(x)]&\le \frac{\nu\log T}{\eta}+\frac{16\eta d^2B^2m^3T}{\delta^2}+\frac{4\eta d\beta DB m^4T}{\delta\sqrt{\alpha}}+mB\\
&= C_3\frac{\eta T}{\delta^2} + C_4\frac{1}{\eta} + C_5.
\end{align*}
\end{proof}

\end{document}